\documentclass[10pt,twocolumn,letterpaper]{article}

\usepackage{cvpr}  

\def\cvprPaperID{1936} 
\def\confName{CVPR}
\def\confYear{2023}

\usepackage[numbers]{natbib}
\usepackage{times}  
\usepackage{helvet}  
\usepackage{courier}  
\usepackage[hyphens]{url}  
\usepackage{graphicx} 
\usepackage{caption} 
\usepackage[table,xcdraw,dvipsnames]{xcolor}
\definecolor{mine}{RGB}{205, 232, 248} 
\usepackage{newfloat}
\usepackage{listings}

\usepackage[utf8]{inputenc} 
\usepackage[T1]{fontenc}    
\usepackage{url}            
\usepackage{booktabs}       
\usepackage{amsfonts}       
\usepackage{nicefrac}       
\usepackage{microtype}      
\usepackage{xcolor}         

\usepackage{amsmath,amssymb,amsfonts}
\usepackage{graphicx}
\usepackage{textcomp}
\usepackage{xcolor}
\usepackage{marvosym}
\usepackage{ifsym}

\usepackage{wrapfig}
\usepackage{amsthm}
\usepackage{algorithm,algorithmicx, algpseudocode}
\usepackage{threeparttable}

\usepackage{booktabs} %
\usepackage{url}
\usepackage[pagebackref]{hyperref}
\hypersetup{colorlinks,linkcolor={red},citecolor={green}}
\usepackage{amssymb}
\usepackage{mathrsfs}
\usepackage{bbm}
\usepackage{dsfont}
\usepackage{bbold}
\usepackage{mathbbol}
\usepackage{newtxmath}
\usepackage{color}
\usepackage{float}
\usepackage{booktabs}
\usepackage{epstopdf}
\DeclareMathAlphabet{\mathpzc}{OT1}{pzc}{m}{it}

\usepackage{amsmath}
\DeclareMathOperator{\Var}{Var}
\usepackage{cleveref}
\usepackage{multirow}

\definecolor{codegreen}{rgb}{0,0.6,0}
\definecolor{codegray}{rgb}{0.5,0.5,0.5}
\definecolor{codepurple}{rgb}{0.58,0,0.82}
\definecolor{backcolour}{rgb}{0.95,0.95,0.92}

\definecolor{mygreen}{rgb}{0,0.6,0}
\definecolor{mygray}{rgb}{0.5,0.5,0.5}
\definecolor{mymauve}{rgb}{0.58,0,0.82}
\lstdefinestyle{mystyle}{
    backgroundcolor=\color{backcolour},   
    commentstyle=\color{codegreen},
    keywordstyle=\color{magenta},
    numberstyle=\tiny\color{codegray},
    stringstyle=\color{codepurple},
    basicstyle=\ttfamily\footnotesize,
    breakatwhitespace=false,         
    breaklines=true,                 
    captionpos=b,                    
    keepspaces=true,                 
    numbers=left,                    
    numbersep=5pt,                  
    showspaces=false,                
    showstringspaces=false,
    showtabs=false,                  
    tabsize=2
}

\usepackage{lipsum}
\usepackage{newfloat}
\usepackage{listings}
\usepackage{stfloats}
\theoremstyle{plain}
\newtheorem{theorem}{Theorem}[section]

\newtheorem{lemma}[theorem]{Lemma}

\theoremstyle{definition}

\newtheorem{assumption}[theorem]{Assumption}
\theoremstyle{remark}

\DeclareCaptionStyle{ruled}{labelfont=normalfont,labelsep=colon,strut=off} 
\floatstyle{ruled}
\newfloat{listing}{tb}{lst}{}
\floatname{listing}{Listing}
\title{Frustratingly Easy Regularization on Representation Can \\Boost Deep Reinforcement Learning}

\author{
Qiang He$^{1}$ \quad  Huangyuan Su$^2$\quad Jieyu Zhang$^3$  \quad  Xinwen Hou$^1$\\
$^1$Institute of Automation, Chinese Academy of Sciences, Beijing, China \\
$^2$Carnegie Mellon University, Pittsburgh, United States \\
$^3$University of Washington, Seattle, United States \\
{\tt\small qianghe97@gmail.com,  huangyus@andrew.cmu.edu, jieyuz2@cs.washington.edu, xinwen.hou@ia.ac.cn}
}

\usepackage{bibentry}

\begin{document}

\maketitle

\begin{abstract}
Deep reinforcement learning (DRL) gives the promise that an agent learns good policy from high-dimensional information, whereas representation learning removes irrelevant and redundant information and retains pertinent information. In this work, we demonstrate that the learned representation of the $Q$-network and its target $Q$-network should, in theory, satisfy a favorable \textit{distinguishable representation property}. Specifically, there exists an upper bound on the representation similarity of the value functions of two adjacent time steps in a typical DRL setting. However, through illustrative experiments, we show that the learned DRL agent may violate this property and lead to a sub-optimal policy. Therefore, we propose a simple yet effective regularizer called \underline{P}olicy \underline{E}valuation with \underline{E}asy Regularization on \underline{R}epresentation (PEER), which aims to maintain the distinguishable representation property via explicit regularization on internal representations. And we provide the convergence rate guarantee of PEER. Implementing PEER requires only one line of code. Our experiments demonstrate that incorporating PEER into DRL can significantly improve performance and sample efficiency. Comprehensive experiments show that PEER achieves state-of-the-art performance on all \textbf{4} environments on PyBullet, \textbf{9} out of \textbf{12} tasks on DMControl, and \textbf{19} out of \textbf{26} games on Atari. To the best of our knowledge, PEER is the first work to study the inherent representation property of $Q$-network and its target. Our code is available at \href{https://sites.google.com/view/peer-cvpr2023/}{https://sites.google.com/view/peer-cvpr2023/}.
\end{abstract}

\section{Introduction}
Deep reinforcement learning (DRL) leverages the function approximation abilities of deep neural networks (DNN) and the credit assignment capabilities of RL to enable agents to perform complex control tasks using high-dimensional observations such as image pixels and sensor information \cite{alpha,nature_dqn,starcraft,sac}. DNNs are used to parameterize the policy and value functions, but this requires the removal of irrelevant and redundant information while retaining pertinent information, which is the task of representation learning. As a result, representation learning has been the focus of attention for researchers in the field \cite{planet,laskin2020curl,dreamer,sacae,drq,decisionformer,trajectoryformer,schmidhuber1990making,DBLP:conf/iclr/JaderbergMCSLSK17}. In this paper, we investigate the inherent representation properties of DRL.

The action-value function is a measure of the quality of taking an action in a given state. In DRL, this function is approximated by the action-value network or $Q$-network. To enhance the training stability of the DRL agent, \citet{nature_dqn} introduced a target network, which computes the target value with the frozen network parameters. The weights of a target network are either periodically replicated from learning $Q$-network, or exponentially averaged over time steps. Despite the crucial role played by the target network in DRL, previous studies\citep{unreal,pbl,cpc,laskin2020curl,lyle2021effect, drq, drqv2,ghosh2020representations,rad,bellemare2019geometric, dadashi2019value, eysenbach2021information,lyle2021understanding,planet,dreamer,sacae,decisionformer,trajectoryformer,schmidhuber1990making,DBLP:conf/iclr/JaderbergMCSLSK17} have not considered the representation property of the target network. In this work, we investigate the inherent representation property of the $Q$-network. Following the commonly used definition of representation of $Q$-network \cite{Levine17,dabney21,lyle2021effect}, the $Q$-network can be separated into a nonlinear encoder and a linear layer, with the representation being the output of the nonlinear encoder. By employing this decomposition, we reformulate the Bellman equation \cite{rl} from the perspective of representation of $Q$-network and its target.  We then analyze this formulation and demonstrate theoretically that a favorable \textit{distinguishable representation property} exists between the representation of $Q$-network and that of its target. Specifically, there exists an upper bound on the representation similarity of the value functions of two adjacent time steps in a typical DRL setting, which differs from previous work.

We subsequently conduct experimental verification to investigate whether agents can maintain the favorable distinguishable representation property. To this end, we choose two prominent DRL algorithms, TD3~\cite{td3} and CURL~\cite{laskin2020curl} (without/with explicit representation learning techniques). The experimental results indicate that the TD3 agent indeed maintains the distinguishable representation property, which is a positive sign for its performance. However, the CURL agent fails to preserve this property, which can potentially have negative effects on the model's overall performance.

These theoretical and experimental findings motivate us to propose a simple yet effective regularizer, named \underline{P}olicy \underline{E}valuation with \underline{E}asy Regularization on \underline{R}epresentation (PEER).  PEER aims to ensure that the agent maintains the distinguishable representation property via explicit regularization on the $Q$-network's internal representations. Specifically, {PEER} regularizes the policy evaluation phase by pushing the representation of the $Q$-network away from its target. Implementing PEER requires only one line of code.  Additionally, we provide a theoretical guarantee for the convergence of PEER.

We evaluate the effectiveness of PEER by combining it with three representative DRL methods TD3~\cite{td3}, CURL~\cite{laskin2020curl}, and DrQ~\cite{drq}. The experiments show that PEER effectively maintains the distinguishable representation property in both state-based PyBullet \cite{bullet3} and pixel-based DMControl \cite{dm-control} suites. Additionally, comprehensive experiments demonstrate that PEER outperforms compared algorithms on four tested suites PyBullet, MuJoCo~\cite{mujoco}, DMControl, and Atari~\cite{atari}. Specifically, PEER achieves state-of-the-art performance on \textbf{4} out of \textbf{4} environments on PyBullet, \textbf{9} out of \textbf{12} tasks on DMControl, and \textbf{19} out of \textbf{26} games on Atari. Moreover, our results also reveal that combining algorithms (e.g., TD3, CURL, DrQ) with the PEER loss outperforms their respective backbone methods. This observation suggests that the performance of DRL algorithms may be negatively impacted if the favorable distinguishable representation property is not maintained. The results also demonstrate that the PEER loss is orthogonal to existing representation learning methods in DRL.

Our contributions are summarized as follows. (i) We theoretically demonstrate the existence of a favorable property, \textit{distinguishable representation property}, between the representation of $Q$-network and its target. (ii) The experiments show that learned DRL agents may violate such a property, possibly leading to sub-optimal policy. To address this issue, we propose an easy-to-implement and effective regularizer {PEER} that ensures that the property is maintained. To the best of our knowledge, the PEER loss is the first work to study the inherent representation property of $Q$-network and its target and be leveraged to boost DRL. (iii) In addition, we provide the convergence rate guarantee of PEER. (iv) To demonstrate the effectiveness of {PEER}, we perform comprehensive experiments on four commonly used RL suits PyBullet, MuJoCo, DMControl, and Atari suites. The empirical results show that PEER can dramatically boost state-of-the-art representation learning DRL methods.
\section{Preliminaries}\label{sec: preliminaries}

DRL aims to optimize the policy through return, which is defined as $R_t=\sum_{i=t}^{T}\gamma^{i-t} r(s_i, a_i)$. 

\textbf{$Q$-network and its target.} Action value function $Q^\pi(s,a)$ represents the quality of a specific action $a$ in a state $s$. Formally, the action value (Q) function  is defined as
\begin{equation}
	Q^\pi(s,a) = \mathbb{E}_{\tau \sim \pi, p} [R_{\tau} | s_0 = s, a_0 = a],
\end{equation}
where trajectory $\tau$ is a state-action sequence $(s_0, a_0, s_1, a_1, s_2, a_2 \cdots)$ induced by policy $\pi$ and transition probability function $p$. A four-tuple $(s_t, a_t, r_t, s_{t+1})$ is called a transition.
 The $Q$ value can be recursively computed by Bellman equation \cite{rl}
\begin{equation}
Q^\pi(s,a) = r(s,a) + \gamma \mathbb{E}_{s',a'} [Q^\pi(s',a')],
\label{eq: Q-learning}
\end{equation}
where $s' \sim p(\cdot|s,a)$ and $a' \sim \pi(s')$. The process of evaluating value function is known as the policy evaluation phase. To stabilize the training of DRL, \citet{nature_dqn} introduced a target network to update the learning network with $\theta' \leftarrow \eta \theta + (1-\eta) \theta'$, where $\eta$ is a small constant controlling the update scale. $\theta$ is the parameters of the $Q$-network. And $\theta'$ denotes the parameters of the target network.

\textbf{Representation of $Q$-network.}
We consider a multi-layer neural network representing the Q function parameterized by $\Theta$. Let $\Theta_i$ represent the parameters of $i-$th layer, $\Theta_{-1}$ represent the parameters of the last layer, and $\Theta_{+}$ represent the parameters of the neural networks except for those of the last layer. The representation $\Phi$ of the $Q$-network is defined as
	\begin{equation}
		Q(s,a; \Theta) = \left \langle \Phi(s,a; \Theta_+), \Theta_{-1}\right \rangle.
		\label{eq: def representation}
	\end{equation}
One intuitive way to comprehend the representation of the $Q$-network is to view the network as composed of a nonlinear encoder and a linear layer. The representation of the $Q$-network is the output of the encoder~\cite{boyan99,Levine17,dabney21,a_geomentrix_perspective, lyle2021effect}. 
 \section{Method}

In this section, we start with a theoretical analysis of the $Q$-network and demonstrate the existence of the distinguishable representation property. The preliminary experiments  (\cref{fig: alpha visuallization}) reveal the desirable property may be violated by agents learned with existing methods. To ensure agents maintain such property and therefore alleviate the potential negative impact on model performance, we propose \underline{P}olicy \underline{E}valuation with \underline{E}asy Regularization on \underline{R}epresentation (PEER), a simple yet effective regularization loss on the representation of the Q-/target network.
Finally, we employ a toy example to demonstrate the efficacy of the proposed PEER. We defer all the proofs to Appendix \cref{appendix: proof}.
\begin{figure*}[!htbp]
	\centering
  \begin{subfigure}{0.25\linewidth}
     \includegraphics[width=1\textwidth]{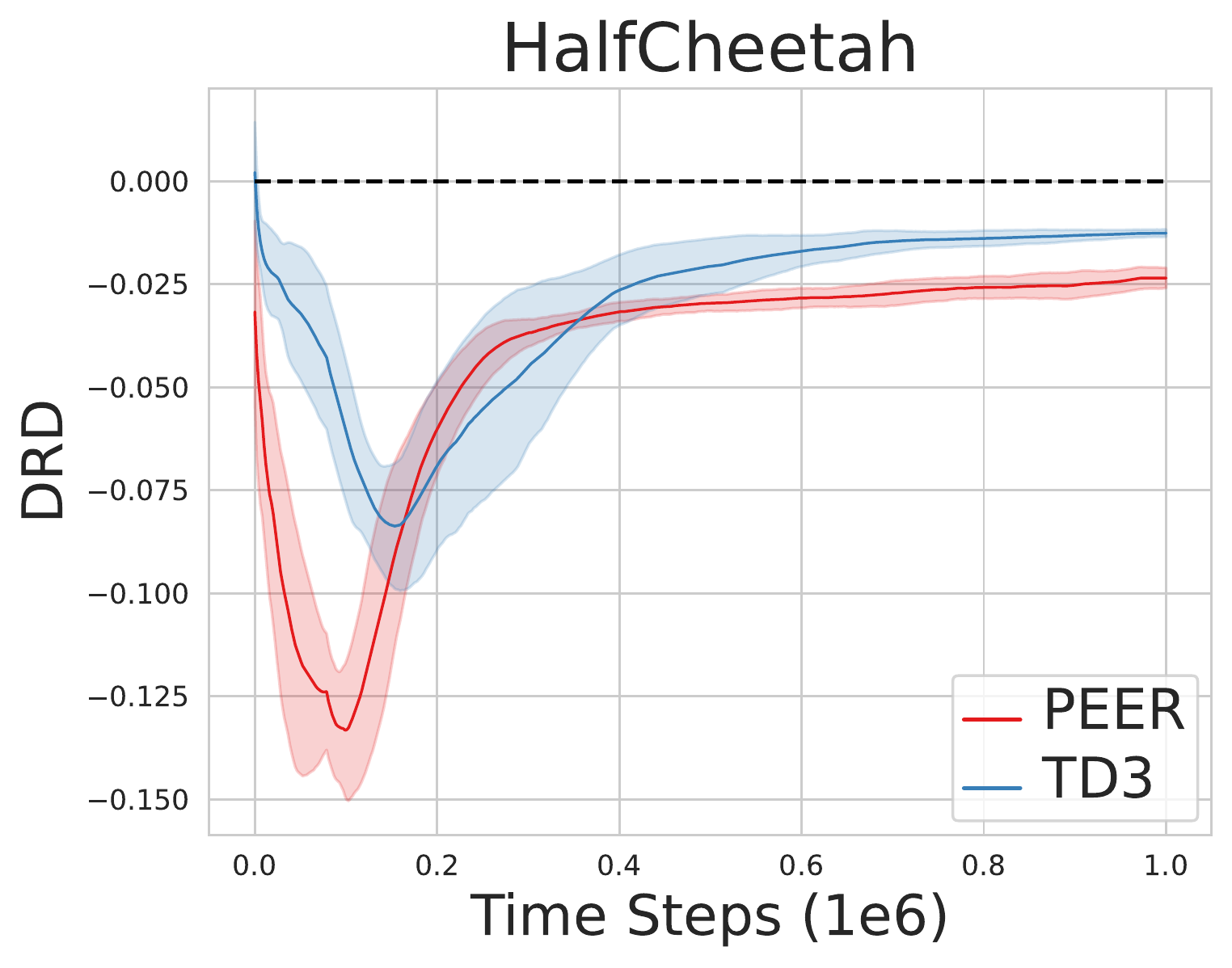}
 \end{subfigure}
\hspace{-0.1in}
  \begin{subfigure}{0.25\linewidth}
     \includegraphics[width=1\textwidth]{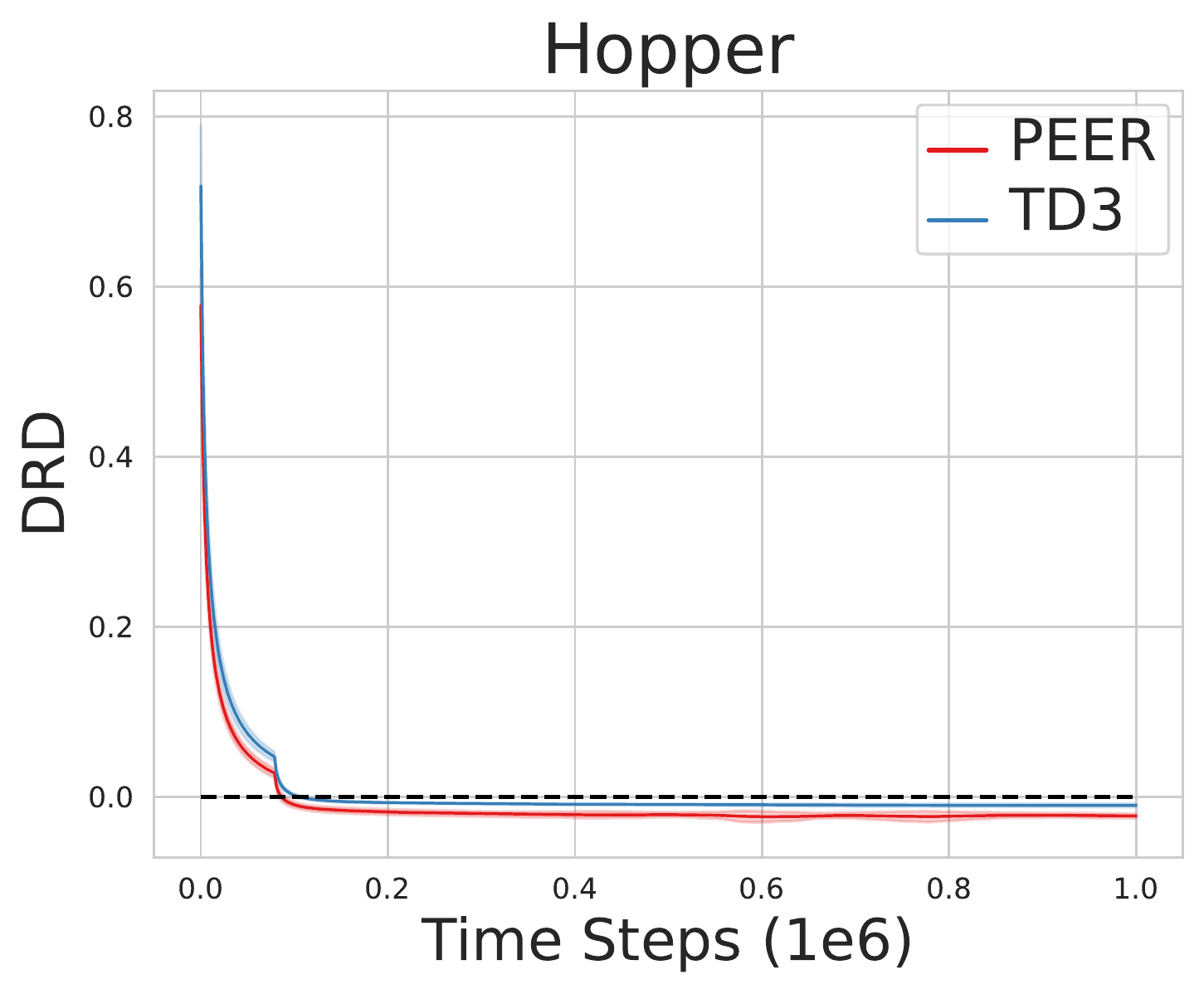}
 \end{subfigure}
\hspace{-0.1in}
   \begin{subfigure}{0.25\linewidth}
     \includegraphics[width=1\textwidth]{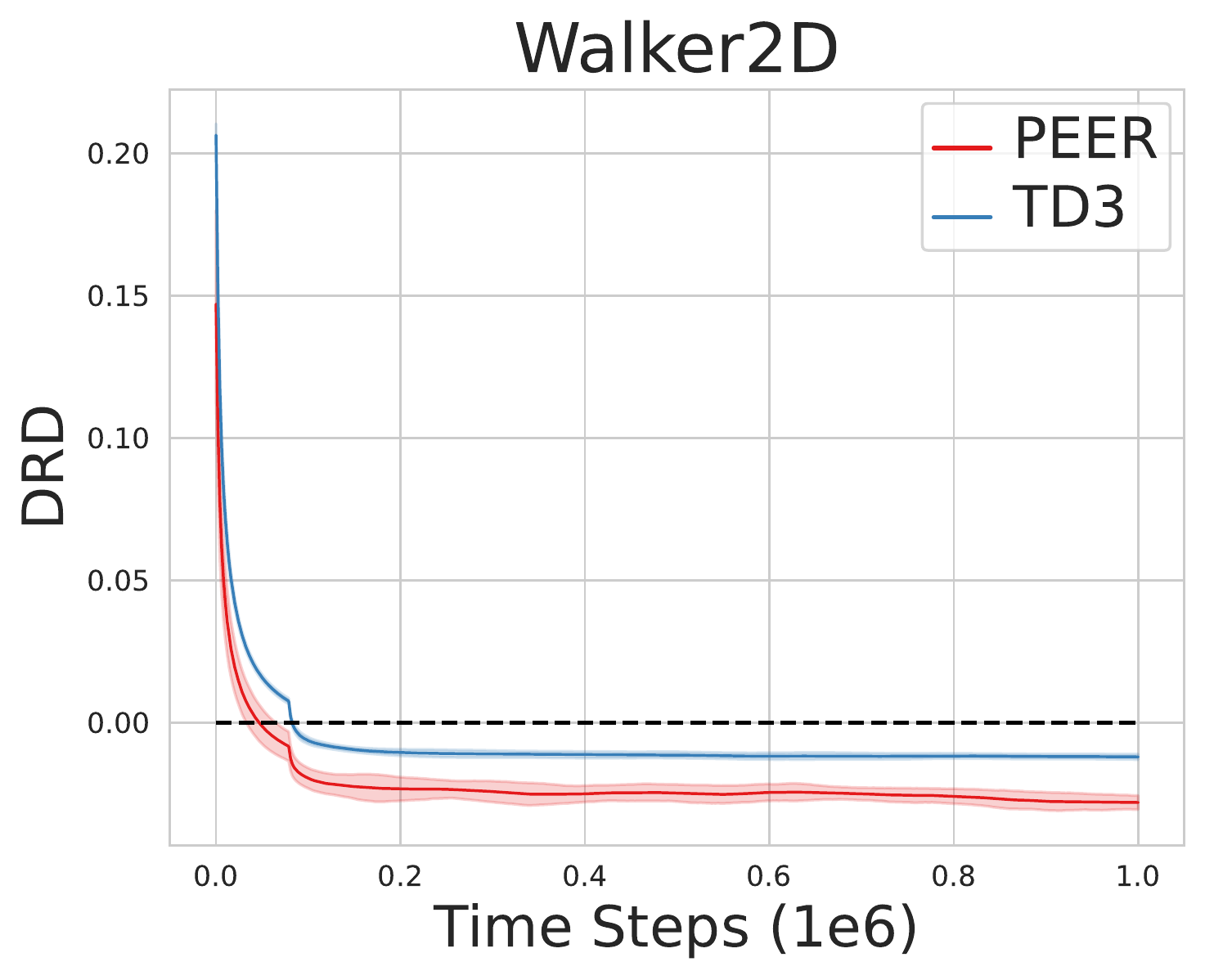}
 \end{subfigure}
\hspace{-0.1in}
   \begin{subfigure}{0.25\linewidth}
     \includegraphics[width=1\textwidth]{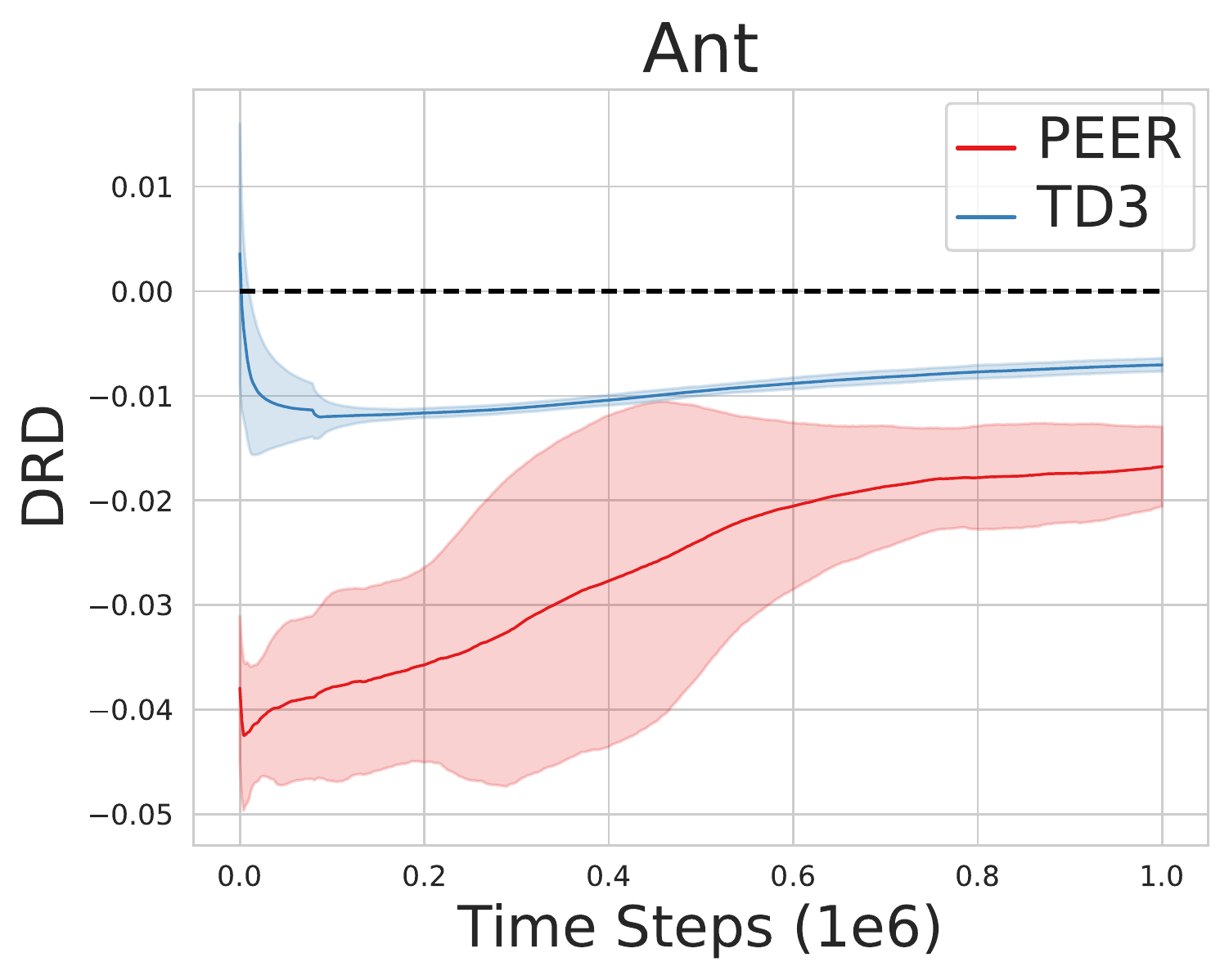}
 \end{subfigure}
   \begin{subfigure}{0.25\linewidth}
     \includegraphics[width=1\textwidth]{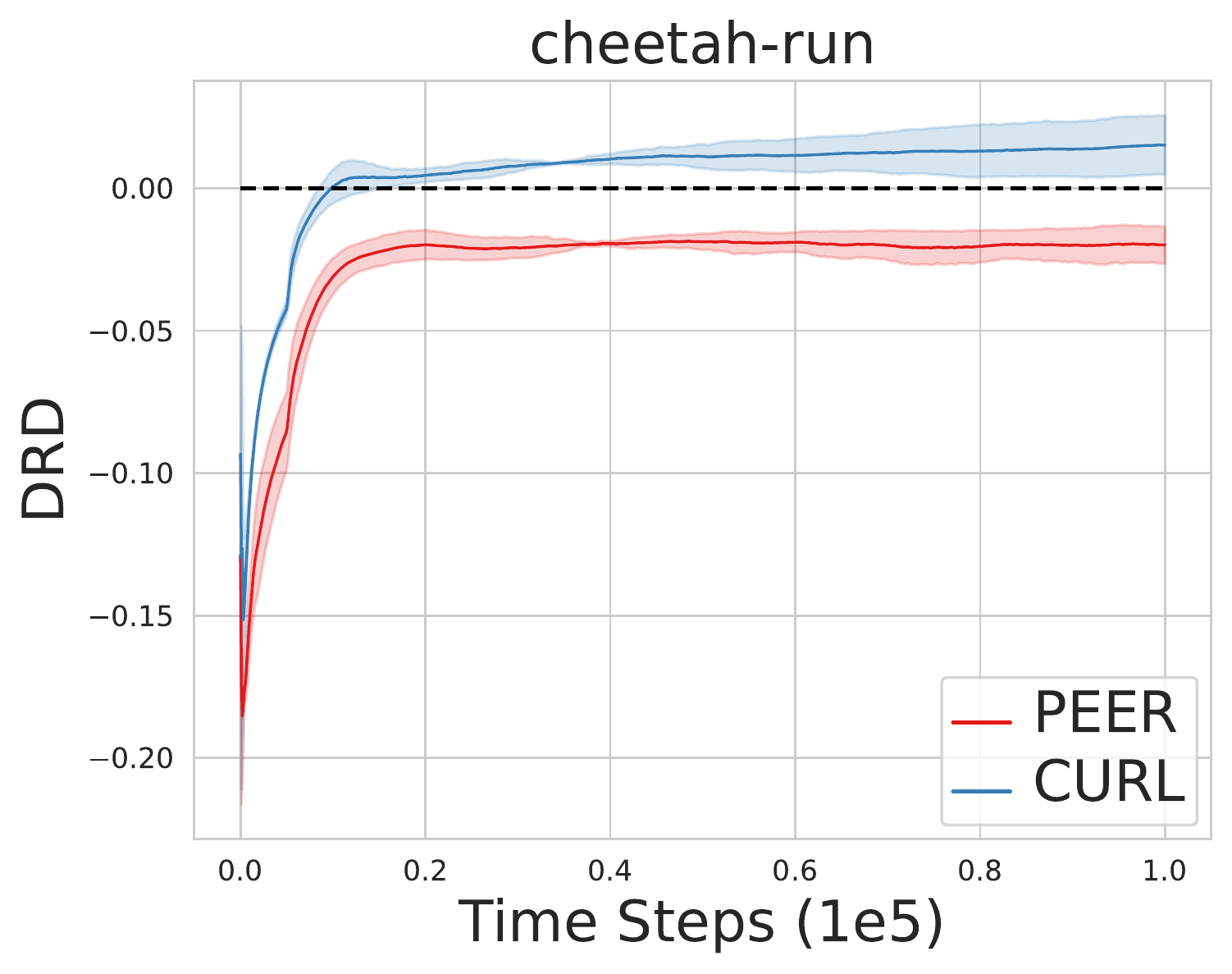}
 \end{subfigure}
\hspace{-0.1in}
   \begin{subfigure}{0.25\linewidth}
     \includegraphics[width=1\textwidth]{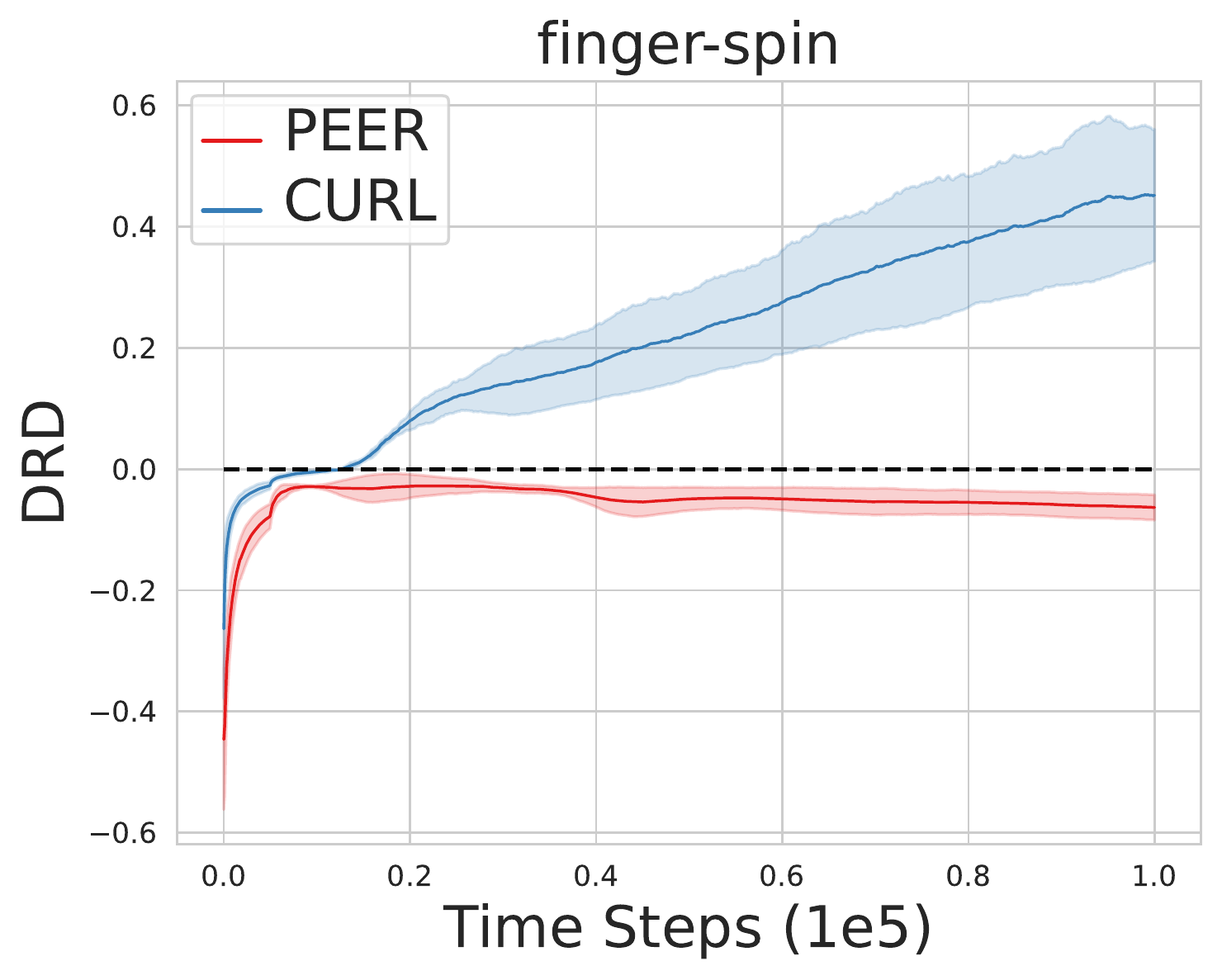}
 \end{subfigure}
\hspace{-0.1in}
   \begin{subfigure}{0.25\linewidth}
     \includegraphics[width=1\textwidth]{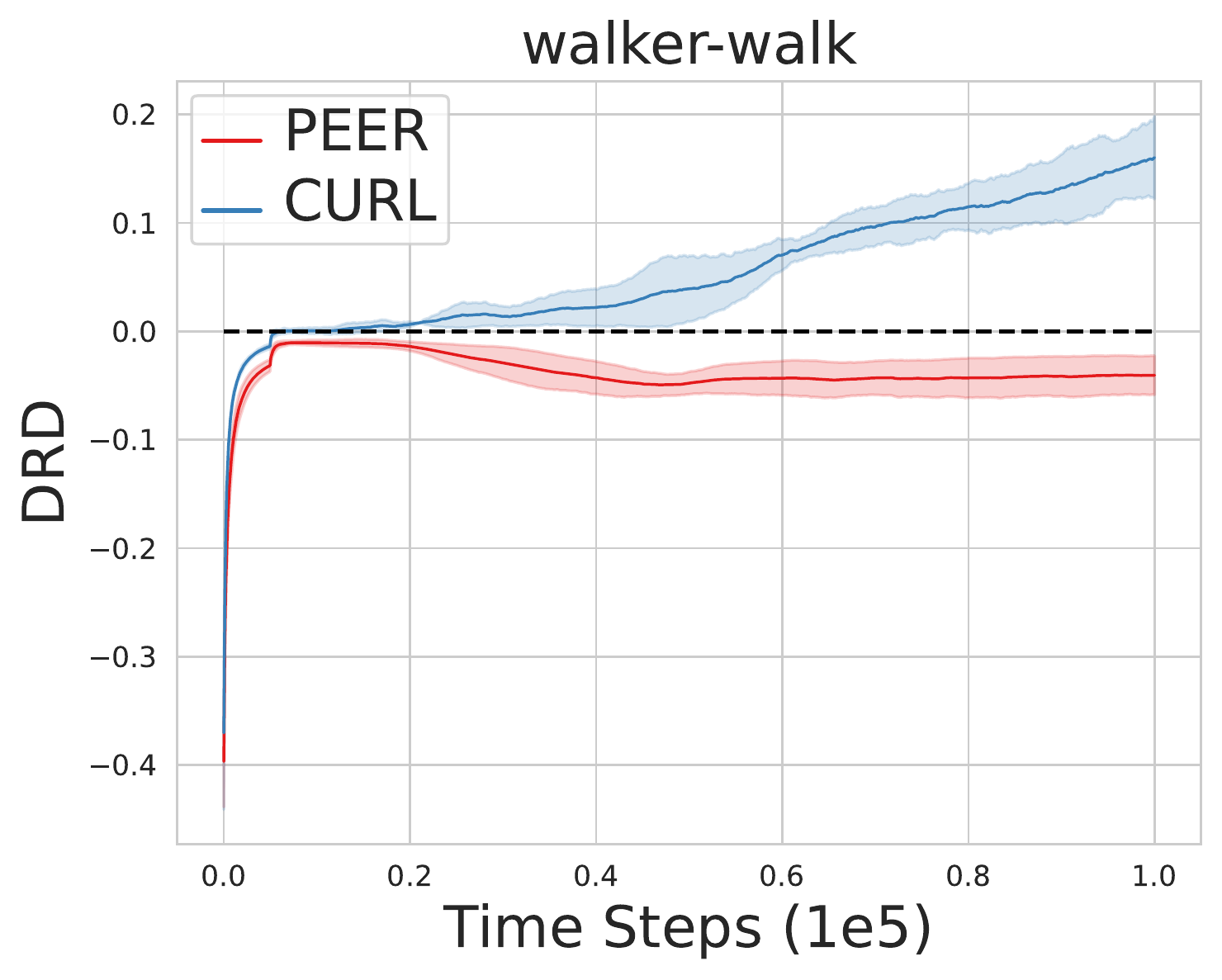}
 \end{subfigure}
\hspace{-0.1in}
    \begin{subfigure}{0.25\linewidth}
     \includegraphics[width=1\textwidth]{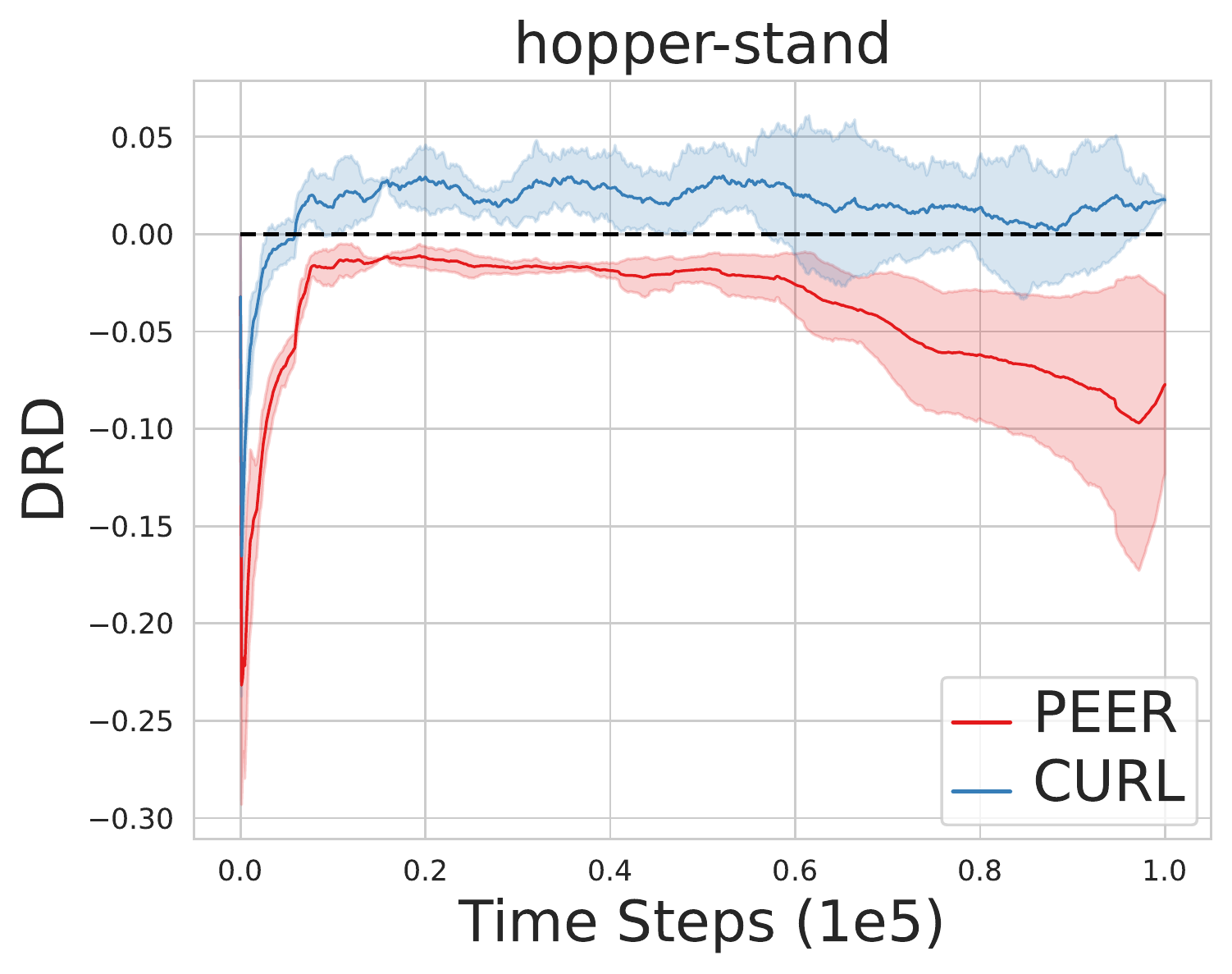}
 \end{subfigure}
	\caption{Distinguishable representation discrepancy (DRD) of TD3, CURL, and PEER agents on PyBullet and DMControl suites. TD3 agent does enjoy the distinguishable representation property. But the CURL agent does not satisfy the distinguishable representation property, which might negatively affect the model performance (see \cref{table: exp dm control}). Not only does PEER enjoy the distinguishable representation property on state-based inputs environment bullets, but also in the pixel-based environment DMControl. The shaded area stands for a standard deviation.}
	\label{fig: alpha visuallization}
	\vspace{-0.1in}
\end{figure*}

\subsection{Theoretical analysis} \label{sec: theoretical analysis}

First of all, we theoretically examine the property that the internal representations of a Q-network and its target should satisfy under \cref{ass: theta-prod}. Specifically, we show that the similarity of internal representations among two adjacent action-state pairs must be below a constant value, determined by both the neural network and reward function.

\begin{assumption}
The $l_2$-norm of the representation is uniformly bounded by the square of some positive constant $G$, i.e. $\lVert \Phi(X; \Theta_{+}) \rVert^2 \leq G^2$ for any $X\in \mathcal{S}\times \mathcal{A}$ and network weights $\Theta$.
\label{ass: theta-prod}
\end{assumption}

\begin{theorem}[Distinguishable Representation Property]\label{theorem: representation gap}
	The similarity (defined as inner product $\langle \cdot, \cdot \rangle$) between normalized representations $\Phi (s,a; \Theta_{+}) $ of the $Q$-network and $ \mathbb{E}_{s',a'}\Phi(s',a'; \Theta'_{+})$ satisfies 

\begin{equation}\label{thm: representation gap}
   \langle \Phi(s,a; \Theta_+), \mathbb{E}_{s',a'} \Phi(s',a'; \Theta'_{+}) \rangle  \leq \frac{1}{\gamma} - \frac{r(s,a)^2}{2\lVert \Theta_{-1} \rVert^2},
\end{equation}
where $s, a$ and $\Theta_+$ are state, action, and parameters of the $Q$-network except for those of the last layer. While $s', a', \Theta'_{+}$ are the state, action at the next time step, and parameters of the target $Q$-network except for those of the last layer. And $\Theta_{-1}$ is the parameters of the last layer of $Q$-network. 
\end{theorem}

Intuitively, \cref{theorem: representation gap} indicates that the internal representations of different action-state pairs should be \emph{distinguishable} so that the model is able to better pick the right action from action space. We define distinguishable representation discrepancy (DRD) as 
\begin{equation}
    \text{DRD} = \langle \Phi(s,a; \Theta_+), \mathbb{E}_{s',a'} \Phi(s',a'; \Theta'_{+}) \rangle  - \big( \frac{1}{\gamma} - \frac{r(s,a)^2}{2\lVert \Theta_{-1} \rVert^2} \big).
\end{equation}
We determine whether the internal representations of the Q-network and its target satisfy the distinguishable representation property by examining the value of DRD. If DRD $\leq 0$, then the property is satisfied; otherwise, it is not.
Then, we evaluate whether this property is maintained by agents learned with two representative DRL methods, namely TD3~\cite{td3} and CURL~\cite{laskin2020curl} (without/with explicit representation learning techniques). The results are shown in \cref{fig: alpha visuallization}. 
From the results, we see that the TD3 agent does preserve the distinguishable representation property. While the CURL agent using explicit representation learning techniques does not satisfy the distinguishable representation property. This violation of the property could have a detrimental impact on agent performance (see \cref{sec: experimental results}).
These theoretical and experimental results motivate us to propose the following PEER regularization loss. This loss ensures that the learned model satisfies the distinguishable representation property through explicit regularization on representations.
\begin{figure}[!htbp]
	\begin{center}
		\includegraphics[width=\linewidth]{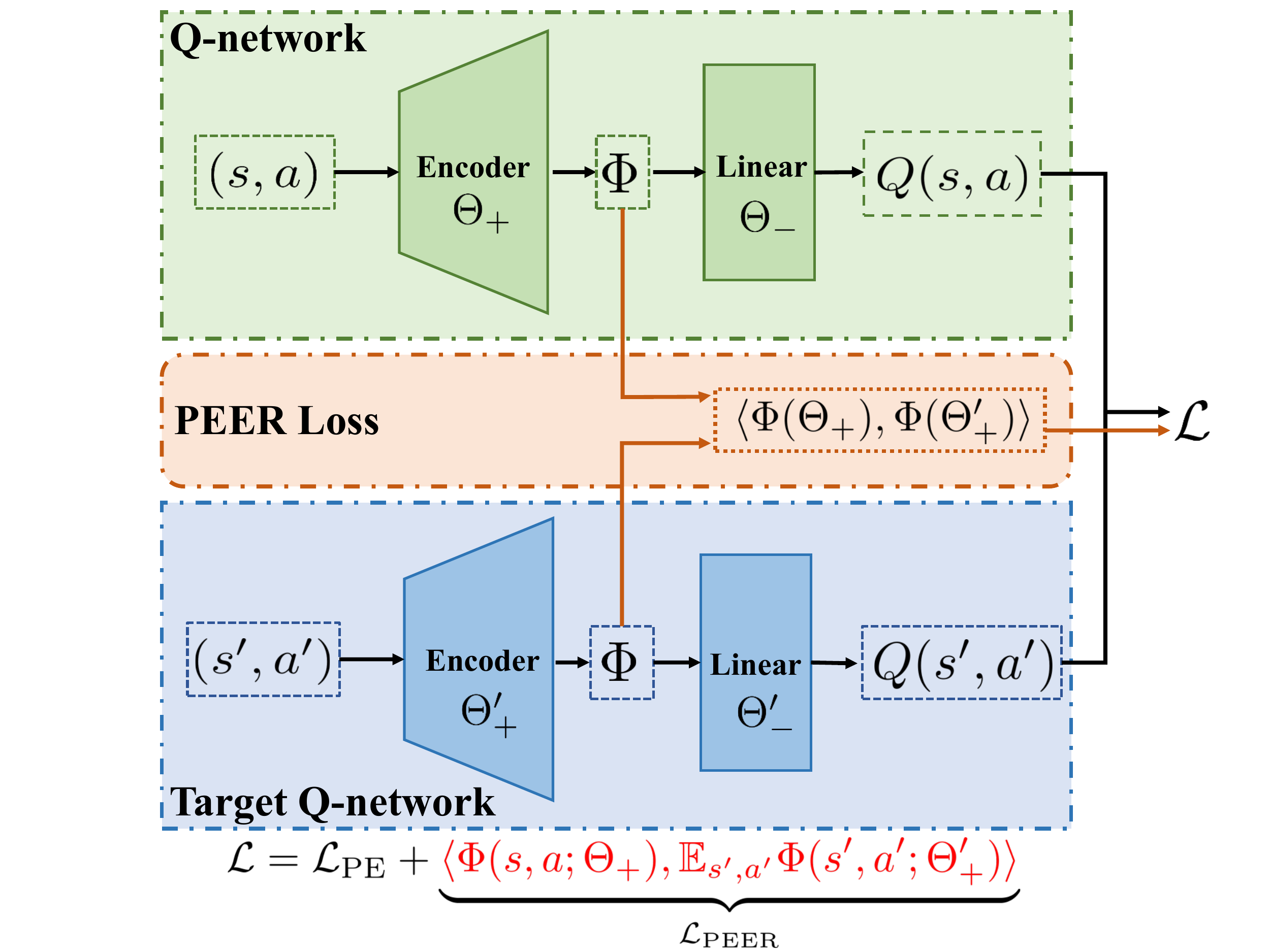} 
	\end{center}
 \vspace{-0.1in}
	\caption{\label{fig: PEER structure}How the PEER loss is computed. The encoder is a nonlinear operator, and the state action pairs generate the representation $\Phi$ through the encoder and then the action value through a linear layer. PEER regularizes the policy evaluation phase by differing the representation $\Phi$ of the $Q$-network from that of its target. $\mathcal{L}_{\text{PE}}$ is a form of policy evaluation loss. And $\beta$ is a small positive constant, controlling the magnitude of the regularization effectiveness.}
 \vspace{-0.2in}
\end{figure}

\subsection{PEER}

Specifically, the PEER loss is defined as
\begin{equation}
	\mathcal{L}_{\text{ PEER}} (\Theta) = \langle\Phi(s,a; \Theta_+), \mathbb{E}_{s',a'} \big[\Phi(s',a'; \Theta'_+)\big] \rangle,
\end{equation}
where the $Q$-network takes the current state and action as inputs, and the inputs to the target network are the state and action at the next time step.
This simple loss can be readily incorporated with the optimization objective of standard policy evaluation methods~\cite{ddpg,sac,nature_dqn}, leading to
\begin{equation}
	\mathcal{L}(\Theta) = \mathcal{L}_{\text{PE}}(\Theta) + \beta \mathcal{L}_{\text{ PEER}} (\Theta),
	\label{eq: PEER loss}
\end{equation}
where hyper-parameter $\beta$ controls the magnitude of the regularization effect of PEER. And $\mathcal{L}_{\text{PE}}(\Theta)$ is a policy evaluation phase loss e.g. 
\begin{equation*}
	\mathcal{L}_{\text{PE}} (\Theta) = \Big[Q(s,a; \Theta) - \Big(r(s,a) + \gamma \mathbb{E}_{s',a'} \big[Q(s',a'; \Theta')\big] \Big) \Big]^2.
\end{equation*}

PEER can be combined with any DRL method that includes a policy evaluation phase, such as DQN \cite{nature_dqn}, TD3,  SAC \cite{sac}. Experiments presented in 
\cref{fig: alpha visuallization} demonstrate that PEER maintains the distinguishable representation property. Furthermore, extensive experiments demonstrate PEER significantly improves existing algorithms by keeping the such property. The Pytorch-like pseudocode for the PEER loss can be found in Appendix \cref{app sec: Pseudocode}. And \cref{fig: PEER structure} illustrates how the PEER loss is computed.

\paragraph{Convergence Guarantee.}
We additionally provide a convergence guarantee of our algorithm.
Following the definition in \cite{rudolf2018upper}, let $\mathcal{F}$ be a class of measurable functions, a $\delta$-cover for $\mathcal{F}$ with $\delta > 0$ is a finite set $\Gamma_{\delta} \subset \mathcal{F}$ such that $\forall f \in \mathcal{F}$, there exists $g \in \Gamma_{\delta}$ such that $ \lVert f - g \rVert_{\infty} \leq \delta$, where $\lVert \cdot \rVert_{\infty}$ is the $l_{\infty}$-norm. A minimal $\delta$-cover is a $\delta$-cover and if taking out any of its elements, it is no longer a $\delta$-cover for $\mathcal{F}$.
Let $T$ be the Bellman Operator. We have the following convergence result for the core update step in PEER.

\begin{theorem}[One-step Approximation Error of the PEER Update]
\label{thm: convergence}
Suppose \cref{ass: theta-prod} hold, let $\mathcal{F}\subset \mathcal{B}(\mathcal{S}\times\mathcal{A})$ be a class of measurable function on $\mathcal{S} \times \mathcal{A}$ that are bounded by $V_{\max} = R_{\max}/(1 - \gamma)$, and let $\sigma$ be a probability distribution on $\mathcal{S} \times \mathcal{A}$. Also,  let $\{(S_i, A_i)\}_{i\in[n]}$ be $n$ i.i.d. random variables in $\mathcal{S} \times \mathcal{A}$ following $\sigma$. For each
$i \in [n]$, let $R_i$ and $S_i$ be the reward and the next state corresponding to $(s_i, a_i)$. In addition, for $Q \in \mathcal{F}$, we define $Y_i = R_i + \gamma \cdot \max_{a\in \mathcal{A}}Q(S'_i, a)$. Based on $\{(X_i, A_i, Y_i)\}_{i\in[n]}$, we define $\hat{O}$ as the solution to the lease-square with regularization problem,
\begin{equation}
    \min_{f\in \mathcal{F}} \frac{1}{n}\sum_{i=1}^n[f(S_i, A_i) - Y_i]^2 + \beta \langle\Phi(s,a;\Theta), \mathbb{E}\Phi_{s',a'}(s',a';\Theta')\rangle.
\end{equation}
Meanwhile, for any $\delta > 0$, let $\mathcal{N}(\delta, \mathcal{F}, \lVert \cdot \rVert_{\infty})$ be the minimal $\delta$-covering set of $\mathcal{F}$ with respect to
$l_\infty$-norm, and we denote by $N_{\delta}$ its cardinality. Then for any $\epsilon \in (0, 1]$ and any $\delta > 0$, we have
\begin{equation}
\label{conv-result}
    \lVert\hat{O} - TQ\rVert_\sigma^2 \leq (1+\epsilon)^2\cdot \omega(\mathcal{F}) + C\cdot V_{\max}^2/(n\cdot \epsilon) + C'\cdot V_{\max}\cdot \delta + 2\beta\cdot G^2,
\end{equation}
where $C$ and $C'$ are two absolute constants and are defined as \begin{equation}
    \omega(\mathcal{F}) = \sup_{g\in\mathcal{F}}\inf_{f\in\mathcal{F}} \lVert f - Tg \rVert_{\sigma}.
\end{equation}
\end{theorem}

This result suggests the convergence rate of the previous result without PEER is maintained while PEER only adds a constant term $2\beta\cdot G^2$.

\begin{figure*}[!ht]
\centering
\hspace{-0.1in}
\subcaptionbox{\label{fig: sub gridworld}Grid World}
{\includegraphics[width=0.25\textwidth]{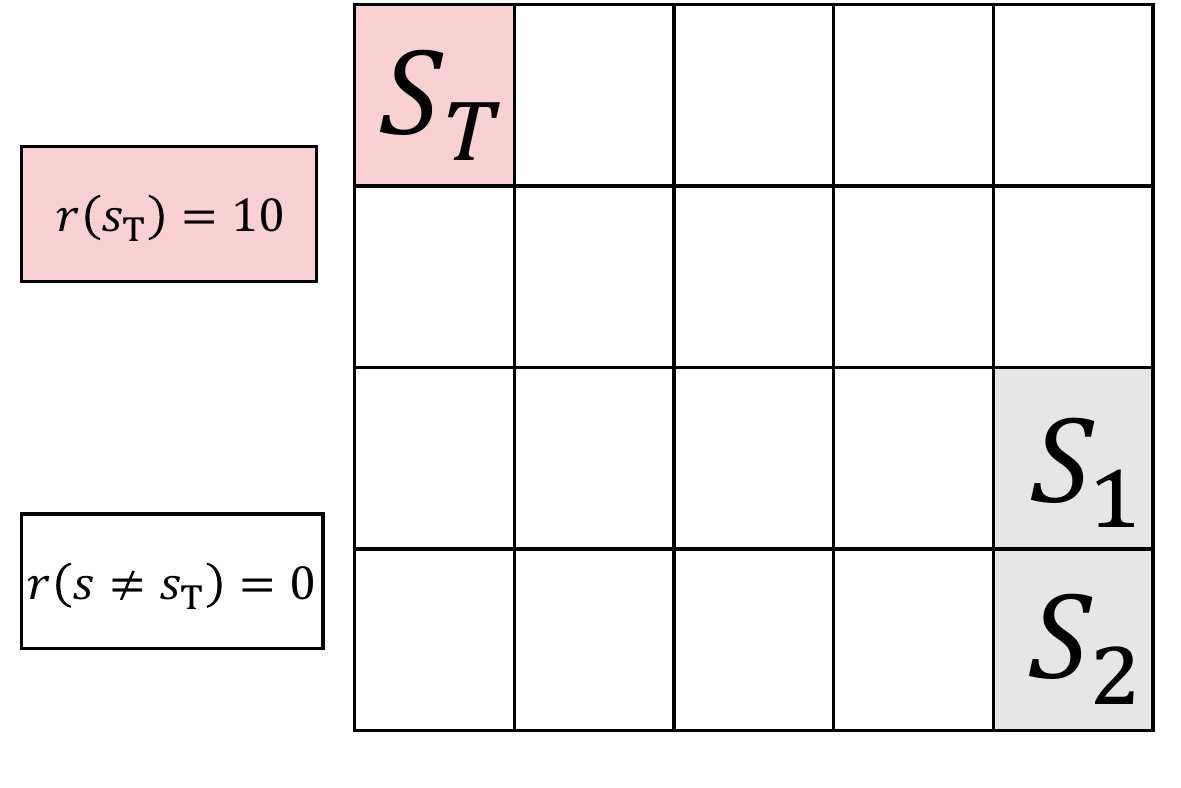}}
\hspace{-0.1in}
\subcaptionbox{\label{fig: sub similarity}Cosin similarity}
{\includegraphics[width=0.25\textwidth]{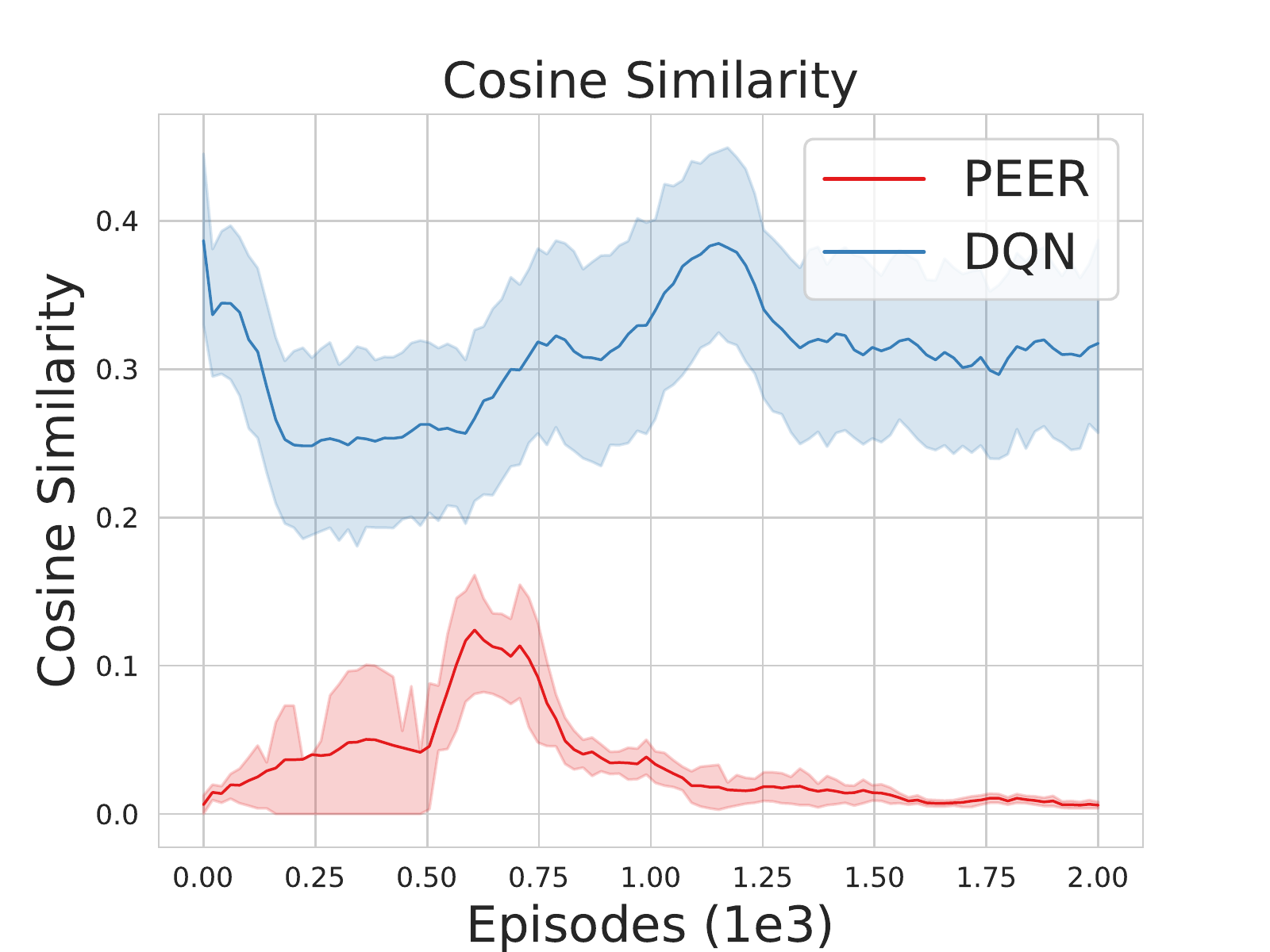}}
\hspace{-0.1in}
\subcaptionbox{\label{fig: sub q value difference}Q value difference}
{\includegraphics[width=0.25\textwidth]{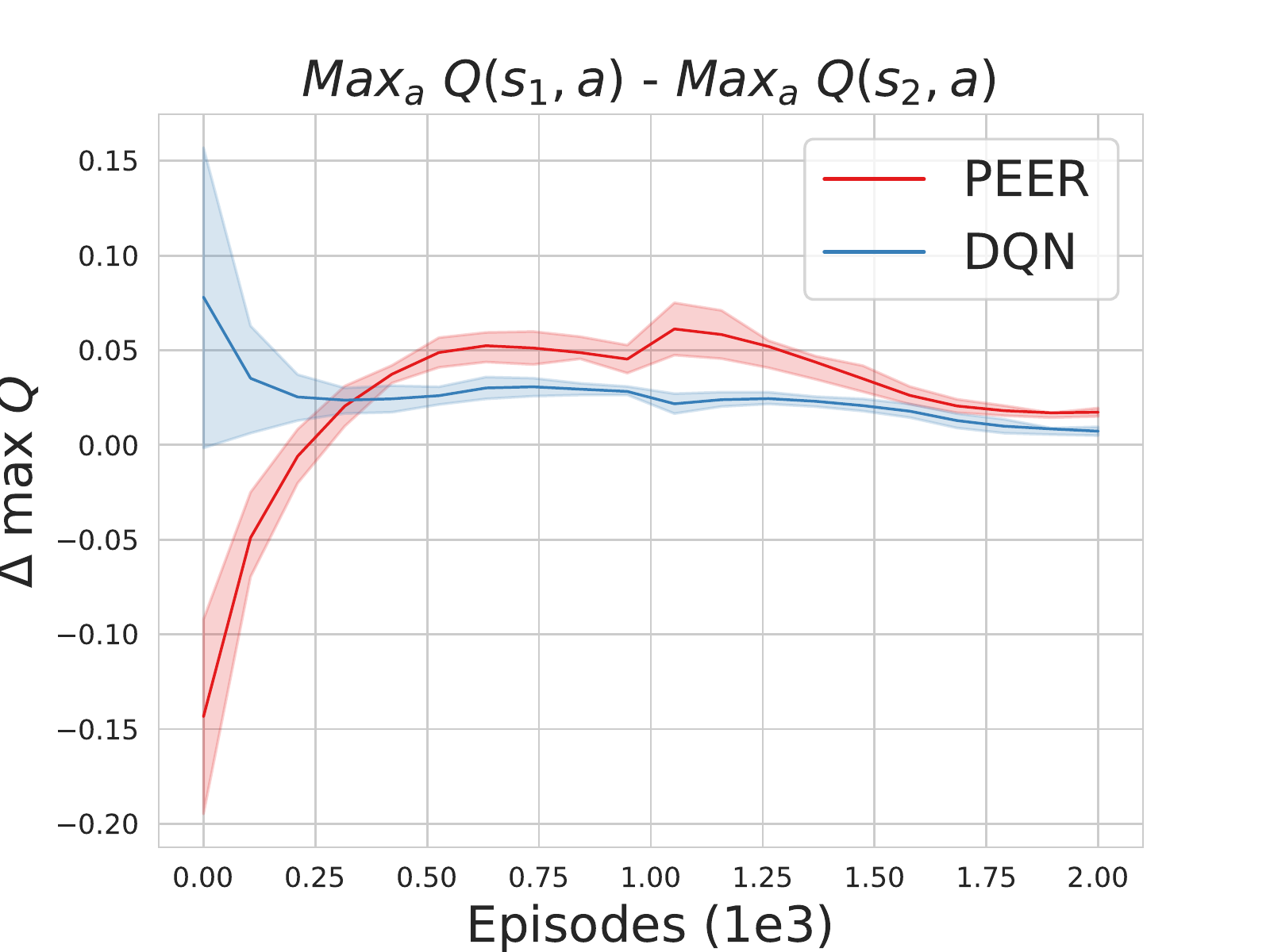}}
\hspace{-0.1in}
\subcaptionbox{\label{fig: sub steps}Steps}
{\includegraphics[width=0.25\textwidth]{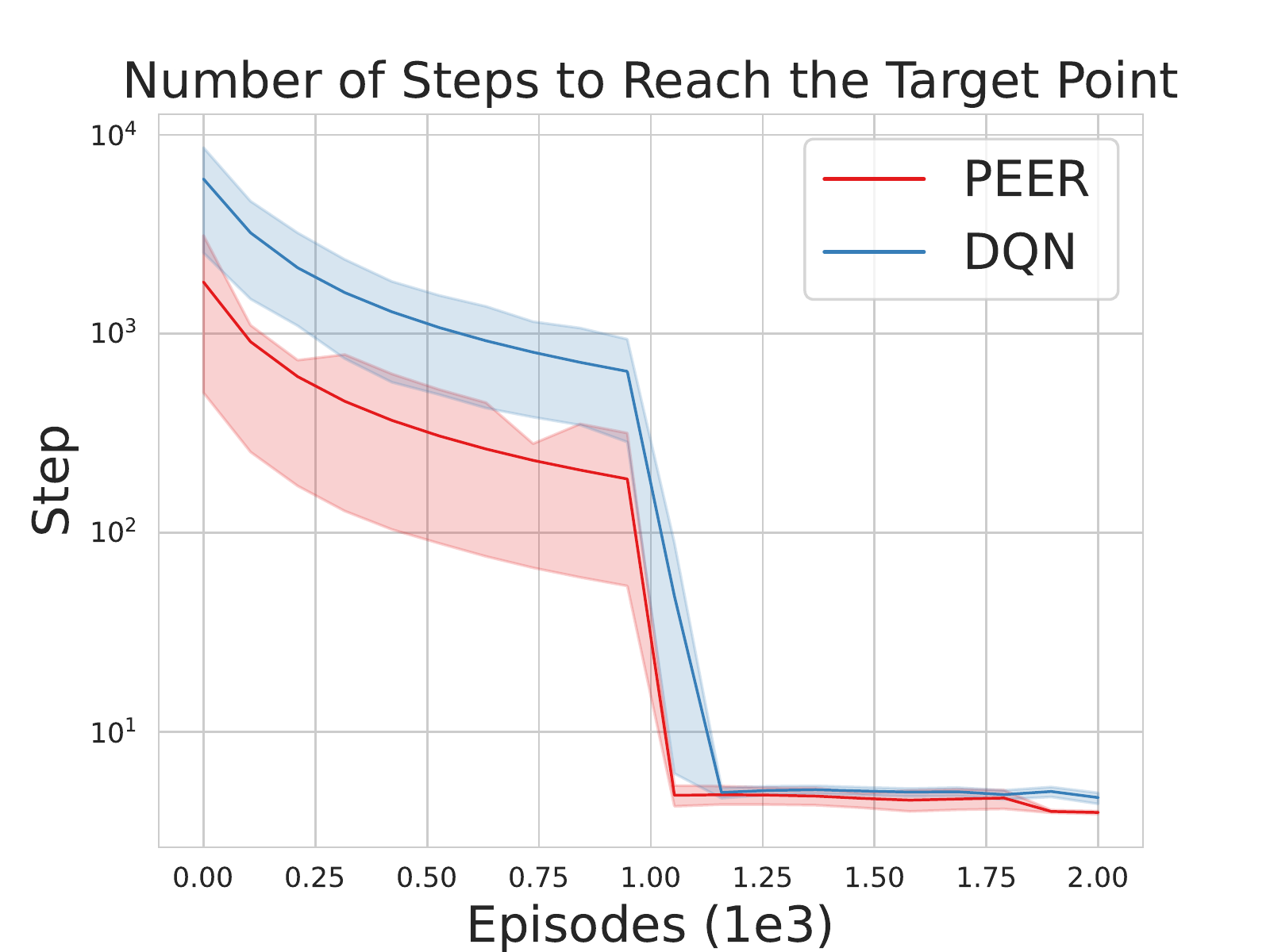}}
\caption{\label{fig: grid world exp}Experiments on the grid world. (a) Grid world. We are interested in the values of $S_1$ and $S_2$. The policies are supposed to be able to differentiate the values of states $S_1$ and $S_2$ and tell that $S_1$ has the higher value. (b) The cosine similarity between the representation of the $Q$-network and its target. PEER (combined with DQN) effectively alleviates the similarity. (c) The difference between the max Q values of $S_1$ and $S_2$. Compared with DQN, PEER is able to better distinguish the Q values of two adjacent states by differentiating the representation of the $Q$-network and its target. (d) The number of steps to reach the $S_T$. PEER is better than DQN. The results are reported over five seeds and the shaded area represents a half standard deviation.}
\end{figure*}

\subsection{A toy example}

We provide a toy example of a grid world (\cref{fig: sub gridworld}) to intuitively illustrate the functionality of PEER. In the grid world, the red state is the only state with a reward. We are interested in the Q values of the two gray states $S_1$ and $S_2$. Ideally, the policy is supposed to differentiate between those two adjacent states and accurately determine that $S_1$ has a higher value. We show the similarity between the representation of the $Q$-network and its target and the difference between the maximum Q value of two adjacent states in \cref{fig: grid world exp}. As shown in the \cref{fig: sub similarity}, PEER is able to better distinguish the representation of the $Q$-network and its target compared to DQN. Thus, PEER is better equipped to differentiate between the maximum Q values of the two nearby states $S_1$ and $S_2$~(\cref{fig: sub q value difference}). Consequently, PEER yields a better policy~(\cref{fig: sub steps}).
 \section{Experiments}\label{sec: exp}

We perform comprehensive experiments to thoroughly assess PEER. Specifically, we evaluate (i) \textbf{performance} by measuring average returns on multiple suites when combined with different backbone algorithms; (ii) \textbf{sample efficiency} by comparing it with other algorithms at fixed timesteps; (iii) \textbf{compatibility}: whether PEER can be combined with other DRL algorithms such as off-policy methods TD3, contrastive unsupervised representation learning (CURL), and DrQ through extensive experiments; and (iv) the distinguishable representation property of PEER. To achieve this, we couple {PEER} with three representative DRL algorithms {TD3}~\cite{td3}, CURL~\cite{laskin2020curl}, and DrQ~\cite{drq}, and perform extensive experiments on four suites, namely, PyBullet, MuJoCo, DMControl, and Atari. PEER's simplicity and ease of implementation are noteworthy, requiring only one line of code. We deliberately avoid using any engineering tricks that could potentially enhance its performance. This decision is made to ensure that the reproducibility crisis, which has been a growing concern in the field of DRL \cite{andrychowicz2021matters, drl_matters}, is not further exacerbated. We also maintain a fixed newly introduced hyper-parameter $\beta$ across all the experiments to achieve fair comparisons. Additional experimental details can be found in the Appendix.

\subsection{Experimental settings} \label{sec: experimental settings}

\textbf{Hyper-parameters.} 
We introduce only one additional hyper-parameter $\beta$ to control the magnitude of regularization effectiveness of {PEER}. We report all results using a fixed value of $\beta=5e-4$. It is important to note that PEER may benefit from selecting a $\beta$ value that is better suited to a specific environment.

\textbf{Random seeds.}
To ensure the reproducibility of our experiments, we evaluate each tested algorithm using ten fixed random seeds unless otherwise specified. Moreover, we maintain fixed seeds across all experiments, including those used in PyTorch, Numpy, Gym, and CUDA packages.

\textbf{Environments.}
To evaluate PEER, we use both state-based (state represented as vectors containing sensor information such as velocity, position, friction, etc) PyBullet and MuJoCo, and pixel-based (state represented as images) DMControl and Atari suites. The action space of PyBullet, MuJoCo, and DMControl is continuous, while that of Atari is discrete. Through the four experimental suites, we can check the performance and sample efficiency of PEER. We use the Gym \cite{gym} library for the interactive protocol. For PyBullet and MuJoCo suites, we run each tested algorithm for 1 million timesteps and evaluate the average return of the algorithm every 5k timesteps over ten episodes. For DMControl and Atari suites, following the commonly used experimental setting \cite{laskin2020curl, drq}, we measure the performance and sample-efficiency of tested algorithms at 100k and 500k environment timesteps, resulting in DMControl100k, DMControl500k, and Atari100k settings.

\begin{table}[!htbp]
\setlength\tabcolsep{3pt}
    \centering
    \scalebox{0.88}{
    \begin{tabular}{lllll}
    \toprule
    Algorithm & Ant &	HalfCheetah	&Hopper	&Walker2D \\
    \midrule
PEER &\colorbox{mine}{3003} $\pm$ {\footnotesize204} &\colorbox{mine}{2494} $\pm$ {\footnotesize276} &\colorbox{mine}{2106} $\pm$ {\footnotesize164} &\colorbox{mine}{1966} $\pm$ {\footnotesize58}\\
TD3 &\;2731 $\pm$ {\footnotesize278} &\;2359 $\pm$ {\footnotesize229} &\;1798 $\pm$ {\footnotesize471} &\;1646 $\pm$ {\footnotesize314}\\
METD3 &\;2601 $\pm$ {\footnotesize246} &\;2345 $\pm$ {\footnotesize151} &\;1929 $\pm$ {\footnotesize351} &\;1901 $\pm$ {\footnotesize111}\\
SAC &\;2561 $\pm$ {\footnotesize146} &\;1675 $\pm$ {\footnotesize567} &\;1984 $\pm$ {\footnotesize103} &\;1716 $\pm$ {\footnotesize30}\\
PPO2 &\;539 $\pm$ {\footnotesize25} &\;397 $\pm$ {\footnotesize63} &\;403 $\pm$ {\footnotesize70} &\;390 $\pm$ {\footnotesize106}\\
TRPO &\;693 $\pm$ {\footnotesize74} &\;639 $\pm$ {\footnotesize154} &\;1140 $\pm$ {\footnotesize469} &\;496 $\pm$ {\footnotesize206}\\
        \bottomrule
    \end{tabular}}
  \caption{\label{table: exp bullet}The average return of the last ten evaluations over ten random seeds. PEER (coupled with TD3) outperforms all the compared algorithms, which shows that the PEER loss works in state-based environments. The best score is marked with \colorbox{mine}{colorbox.} $\pm$ corresponds to a standard deviation over trials. }
\end{table}

\begin{table*}[!htbp]        
    \centering
    \scalebox{0.95}{
    \begin{tabular}{l|l|llllll|l}
    \toprule
        500K Step Scores & State SAC & PlaNet & Dreamer & SAC+AE & DrQ & DrQ-v2 & CURL & PEER \\ 
        \midrule
Finger, Spin &\;923 $\pm$ {\footnotesize21} &\;561 $\pm$ {\footnotesize284} &\;796 $\pm$ {\footnotesize183} &\;884 $\pm$ {\footnotesize128} &\colorbox{mine}{938} $\pm$ {\footnotesize 103} &\;789 $\pm$ {\footnotesize124} &\;926 $\pm$ {\footnotesize45} &\;864 $\pm$ {\footnotesize160}\\
Cartpole, Swingup &\;848 $\pm$ {\footnotesize15} &\;475 $\pm$ {\footnotesize71} &\;762 $\pm$ {\footnotesize27} &\;735 $\pm$ {\footnotesize63} &\colorbox{mine}{868} $\pm$ {\footnotesize 10} &\;845 $\pm$ {\footnotesize18} &\;841 $\pm$ {\footnotesize45} &\;866 $\pm$ {\footnotesize17}\\
Reacher, Easy &\;923 $\pm$ {\footnotesize24} &\;210 $\pm$ {\footnotesize390} &\;793 $\pm$ {\footnotesize164} &\;627 $\pm$ {\footnotesize58} &\;942 $\pm$ {\footnotesize71} &\;748 $\pm$ {\footnotesize229} &\;929 $\pm$ {\footnotesize44} &\colorbox{mine}{980} $\pm$ {\footnotesize 3}\\
Cheetah, run &\;795 $\pm$ {\footnotesize30} &\;305 $\pm$ {\footnotesize131} &\;570 $\pm$ {\footnotesize253} &\;550 $\pm$ {\footnotesize34} &\;660 $\pm$ {\footnotesize96} &\;607 $\pm$ {\footnotesize32} &\;518 $\pm$ {\footnotesize28} &\colorbox{mine}{732} $\pm$ {\footnotesize 41}\\
Walker, Walk &\;948 $\pm$ {\footnotesize54} &\;351 $\pm$ {\footnotesize58} &\;897 $\pm$ {\footnotesize49} &\;847 $\pm$ {\footnotesize48} &\;921 $\pm$ {\footnotesize45} &\;696 $\pm$ {\footnotesize370} &\;902 $\pm$ {\footnotesize43} &\colorbox{mine}{946} $\pm$ {\footnotesize 17}\\
Ball\_in\_cup, Catch &\;974 $\pm$ {\footnotesize33} &\;460 $\pm$ {\footnotesize380} &\;879 $\pm$ {\footnotesize87} &\;794 $\pm$ {\footnotesize58} &\;963 $\pm$ {\footnotesize9} &\;844 $\pm$ {\footnotesize174} &\;959 $\pm$ {\footnotesize27} &\colorbox{mine}{973} $\pm$ {\footnotesize 5}\\\midrule
100K Step Scores\\\midrule
Finger, Spin &\;811 $\pm$ {\footnotesize46} &\;136 $\pm$ {\footnotesize216} &\;341 $\pm$ {\footnotesize70} &\;740 $\pm$ {\footnotesize64} &\colorbox{mine}{901} $\pm$ {\footnotesize 104} &\;325 $\pm$ {\footnotesize292} &\;767 $\pm$ {\footnotesize56} &\;820 $\pm$ {\footnotesize166}\\
Cartpole, Swingup &\;835 $\pm$ {\footnotesize22} &\;297 $\pm$ {\footnotesize39} &\;326 $\pm$ {\footnotesize27} &\;311 $\pm$ {\footnotesize11} &\;759 $\pm$ {\footnotesize92} &\;677 $\pm$ {\footnotesize214} &\;582 $\pm$ {\footnotesize146} &\colorbox{mine}{863} $\pm$ {\footnotesize 17}\\
Reacher, Easy &\;746 $\pm$ {\footnotesize25} &\;20 $\pm$ {\footnotesize50} &\;314 $\pm$ {\footnotesize155} &\;274 $\pm$ {\footnotesize14} &\;601 $\pm$ {\footnotesize213} &\;256 $\pm$ {\footnotesize145} &\;538 $\pm$ {\footnotesize233} &\colorbox{mine}{961} $\pm$ {\footnotesize 28}\\
Cheetah, run &\;616 $\pm$ {\footnotesize18} &\;138 $\pm$ {\footnotesize88} &\;235 $\pm$ {\footnotesize137} &\;267 $\pm$ {\footnotesize24} &\;344 $\pm$ {\footnotesize67} &\;273 $\pm$ {\footnotesize130} &\;299 $\pm$ {\footnotesize48} &\colorbox{mine}{499} $\pm$ {\footnotesize 74}\\
Walker, Walk &\;891 $\pm$ {\footnotesize82} &\;224 $\pm$ {\footnotesize48} &\;277 $\pm$ {\footnotesize12} &\;394 $\pm$ {\footnotesize22} &\;612 $\pm$ {\footnotesize164} &\;171 $\pm$ {\footnotesize160} &\;403 $\pm$ {\footnotesize24} &\colorbox{mine}{714} $\pm$ {\footnotesize 148}\\
Ball\_in\_cup, Catch &\;746 $\pm$ {\footnotesize91} &\;0 $\pm$ {\footnotesize0} &\;246 $\pm$ {\footnotesize174} &\;391 $\pm$ {\footnotesize82} &\;913 $\pm$ {\footnotesize53} &\;359 $\pm$ {\footnotesize228} &\;769 $\pm$ {\footnotesize43} &\colorbox{mine}{968} $\pm$ {\footnotesize 7}\\
        \bottomrule
    \end{tabular}}
    \caption{\label{table: exp dm control}Scores achieved by {PEER} (coupled with CURL) on DMControl continuous control suite. PEER achieves superior performance on the majority (\textbf{9} out of \textbf{12}) tasks. And PEER also outperforms its backbone algorithm CURL on \textbf{11} out of \textbf{12} tasks by a large margin. The best score is marked with \colorbox{mine}{colorbox.} $\pm$ corresponds to a standard deviation over trials. }
\end{table*}

\begin{table*}[!htb]
    \centering
    \scalebox{0.92}{
    
        \begin{tabular}{l|ll|lllll|ll}
        \toprule
        Game & Human & Random & OTRainbow & Eff. Rainbow & Eff. DQN & DrQ & CURL & PEER{\tiny +CURL} & PEER{\tiny +DrQ} \\ 
        \midrule
        Alien &\;7127.7 &\;227.8 &\;824.7 &\;739.9 &\;558.1 &\;771.2 &\;558.2 & \colorbox{mine}{1218.9}  &\;712.7\\
Amidar &\;1719.5 &\;5.8 &\;82.8 & \colorbox{mine}{188.6}  &\;63.7 &\;102.8 &\;142.1 &\;185.2 &\;163.1\\
Assault &\;742 &\;222.4 &\;351.9 &\;431.2 &\;589.5 &\;452.4 &\;600.6 &\;631.2 & \colorbox{mine}{721} \\ 
        Asterix  &\;8503.3 &\;210 &\;628.5 &\;470.8 &\;341.9 &\;603.5 &\;734.5 &\;834.5 & \colorbox{mine}{918.2} \\ 
        BankHeist  &\;753.1 &\;14.2 & \colorbox{mine}{182.1}  &\;51 &\;74 &\;168.9 &\;131.6 &\;78.6 &\;12.7\\
BattleZone &\;37187.5 &\;2360 &\;4060.6 &\;10124.6 &\;4760.8 &\;12954 &\;14870 & \colorbox{mine}{15727.3}  &\;5000\\
Boxing &\;12.1 &\;0.1 &\;2.5 &\;0.2 &\;-1.8 &\;6 &\;1.2 &\;3.7 & \colorbox{mine}{14.5} \\ 
        Breakout  &\;30.5 &\;1.7 &\;9.8 &\;1.9 &\;7.3 & \colorbox{mine}{16.1}  &\;4.9 &\;3.9 &\;8.5\\
ChopperCommand &\;7387.8 &\;811 &\;1033.3 &\;861.8 &\;624..4 &\;780.3 &\;1058.5 & \colorbox{mine}{1451.8}  &\;1233.6\\
CrazyClimber &\;35829.4 &\;10780.5 & \colorbox{mine}{21327.8}  &\;16185.3 &\;5430.6 &\;20516.5 &\;12146.5 &\;18922.7 &\;18154.5\\
DemonAttack &\;1971 &\;152.1 &\;711.8 &\;508 &\;403.5 &\;1113.4 &\;817.6 &\;742.9 & \colorbox{mine}{1236.7} \\ 
        Freeway  &\;29.6 &\;0 &\;25 &\;27.9 &\;3.7 &\;9.8 &\;26.7 & \colorbox{mine}{30.4}  &\;21.2\\
Frostbite &\;4334.7 &\;65.2 &\;231.6 &\;866.8 &\;202.9 &\;331.1 &\;1181.3 & \colorbox{mine}{2151}  &\;537.4\\
Gopher &\;2412.5 &\;257.6 & \colorbox{mine}{778}  &\;349.5 &\;320.8 &\;636.3 &\;669.3 &\;583.6 &\;681.8\\
Hero &\;30826.4 &\;1027 &\;6458.8 &\;6857 &\;2200.1 &\;3736.3 &\;6279.3 & \colorbox{mine}{7499.9}  &\;3953.2\\
Jamesbond &\;302.8 &\;29 &\;112.3 &\;301.6 &\;133.2 &\;236 & \colorbox{mine}{471}  &\;414.1 &\;213.6\\
Kangaroo &\;3035 &\;52 &\;605.4 &\;779.3 &\;448.6 &\;940.6 &\;872.5 & \colorbox{mine}{1148.2}  &\;663.6\\
Krull &\;2665.5 &\;1598 &\;3277.9 &\;2851.5 &\;2999 &\;4018.1 &\;4229.6 &\;4116.1 & \colorbox{mine}{5444.7} \\ 
        KungFuMaster  &\;22736.3 &\;258.5 &\;5722.2 &\;14346.1 &\;2020.9 &\;9111 &\;14307.8 & \colorbox{mine}{15439.1}  &\;4090.9\\
MsPacman &\;6951.6 &\;307.3 &\;941.9 &\;1204.1 &\;872 &\;960.5 &\;1465.5 & \colorbox{mine}{1768.4}  &\;1027.3\\
Pong &\;14.6 &\;-20.7 & \colorbox{mine}{1.3}  &\;-19.3 &\;-19.4 &\;-8.5 &\;-16.5 &\;-9.5 &\;-18.2\\
PrivateEye &\;69571.3 &\;24.9 &\;100 &\;97.8 &\;351.3 &\;-13.6 &\;218.4 & \colorbox{mine}{3207.7}  &\;8.2\\
Qbert &\;13455 &\;163.9 &\;509.3 &\;1152.9 &\;627.5 &\;854.4 &\;1042.4 & \colorbox{mine}{2197.7}  &\;913.6\\
RoadRunner &\;7845 &\;11.5 &\;2696.7 &\;9600 &\;1491.9 &\;8895.1 &\;5661 & \colorbox{mine}{10697.3}  &\;6900\\
Seaquest &\;42054.7 &\;68.4 &\;286.9 &\;354.1 &\;240.1 &\;301.2 &\;384.5 & \colorbox{mine}{538.5}  &\;409.6\\
UpNDown &\;11693.2 &\;533.4 &\;2847.6 &\;2877.4 &\;2901.7 &\;3180.8 &\;2955.2 &\;6813.5 & \colorbox{mine}{7680.9} \\ 
        \bottomrule
    \end{tabular}
    }
\caption{\label{table: atari exp results}Scores achieved by {PEER} (coupled with CURL and DrQ) and baselines on Atari. {PEER} achieves state-of-the-art performance on \textbf{19} out of \textbf{26} games. PEER implemented on top of CURL/DrQ improves over CURL/DrQ on \textbf{19/16} out of \textbf{26} games. Algorithms combined with PEER are reported across 10 random seeds. We also see that PEER achieves superhuman performance on Boxing, Freeway, JamesBond, Krull, and RoadRunner games.  The best score is marked with \colorbox{mine}{colorbox.}}
\vspace{-10pt}
\end{table*}

\textbf{Baselines.}
We first evaluate {PEER} on the state-based PyBullet suite, using TD3, SAC \cite{sac}, {TRPO} \cite{trpo}, {PPO}\cite{ppo} as our baselines for their superior performance. And we couple PEER with TD3 on PyBullet and MuJoCo experiments. PEER works as preventing the similarity between the representation of the Q-network and its target. Another relevant baseline is MEPG \cite{mepg}, which enforces a dropout operator on both the Q-network and its target. Dropout operator \cite{warde-farleyEmpiricalAnalysisDropout2013,srivastava2014dropout} is generally believed to prevent feature co-adaptation, which is close to what PEER achieves. Therefore, we use MEPG combined with TD3 as a baseline, denoted as METD3.
We employ the authors' TD3 implementation, along with the public implementation \cite{pytorch_sac} for SAC,  and the Baselines codebase \cite{baselines} for TRPO and PPO. We opt to adhere to the authors' recommended default hyper-parameters for all considered algorithms.

Then we evaluate {PEER} on the pixel-based DMControl and Atari suites, combining it with CURL and DrQ algorithms. For DMControl, we select (i) PlaNet \cite{planet} and (ii) Dreamer \cite{dreamer}, which learn a world model in latent space and execute planning; (iii) SAC+AE \cite{sacae} uses VAE \cite{vae} and a regularized encoder; (iv) {CURL} using contrastive unsupervised learning to extract high-level features from raw pixels; (v) DrQ \cite{drq} and (vi) DrQ-v2 \cite{drqv2}, which adopt data augmentation technique; and (vii) state-based {SAC}. For Atari suite, we select {CURL}, OTRainbow \cite{otrainbow}, {Efficient Rainbow} \cite{effrainbow}, {Efficient DQN} \cite{drq}, and {DrQ} as baselines.
\subsection{Results} \label{sec: experimental results}

\begin{figure}[!htbp]
	\begin{center}
	\includegraphics[width=\linewidth]{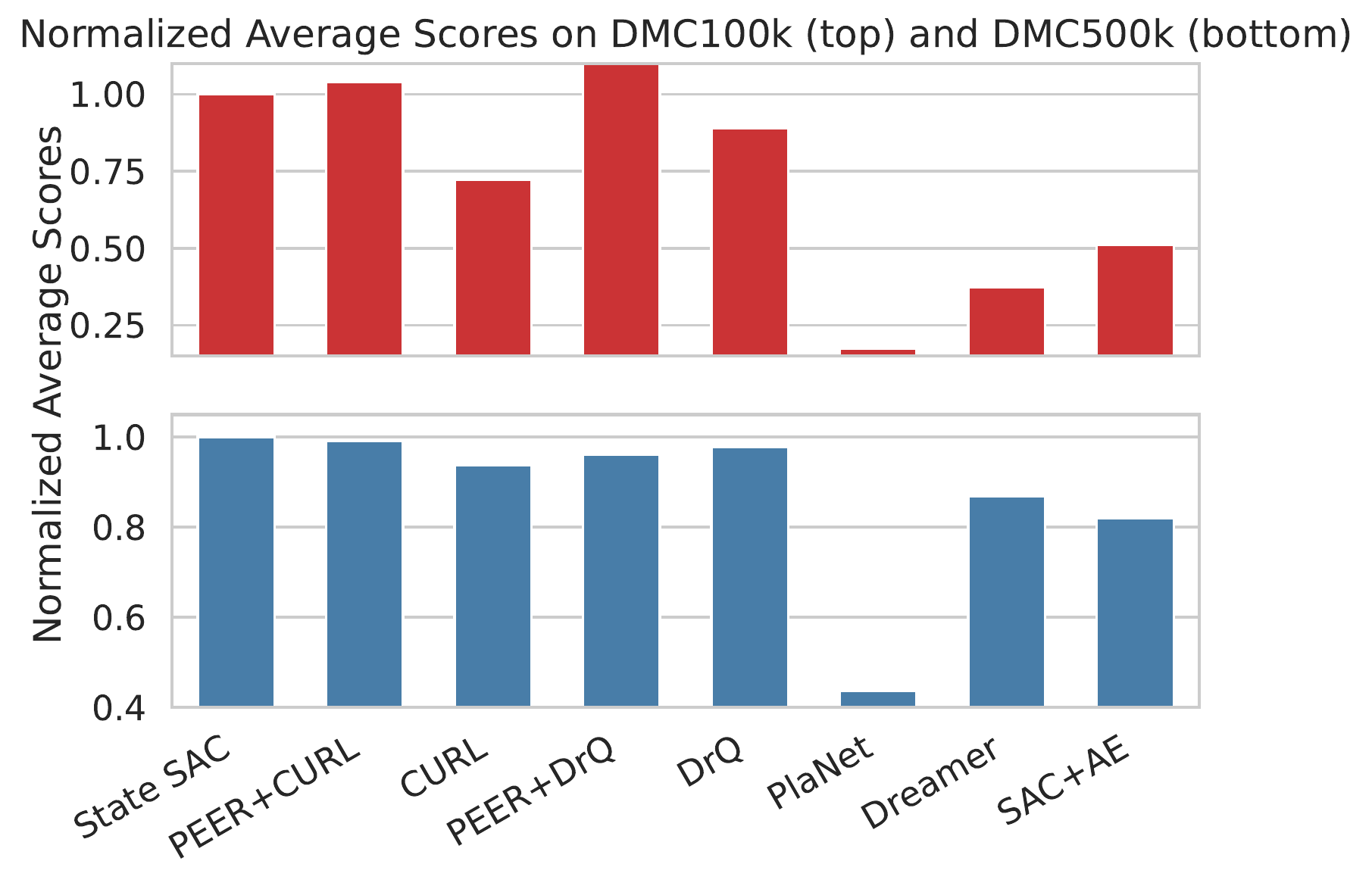} 
	\end{center}
	\caption{\label{fig: normalized score}The normalized average scores on the DMControl suite. We normalize the average score of the tested algorithm by the average scores of State SAC. On the DMC100k benchmark, PEER (coupled with CURL and DrQ) outperforms all the compared algorithms including State SAC.}
 \vspace{-0.2in}
\end{figure}

\textbf{PyBullet.}
We first evaluate the PEER loss in the state-based suite PyBullet. We show the final ten evaluation average return in \cref{table: exp bullet}. The results show (i) {PEER} coupled with {TD3} outperforms {TD3} on all environments. Furthermore, (ii) {PEER} also surpasses all other tested algorithms. The superior performance of {PEER} shows that the PEER loss can improve the empirical performance of the off-policy model-free algorithm that does not adopt any representation learning techniques.

\textbf{DMControl.}
We then perform experiments on the DMControl suite. Specifically, we couple {PEER} with the {CURL} and DrQ algorithms, and run them in DMControl500k and DMControl100k settings. The results are shown in \cref{table: exp dm control} and \cref{fig: normalized score}. The key findings are as follows:
(i) On the DMControll500k setting, {PEER} coupled with {CURL} outperforms its backbone by a large margin on \textbf{5} out of \textbf{6} tasks, which shows the proposed PEER does improve the performance of its backbone algorithm. And the performance improvement on DMC500k shows that the PEER is beneficial for contrastive unsupervised learning (CURL).  
(ii) On the DMControl100k setting, The PEER (coupled with CURL) outperforms its backbone CURL on \textbf{6} out of \textbf{6} tasks. Besides, results in \cref{fig: normalized score} demonstrates that PEER (coupled with DrQ) improves the sample efficiency of DrQ by a large margin. Overall, the {PEER} achieves SOTA performance on \textbf{11} out of \textbf{12} tasks. 
(iii) PEER outperforms State SAC on \textbf{6} out of \textbf{12} tasks. In \cref{fig: normalized score}, we computed the average score of the tested algorithm on all environments under DMControl100k and DMControl500k settings, normalized by the score of State SAC. The results in DMControl100k show that {PEER} (combined with CURL and DrQ) is more sample-efficient than State SAC and other algorithms. And in the DMControl500k suite, the sample efficiency of {PEER} matches that of State SAC. These results illustrate that {PEER} remarkably improves the empirical \textbf{performance} and \textbf{sample efficiency} of the backbone algorithms TD3, CURL, and DrQ.

\textbf{Atari.} We present Atari100k experiments in \cref{table: atari exp results}, which show that (i) PEER achieves state-of-the-art performance on \textbf{19} out of \textbf{26} environments in given timesteps. And (ii) {PEER} implemented on top of CURL/DrQ improves over CURL on \textbf{19/16} out of \textbf{26} games. (iii) We also see that {PEER} achieves superhuman performance on Boxing, Freeway, JamesBond, Krull, and RoadRunner games. The empirical results demonstrate that PEER dramatically improves the \textbf{sample efficiency} of backbone algorithms.

\subsection{Compatibility}
Given a fixed hyper-parameter $\beta=5e-4$, PEER, as a plug-in, outperforms its backbone CURL algorithm on DMControl (\textbf{11} out of \textbf{12}) and Atari (\textbf{19} out of \textbf{26}) tasks, which shows that PEER is able to be incorporated into contrastive learn-based and data-augmentation methods to improve performance and sample efficiency. Besides, PEER coupled with DrQ also outperforms DrQ on \textbf{16} out of \textbf{26} environments. For state-based continuous control tasks, PEER (coupled with TD3) surpasses all the tested algorithms. These facts show that the compatibility of the PEER loss is remarkable and the PEER can be extended to incorporate more DRL algorithms. Theoretically, PEER tackles representation learning from the perspective of the representation of the Bellman equation, which is contrasting with other representation learning methods in DRL. Empirically, the performance of PEER loss coupled with representation learning DRL methods is better than that of backbone methods, which means that the PEER loss is orthogonal to other representation learning methods in DRL. Thus the compatibility of the PEER loss is auspicious.

\subsection{PEER preserves distinguishable representation property}
To validate whether the  PEER regularizer preserves the distinguishable representation property or not, we measure the distinguishable representation discrepancy of the action value network and its target in PEER following \cref{sec: theoretical analysis}. We show the experimental results in \cref{fig: alpha visuallization}. The results show (i) TD3 and PEER (based on TD3) agents do enjoy the distinguishable representation property but that of PEER is more evident, which reveals the performance gain of PEER in this setting comes from better distinguishable representation property. (ii) The CURL agent does not maintain the distinguishable representation property on the tested DMControl suite, which negatively affects the model performance. And (iii) PEER also enjoys the distinguishable representation property on the pixel-based environment DMControl. Thus the performance of PEER is naturally improved due to the property being desirable. (iv) Combined with the performance improvement shown in performance experiments (\cref{sec: experimental results}), PEER does improve the sample efficiency and performance by preserving the distinguishable representation property of the $Q$-network and its target.
 \section{Related Work}

Representation learning \cite{bert,moco,mucov2,simclr,cpc_o,cpc} aims at learning good or rich features of data. Such learned representation may be helpful for downstream tasks. The fields of natural language processing and computer vision have benefited from such techniques.
It \cite{ghosh2020representations,laskin2020curl,drq,drqv2,rad,cpc} is widely acknowledged that good representations are conducive to improving the performance of DRL~\cite{ddqn, c51, ddpg, td3, sac, ppo, mepg, wd3}. Recent works \cite{unreal,pbl,cpc,laskin2020curl,lyle2021effect, drq, drqv2} used self-supervised, unsupervised, contrastive representation learning, and data-augmentation approaches to improve the performance of DRL methods. There emerged several works studying representation learning from a geometric view~\cite{bellemare2019geometric, dadashi2019value, eysenbach2021information}. \citet{lyle2021understanding} explicitly considers representation capacity, which imposes regularization on the neural networks. The closest work to PEER is DR3 \cite{dr3}, which proposed a regularizer to tackle the implicit regularization effect of stochastic gradient descent in an offline RL setting. Despite the coincidental synergistic use of dot product form in PEER and DR3, we prove its necessity with different derivations and motivations.  Our upper bound is rigorously derived for the representation of critic and its \textit{target network}. But DR3 does not specify which network they select. DR3 only works in the offline setting, while PEER is, in theory, applicable to both online and offline settings.

Our work differentiates from previous works from the following three perspectives. First, PEER tackles representation learning by preserving the distinguishable representation property instead of learning good representations with the help of auxiliary tasks. The convergence rate guarantee of PEER can be proved. Second, the experiments show that PEER is orthogonal to existing representation learning methods in DRL. Third, PEER is also suitable for environments based on both states and pixels while other representation learning methods in DRL are almost only performed on pixel-based suites.
  \section{Conclusion}
\label{sec: limitations-conclusion}

In this work, we initiated an investigation of the representation property of the $Q$-network and its target and found that they ought to satisfy an inherently favorable distinguishable representation property. Then illustrative experimental results demonstrate that deep RL agents may not be able to preserve this desirable property during training. As a solution to maintain the distinguishable representation property of deep RL agents, we propose a straightforward yet effective regularizer, PEER, and provide a convergence rate guarantee. Our extensive empirical results indicate that, with a fixed hyper-parameter, PEER exhibits superior improvements across all suites. Furthermore, the observed performance improvements are attributable to the preservation of distinguishable representational properties. Previous work on value function representations often makes the representations of adjacent moments similar, which is thought to maintain some smoothing. However, our work demonstrates that there is an upper bound on this similarity and that minimizing it can preserve a beneficial property, leading to sampling efficiency and performance improvements. In some cases, the performance of PEER is commensurate with that of State SAC. However, we leave the task of further analyzing the reasons for this parity for future research. To the best of our knowledge, PEER is the first work to study the inherent representation property of $Q$-network and its target. We believe that our work sheds light on the nature of the inherent representational properties arising from combining the parameterization tool neural networks and RL.

\section*{Acknowledgements} 
We thank anonymous reviewers and Area Chairs for their fair evaluations and professional comments. Qiang He thanks Yuxun Qu for his sincere help. This work was done when Qiang He was with Insititute of Automation, Chinese Academy of Sciences. This work was supported by the National Key Research and Development Program of China (2021YFC2800501).



\clearpage
{\small
\bibliographystyle{unsrtnat}
\bibliography{ref}
}

\clearpage
\onecolumn
 \section{Appendix: Theoretical Derivation}\label{app sec proofs}
\label{appendix: proof}

\begin{theorem}[Distinguishable Representation Property]\label{app theorem: representation gap}
	The similarity (defined as inner product $\langle \cdot, \cdot \rangle$) between normalized representations $\Phi (s,a; \Theta_{+}) $ of the $Q$-network and $ \mathbb{E}_{s',a'}\Phi(s',a'; \Theta'_{+})$ satisfies 

\begin{equation}\label{app thm: representation gap}
   \langle \Phi(s,a; \Theta_+), \mathbb{E}_{s',a'} \Phi(s',a'; \Theta'_{+}) \rangle  \leq \frac{1}{\gamma} - \frac{r(s,a)^2}{2\lVert \Theta_{-1} \rVert^2},
\end{equation}
where $s, a$ and $\Theta_+$ are state, action, and parameters of the $Q$-network except for those of the last layer. While $s', a', \Theta'_{+}$ are the state, action at the next time step, and parameters of the target $Q$-network except for those of the last layer. And $\Theta_{-1}$ is the parameters of the last layer of $Q$-network. 
\end{theorem}
	
\begin{proof}\label{proof: size of representation gap}
Following \cref{eq: def representation}, the Bellman Equation \cref{eq: Q-learning} can be rewritten as 
\begin{equation}
\begin{split}
    \Phi(s,a; \Theta_+)^\top \Theta_{-1} &= r(s,a) + \gamma \mathbb{E}_{s',a'} \Phi(s',a'; \Theta'_{+}) ^\top \Theta'_{-1}.
\end{split}
\end{equation}
	After the policy evaluation converges, $\Theta 
 \text{ and }\Theta'$ satisfy $\Theta = \Theta'.$ Thus we have

\begin{equation}
\begin{split}
    (\Phi(s,a; \Theta_+)^\top - \gamma \mathbb{E}_{s',a'} \Phi(s',a'; \Theta'_{+}) ^\top) \Theta_{-1} &= r(s,a) \\
    \lVert (\Phi(s,a; \Theta_+)^\top - \gamma \mathbb{E}_{s',a'} \Phi(s',a'; \Theta'_{+}) ^\top) \Theta_{-1} \rVert &= | r(s,a)| \\
    \lVert (\Phi(s,a; \Theta_+)^\top - \gamma \mathbb{E}_{s',a'} \Phi(s',a'; \Theta'_{+}) ^\top)\rVert \lVert \Theta_{-1}  \rVert \cos{\varphi} &= |r(s,a)| \\ 
    \lVert (\Phi(s,a; \Theta_+)^\top - \gamma \mathbb{E}_{s',a'} \Phi(s',a'; \Theta'_{+}) ^\top)\rVert \lVert \Theta_{-1} \rVert &\geq | r(s,a)| \\ 
    \lVert (\Phi(s,a; \Theta_+)^\top - \gamma \mathbb{E}_{s',a'} \Phi(s',a'; \Theta'_{+}) ^\top)\rVert  &\geq  \frac{|r(s,a)|}{\lVert \Theta_{-1} \rVert}. \\ 
\end{split}
\label{eqn: repre-norm}
	\end{equation}
 
\begin{figure}[!hbp]
	\hspace{-0.6in}
	\begin{center}
		\includegraphics[width=0.4\textwidth]{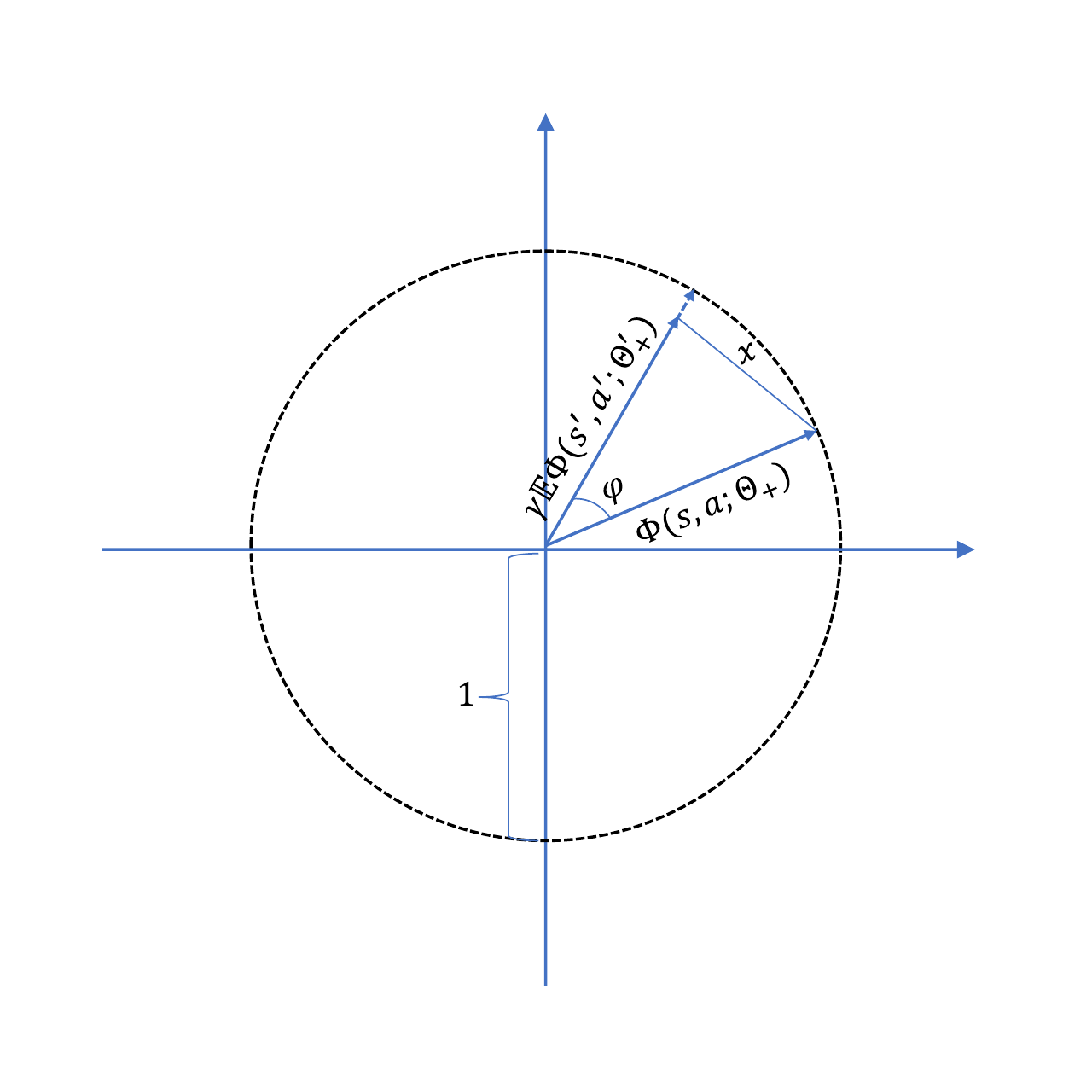}
	\end{center} 
	\caption{\label{fig: repr}Normalized representation vectors.} 
\end{figure}
 Since the representation vectors are normalized, they should co-exist on some tangent plane as visualized in \cref{fig: repr}. Let $x$ be $ \| \Phi(s,a; \Theta_+) - \gamma \mathbb{E}  \Phi(s',a'; \Theta'_{+}\|$, then we have $x \geq \frac{|r(s,a)|}{\lVert \Theta_{-1} \rVert}$, and 
 \begin{equation}
     \begin{aligned}
         \cos \varphi = \frac{ \| \Phi(s,a; \Theta_+) \|^2 + \| \gamma \mathbb{E}  \Phi(s',a'; \Theta'_{+})\|^2 - x^2}{2 \| \Phi(s,a; \Theta_+) \| \| \gamma \mathbb{E}  \Phi(s',a'; \Theta'_{+} )\|}.
     \end{aligned}
 \end{equation}
Now we have 
\begin{equation}
    \begin{aligned}
        \langle \Phi(s,a; \Theta_+), \gamma \mathbb{E}_{s',a'} \Phi(s',a'; \Theta'_{+}) \rangle &= \| \Phi(s,a; \Theta_+) \| \| \gamma \mathbb{E}_{s',a'} \Phi(s',a'; \Theta'_{+}) \| \cos \varphi \\
        &= 1 \cdot \| \gamma \mathbb{E}_{s',a'} \Phi(s',a'; \Theta'_{+}) \| \cdot \frac{ \| \Phi(s,a; \Theta_+) \|^2 + \| \gamma \mathbb{E}  \Phi(s',a'; \Theta'_{+})\|^2 - x^2}{2 \| \Phi(s,a; \Theta_+) \| \| \gamma \mathbb{E}  \Phi(s',a'; \Theta'_{+} )\|} \\
        &= \frac{1 + \| \gamma \mathbb{E}  \Phi(s',a'; \Theta'_{+} )\|^2 - x^2 }{2} \\
        &\leq \frac{1 + \gamma ^2}{2} - \frac{r(s,a)^2}{2\lVert \Theta_{-1} \rVert^2} \\
        & \leq 1 - \frac{r(s,a)^2}{2\lVert \Theta_{-1} \rVert^2}.
    \end{aligned}
\end{equation}

Thus, we have

\begin{equation}
    \begin{aligned}
        \langle \Phi(s,a; \Theta_+),  \mathbb{E}_{s',a'} \Phi(s',a'; \Theta'_{+}) \rangle & \leq \frac{1}{\gamma} - \frac{r(s,a)^2}{2 \gamma \lVert \Theta_{-1} \rVert^2} \\
        & \leq \frac{1}{\gamma} - \frac{r(s,a)^2}{2 \lVert \Theta_{-1} \rVert^2}.
    \end{aligned}
\end{equation}
\end{proof}

In the following, for notational simplicity, we use $X_i$ to denote $S_i, A_i$ for all $i\in [n]$. For any $f \in \mathcal{F}$, $\lVert f \rVert_n^2 = 1/n \cdot \sum_{i=1}^{n}[f(X_i)]^2$. Since both $\hat{O}$ and $TQ$ are bounded by $V_{\max} = R_{\max}/(1 - \gamma)$, we only need to consider the case where $\log{N_{\delta}} \leq n$. 

Let $f_1, \cdots, f_{N_{\delta}}$ be the centers of minimal $\delta$-cover the of $\mathcal{F}$. By the definition of $\delta$-cover, there exists $k^* \in [N_{\delta}]$ such that $\lVert \hat{O} - f_{k^*} \rVert_{\infty} \leq \delta$. Notice that $k^*$ is a random variable since $\hat{O}$ is obtained from data. 
\begin{theorem}[One-step Approximation Error of PEER Update]
Suppose \cref{ass: theta-prod} hold, let $\mathcal{F}\subseteq \mathcal{B}(\mathcal{S}\times\mathcal{A})$ be a class of measurable function on $\mathcal{S}\times \mathcal{A}$ that are bounded by $V_{\max} = R_{\max}/(1 - \gamma)$, and let $\sigma$ be a probability distribution on $\mathcal{S} \times \mathcal{A}$. Also, let $\{(S_i, A_i)\}_{i\in[n]}$ be $n$ i.i.d. random variables in following $\sigma$. Based on $\{(X_i, A_i, Y_i)\}_{i\in[n]}$, we define $\hat{O}$ as the solution to the lease-square with regularization problem,
\begin{equation}
    \min_{f\in \mathcal{F}} \frac{1}{n}\sum_{i=1}^n[f(S_i, A_i) - Y_i]^2 + \beta\Phi(s,a;\Theta)\mathbb{E}\Phi_{s',a'}(s',a';\Theta').
\end{equation}
At the same time, for any $\delta > 0$, let $\mathcal{N}(\delta, \mathcal{F}, \lVert\cdot\rVert_{\infty})$ be the 
\begin{equation}
\label{app conv-result}
    \lVert\hat{O} - TQ\rVert_\sigma^2 \leq (1+\epsilon)^2\cdot \omega(\mathcal{F}) + C\cdot V_{\max}^2/(n\cdot \epsilon) + C'\cdot V_{\max}\cdot \delta + 2\beta\cdot G^2,
\end{equation}
where $C$ and $C'$ are two absolute constants and is defined as \begin{equation}
    \omega(\mathcal{F}) = \sup_{g\in\mathcal{F}}\inf_{f\in\mathcal{F}} \lVert f - Tg \rVert_{\sigma}.
\end{equation}
\label{app thm:convergence}
\end{theorem}

\begin{proof}
\textbf{Step (i):}
We relate $\mathbb{E}[\lVert \hat{O} - TQ \rVert_n^2]$ with its empirical counterpart $\lVert \hat{O} - TQ \rVert_n^2$. Since $Y_i = R_i + \gamma\max_{a\in\mathcal{A}}Q(S_{i+1}, a)$ for each $i \in [n]$. By the definition of $\hat{O}$, for any $f \in \mathcal{F}$ we have
\begin{equation}
\label{C.31}
    \sum_{i=1}^n [Y_i - \hat{O}(X_i)]^2 + \beta\Phi^{\top}(X_{i};\Theta_{\hat{O}})\mathbb{E}\Phi_{X_{i+1}}(X_{i+1};\Theta_{\hat{O}}') \leq \sum_{i=1}^n [Y_i - f(X_i)]^2 + \beta\Phi^{\top}(X_{i};\Theta_f)\mathbb{E}\Phi_{X_{i+1}}(X_{i+1};\Theta_f').
\end{equation}
For each $i \in [n]$, we define $\xi_i = Y_i - (TQ)(X_i)$. Then \cref{C.31} can be rewritten as
\begin{equation}
\begin{split}
       \lVert \hat{O} - TQ \rVert_n^2  \leq \lVert f - TQ \rVert_n^2 + \frac{1}{n}\sum_{i=1}^n \left[2\xi_i[\hat{O}(X_i) - f(X_i)] + \beta\left(\Phi^{\top}(X_i;\Theta_f)\mathbb{E}\Phi^\top(X_{i+1};\Theta_f') - \Phi^\top(X_i;\Theta_{\hat{O}})\mathbb{E}\Phi(X_{i+1};\Theta_{\hat{O}}')\right)\right].
\end{split}
    \label{C.32}
 \end{equation}
 
 We start by bounding the value of $\left(\Phi^\top(X_i;\Theta_f)\mathbb{E}\Phi(X_{i+1};\Theta_f') - \Phi^\top(X_i;\Theta_{\hat{O}})\mathbb{E}\Phi(X_{i+1};\Theta_{\hat{O}}')\right)$. 
 First, by Cauchy-Schwartz Inequality, we have \begin{equation}
    \left\vert\Phi(X_i;\Theta_f)\mathbb{E}\Phi(X_{i+1};\Theta_f')\right\vert \leq \sqrt{\lVert \Phi(X_i;\Theta_{f,+})\rVert^2} \cdot \sqrt{\lVert \mathbb{E}\Phi(X_{i+1};\Theta_{f,+}')\rVert^2} \leq G^2,
 \end{equation}
 where we used \cref{ass: theta-prod} for the second inequality.
  Thus, by triangle inequality, we have 
 \begin{equation}
     \left\vert\Phi^\top(X_i;\Theta_f)\mathbb{E}\Phi(X_{i+1};\Theta_f') - \Phi(X_i;\Theta_{\hat{O}})\mathbb{E}\Phi(X_{i+1};\Theta_{\hat{O}}'\right\vert \leq 2G^2.
 \end{equation}

 And \cref{C.32} reduces to \begin{equation}
    \lVert \hat{O} - TQ \rVert_n^2  \leq \lVert f - TQ \rVert_n^2 + \frac{2}{n}\sum_{i=1}^n \left[\xi_i[\hat{O}(X_i) - f(X_i)] + \beta G^2\right].
    \label{new-C.32}
 \end{equation}
Then we bound the rest on the right side of \cref{C.32}. Since both $f$ and $Q$ are deterministic, we have $\mathbb{E}(\lVert f - TQ\rVert^2_n) = \lVert f - TQ \rVert^2_{\sigma}$. Moreover, since $\mathbb{E}(\xi_i|X_i) = 0$ by definition, we have $\mathbb{E}[\xi_i \cdot g(X_i)] = 0$ for any bounded and measurable function $g$. Thus it holds that 
\begin{equation}
\label{eqn: subst}
    \mathbb{E}\left\{\sum_{i=1}^n \xi_i\cdot[\hat{O}(X_i) - f(X_i)]\right\} = \mathbb{E}\left\{\sum_{i=1}^n \xi_i\cdot[\hat{O} - (TQ)(X_i)]\right\}.
\end{equation}
In addition, by triangle inequality and \cref{eqn: subst} we have 
\begin{equation}
   \left\vert\mathbb{E}\left\{\sum_{i=1}^n\xi_i\cdot[\hat{O}(X_i) - (TQ)(X_i)]\right\}\right\vert 
   \leq \left\vert\mathbb{E}\left\{\sum_{i=1}^n \xi_i \cdot[\hat{O} - f_{k^*}(X_i)]\right\}\right\vert + \left\vert\mathbb{E}\left\{\sum_{i=1}^n \xi_i \cdot [f_{k^*}(X_i) - (TQ)(X_i)]\right\}\right\vert
   \label{two-terms},
\end{equation}
where $f_{k*}$ satisfies $\lVert f_{k*} \rVert \leq \delta$. In the following, we upper bound the two terms on the right side of \cref{two-terms} respectively. For the first term, by applying the Cauchy-Schwarz inequality twice, we have
\begin{equation}
\begin{split}
    \left\vert \mathbb{E}\left\{\sum_{i=1}^n \xi_i\cdot[\hat{O} - f_{k^*}(X_i)]\right\} \right\vert &\leq \sqrt{n}\cdot\left\vert \mathbb{E} \left[\left(\sum_{i=1}^n \xi_i^2\right)^{1/2}\cdot\lVert \hat{O} - f_{k^*}\rVert_n\right]\right\vert \\
    &\leq \sqrt{n}\cdot[\mathbb{E}(\sum_{i=1}^n \xi_i^2)]^{1/2}\cdot[\mathbb{E}(\lVert \hat{O}- f_{k^*} \rVert_n^2)]^{1/2} \leq n\delta \cdot [\mathbb{E}(\xi_i^2)]^{1/2}.
\end{split}
    \label{C.35}
\end{equation}
where we use the fact that $\{\xi_i\}_{i\in[n]}$ have the same marginal distributions and $\lVert \hat{O} - f_{k*} \rVert_n \leq \delta$. Since both $Y_i$ and $TQ$ are bounded by $V_{\max}$, $\xi_i$ is a bounded random variable by its definition. Thus, there exists a constant $C_{\xi} > 0$ depending on $\xi$ such that $\mathbb{E}(\xi_i^2) \leq C_{\xi}^2\cdot V_{\max}^2$. Then \cref{C.35} implies
\begin{equation}
    \left\vert \mathbb{E}\left\{\sum_{i=1}^n\xi_{i}\cdot [\hat{O}(X_i) - f_{k^*}(X_i)]\right\} \right\vert \leq C_{\xi}\cdot V_{\max} \cdot n\delta.
    \label{C.36}
\end{equation}
It remains to upper bound the second term on the right side of \cref{two-terms}. We define $N_\delta$ self-normalized random variables 
\begin{equation}
    Z_j = \frac{1}{\sqrt{n}}\sum_{i=1}^n\xi_i \cdot [f_j(X_i) - (TQ)(X_i)] \cdot \lVert f_j - (TQ)\rVert_n^{-1}
    \label{C.37}
\end{equation}
for all $j \in [N_{\delta}]$. Here recall that $\left\{f_j\right\}_{j\in[N_{\delta}]}$ are the centers of the minimal $\delta$-covering of $\mathcal{F}$. Then we have
\begin{equation}
\begin{split}
    \left\vert \mathbb{E}\left\{\sum_{i=1}^n \xi_i\cdot[ f_{k*}(X_i) - (TQ)(X_i)]  \right\}\right\vert &= \sqrt{n}\cdot\mathbb{E}[{\lVert f_{k^*} - TQ\rVert_n} \cdot \vert Z_{k^*}\vert] \\
    &\leq \sqrt{n}\cdot\mathbb{E}\left\{[\lVert \hat{O} - TQ \rVert_n + \lVert \hat{O} - f_{k^{*}}\rVert_n] \cdot\vert Z_{k^*}\vert\right\} \leq \sqrt{n}\cdot\left\{[\lVert \hat{O} - TQ\rVert_n+ \delta]\cdot \vert Z_{k^*} \vert\right\},
\end{split}
\label{C.38}
\end{equation}
where the first inequality follows from triangle inequality and the second follows from the fact that $\leq \delta$ \cref{C.38}, we obtain 
\begin{equation}
\begin{split}
    \mathbb{E}\left\{[\lVert \hat{O} - TQ \rVert_n+ \delta] \cdot \vert Z_{k^*}\vert\right\} &\leq \left(\mathbb{E}\left\{[\lVert \hat{O} - TQ \rVert_n + \delta]^2\right\}\right)^{1/2}\cdot[\mathbb{E}(Z_{k^*}^2)]^{1/2} \\
&\leq \left(\left\{\mathbb{E}[\lVert \hat{O} - TQ \rVert_n^2]\right\}^{1/2} + \delta\right)\cdot[\mathbb{E}(\max_{j\in[n]} Z_j^2)]^{1/2},
\end{split}    \label{C.39}
\end{equation}
where the last inequality follows from 
\begin{equation}
    \mathbb{E}[{\lVert \hat{O} - TQ\rVert_n}] \leq \left\{\mathbb{E}[\lVert \hat{O} - TQ\rVert_n^2]\right\}^{1/2}.
\end{equation}
Moreover, since $\xi_i$ is centered conditioning on $\left\{X_i\right\}$, $\xi_i$ is a sub-Gaussian random variable. Specifically, there exists an absolute constant $H_{\xi} > 0$ such that $\lVert \xi_i \rVert_{\psi_2} \leq H_{\xi}\cdot V_{\max}$ for each $i \in [n]$. Here the $\psi_2$-norm of a random variable $W$ is defined as \begin{equation}
    \lVert W\rVert_{\psi_2} = \sup_{p\geq 1}p^{-1/2}[\mathbb{E}({\vert W\vert^p})]^{1/p}.
\end{equation}
By the definition of $Z_j$ in \cref{C.37}, conditioning on $\left\{X_i\right\}_{i\in[n]}$, $\xi_i\cdot[f_j(X_i) - (TQ)(X_i)]$ is a centered and sub-Guassian random variable with \begin{equation}
   \lVert \xi_i \cdot [f_j(X_i) - TQ(X_i)] \rVert_{\psi_2} \leq H_{\xi} \cdot V_{\max} \cdot \vert f_j(X_i) - (TQ)(X_i)\vert.
\end{equation}
Moreover, since $Z_j$ is a summation of independent sub-Gaussian random variables, by Lemma 5.9 of \cite{vershynin2010introduction}, the $\psi_2$-norm of $Z_j$ satisfies 
\begin{equation}
    \lVert Z_j\rVert_{\psi_2} \leq C\cdot H_{\xi} \cdot V_{\max} \cdot \lVert f_j - TQ \rVert_n^{-1} \cdot \left[\frac{1}{n}\sum_{i=1}^n \vert [f_j(X_i) - (TQ)(X_i)] \vert^2 \right]^{1/2} \leq C\cdot H_{\xi} \cdot V_{\max},
\end{equation} where $C > 0$ is an absolute constant. Furthermore, by Lemma 5.14 and 5.15 of \cite{vershynin2010introduction}, $Z_j^2$ is a sub-exponential random variable, and its moment-generating function is bounded by 
\begin{equation}
\label{mgf}
    \mathbb{E}\left[\exp(t \cdot Z_j^2)\right] \leq \exp(C \cdot t^2 \cdot H_{\xi}^4 \cdot V_{\max}^4)
\end{equation}
for any $t$ satisfying $C' \cdot |t| \cdot H_{\xi}^2 \cdot V_{\max}^2 \leq 1$, where $C$ and $C'$ are two positive absolute constants. Moreover, by Jensen's Inequality, we bound the moment-generating function of $\max_{j\in[N_\delta]}Z_j^2$ by 
\begin{equation}
\label{C.41}
    \mathbb{E}\left[\exp(t \cdot \max_{j \in [N_\delta]} Z_j^2)\right] \leq \sum_{j \in [N_\delta]}\mathbb{E}[\exp(t\cdot Z_j^2)].
\end{equation}
Combining \cref{mgf} and \cref{C.41}, we have
\begin{equation}
\label{C.42}
    \mathbb{E}(\max_{j \in [N]}Z_j^2) \leq C^2 \cdot H_{\xi}^2 \cdot V_{\max}^2 \cdot\log{N_\delta},
\end{equation}
where $C>0$ is an absolute constant. Hence, plugging \cref{C.42} into \cref{C.38} and \cref{C.39}, we upper bound the second term of \cref{eqn: subst} by \begin{equation}
    \begin{split}
        \left\vert \mathbb{E}\left\{\sum_{i=1}^n \xi_i \cdot [f_{k^*}(X_i) - (TQ)(X_i)]\right\} \right\vert 
        &\leq \left(\left\{\mathbb{E}\lVert \hat{O} - TQ \rVert_n^2\right\}^{1/2} + \delta\right)\cdot C \cdot H_{\xi} \cdot V_{\max} \cdot \sqrt{n\cdot \log{N_{\delta}}}.
    \end{split}
    \label{C.43}
\end{equation}
Finally, combining \cref{new-C.32}, \cref{C.36} and \cref{C.43}, we obtain the following inequality
\begin{equation}
    \begin{split}
       \mathbb{E}[\lVert \hat{O} - TQ \rVert_n^2] &\leq \inf_{f\in \mathcal{F}}\mathbb{E}[\lVert f - TQ \rVert_n^2] + C_{\xi}\cdot V_{\max} \cdot \delta \\
       &+ \left(\left\{\mathbb{E}\lVert \hat{O} - (TQ)\rVert \right\}^{1/2}+ \delta\right)\cdot C\cdot H_{\xi} \cdot V_{\max} + \sqrt{\log{N_{\delta}/n}} + 2\cdot \beta\cdot G^2 \\
       &\leq C\cdot V_{\max}\sqrt{\log{N_\delta/n}} + \inf_{f \in \mathcal{F}}\mathbb{E}[\lVert f - TQ \rVert_n^2] + C'\cdot V_{\max}\delta + 2\cdot\beta\cdot G^2,
    \end{split}
    \label{C.44}
\end{equation}
where $C$ and $C'$ are two constants. Here in the first inequality we take the infimum over $\mathcal{F}$ because \cref{C.31} holds for any $f \in \mathcal{F}$, and the second inequality holds because $\log{N_\delta} \leq n$.

Now we invoke a fact to obtain the final bound for $\mathbb{E}[\lVert\hat{O} - TQ\rVert_n^2]$ from \cref{C.44}. Let $a, b$ and $c$ be positive numbers satisfying $a^2 \leq 2ab + c$. For any $\epsilon \in (0, 1]$, since $2ab \leq \frac{\epsilon}{1 + \epsilon}a^2 + \frac{1 + \epsilon}{\epsilon}b^2$, we have
\begin{equation}
    a^2 \leq (1 + \epsilon)^2\cdot b^2/\epsilon + (1 + \epsilon)\cdot c.
\label{C.45}
\end{equation}
Therefore, applying \cref{C.45} to \cref{C.44} with $a^2 = \mathbb{E}[\lVert \hat{O} - TQ \rVert_n^2]$, $b = C \cdot V_{\max} \cdot \sqrt{\log{N_{}}}$ and $c = \inf_{f\in\mathcal{F}}\mathbb{E}[\lVert f - TQ \rVert_n^2] + C'\cdot V_{\max} \cdot \delta$, we obtain
\begin{equation}
    \mathbb{E}[\lVert \hat{O}- TQ \rVert_n^2] \leq (1 + \epsilon) \cdot \inf_{f\in\mathcal{F}}\mathbb{E}[\lVert f - TQ \rVert_n^2] + C\cdot V_{\max}^2\cdot \log{N_{\delta}}/(n\epsilon) + C'\cdot V_{\max}\cdot \delta + 2\beta G^2,
    \label{C.46}
\end{equation}
where $C$ and $C'$ are two positive absolute constants. This concludes the first step.

\textbf{Step (ii):} In this step, we relate the population risk $\lVert \hat{O} - TQ\rVert^2_\delta$ with $\mathbb{E}[\lVert \hat{O} - TQ\rVert^2_n]$, which is bounded in the first step. To begin with, we generate $n$ i.i.d. random variables $\left\{\tilde{X}_i = (\tilde{S}_i, \tilde{A}_i)\right\}_{i\in[n]}$ following $\sigma$, independent of $\{(S_i, A_i, R_i, S'_i)\}_{i\in[n]}$. Since $\lVert\hat{O} - f_{k^{*}} \rVert_{\infty} \leq \delta$, for any $x \in \mathcal{S} \times \mathcal{A}$, we have
\begin{equation}
\begin{split}
    \left\vert [\hat{O}(x) - (TQ)(x)]^2 - [f_{k^*}(x) - (TQ)(x)]^2 \right\vert = \left\vert \hat{O}(x) - f_{k^*}(x)\right\vert \cdot \left\vert \hat{O}(x) + f_{k^*}(x) - 2(TQ)(x)\right\vert \leq 4V_{\max}\cdot\delta,
\end{split}
\label{C.47}
\end{equation}
where the last inequality follows from the fact that $\lVert TQ \rVert_{\infty} \leq V_{\max}$ and $\lVert f\rVert_{\infty} \leq V_{\max}$ for any $f\in \mathcal{F}$. 

Then by the definition of $\lVert \hat{O} - TQ\rVert^2_{\delta}$ and \cref{C.47}, we have 
\begin{equation}
    \begin{split}
    \lVert \hat{O} - TQ \rVert_{\sigma}^2 &= \mathbb{E}\left\{\frac{1}{n}\sum_{i=1}^n [\hat{O}(\tilde{X_i}) - (TQ)(\tilde{X_i})]^2\right\}\\
    &\leq \mathbb{E}\left\{\lVert\hat{O} - TQ\rVert^2_n+ \frac{1}{n}\sum_{i=1}^n [f_{k^*}(\tilde{X_i}) -  (TQ)(\tilde{X_i})]^2 - \frac{1}{n}\sum_{i=1}^n [f_{k^*}(X_i) - (TQ)(\tilde{X_i})]^2\right\} + 8V_{\max}\cdot \delta\\
    &= \mathbb{E}(\lVert \hat{O} - TQ\rVert^2_n) + \mathbb{E}[\frac{1}{n}\sum_{}^n h_{k^*}(X_i, \tilde{X_i})] + 8V_{\max}\cdot \delta,
\end{split}
\label{C.48}
\end{equation}
where we apply \cref{C.47} to obtain the first inequality, and in the last equality we define 
\begin{equation}
    h_j(x, y) = [f_j(y) - (TQ)(y)]^2 - [f_j(x) - (TQ)(x)]^2,
    \label{C.49}
\end{equation} for any $x, y \in \mathcal{S} \times \mathcal{A}$ and any $j\in [N_{\delta}]$. Note that $h_{k^*}$ is a random function since $k^*$ is random. By the definition of $h_j$, we have $\vert h_j(x, y) \vert \leq 4V_{\max}^2$ for any $(x, y) \in \mathcal{S} \times \mathcal{A}$ and $\mathbb{E}[h_j(X_i, \tilde{X}_i)] = 0$ for any $i \in [n]$.
Moreover, the variance of $h_j(X_i, \tilde{X_i})$ satisfies
\begin{equation}
    \begin{split}
        \Var[h_j(X_i, \tilde{X_i})] &= 2\Var{\left\{[f_j(X_i) - (TQ)(X_i)]^2\right\}}\\
        &\leq 2\mathbb{E}\left\{[f_j(X_i) - (TQ)(X_i)]^4\right\} \leq 8\Upsilon^2\cdot V_{\max}^2,
    \end{split}
    \label{C.50}
\end{equation}
where we define $\Upsilon$ by letting
\begin{equation}
    \Upsilon = \max(4V_{\max}^2 \cdot \log{N_{\delta}/n}, \max_{j\in [N_\delta]} \mathbb{E}\left\{[f_j(X_i) - (TQ)(X_i)]^2\right\}).
\end{equation}

Furthermore, we define \begin{equation}
   T = \sup_{j\in[N_\delta]}\left\vert \sum_{i=1}^n h(X_i, \tilde{X_i})/\Upsilon \right\vert.
    \label{C.51}
\end{equation}
Combining \cref{C.48} and \cref{C.51}, 
\begin{equation}
   \lVert \hat{O} - TQ\rVert_{\sigma}^2 \leq \mathbb{E}[\lVert \hat{O} - TQ \rVert_n^2] + \Upsilon/n \cdot \mathbb{E}[T] + 8V_{\max} \cdot \delta.
   \label{C.52}
\end{equation}
In the following, we use Bernstein's Inequality to establish an upper bound for $\mathbb{E}(T)$:
\begin{lemma}(Bernstein's Inequality)
\label{lem:Bernstein}
Let $U_1, \cdots, U_n$ be n independent random variables satisfying $\mathbb{E}(U_i) = 0$ and $ \leq$ for all $i \in [n]$. Then for any $t > 0$,  we have
\begin{equation}
    \mathbb{P}\left(\left\vert\sum_{i=1}^n U_i \right\vert \geq t\right) \leq 2\exp(\frac{-t^2}{2M\cdot t/3 + 2 \sigma^2}), 
\end{equation}
where $\sigma^2 = \sum_{i=1}^n $ is the variance of $\sum_{i=1}^n U_i$.
\end{lemma}
We first apply Bernstein's Inequality by setting $U_i = h_j(X_i,\tilde{X_i})/\Upsilon$ for each $i \in [n]$. Then we take a union bound for all $j \in [N_{\delta}]$ to obtain
\begin{equation}
    \mathbb{P}(T\geq t) = \mathbb{P}\left[\sup_{j\in [N_\delta]}\frac{1}{n}\left\vert\sum_{i=1}^n h_j(X_i, \tilde{X_i}
    )/\Upsilon\right\vert\geq t\right] \leq 2N_{\delta}\cdot \exp\left\{\frac{-t^2}{8V_{\max}^2\cdot [t/(3\Upsilon) + n]}\right\} 
    \label{C.53}.
\end{equation}
Since $T$ is nonnegative, $\mathbb{E}(T) = \int_{0}^{\infty}\mathbb{P}(T\geq t)dt$. Thus, for any $u \in (0, 3\Upsilon\cdot n)$,
\begin{equation}
\begin{split}
    \mathbb{E}(T) &\leq u + \int_{u}^{\infty}\mathbb{P}(T\geq t)\,dt \leq u + 2 N_{\delta}\int_{u}^{3\Upsilon\cdot n}\exp\left(\frac{-t^2}{16V_{\max}^2\cdot n}\right)\,dt + 2 N_{\delta}\int_{3\Upsilon\cdot n}^{\infty}\exp{\left(\frac{-3\Upsilon \cdot t}{16V_{\max}^2}\right)}\,dt \\
    &\leq u + 32 N_{\delta}\cdot V_{\max} \cdot n/u \cdot \exp\left(\frac{-u^2}{16V_{\max}^2 \cdot n}\right) + 32 N_{\delta} \cdot V_{\max}^2/(3\Upsilon)\cdot \exp{\left(\frac{-9\Upsilon^2\cdot n}{16V_{\max}^2}\right)},
\end{split}
    \label{C.54}
\end{equation}
where in the second inequality we use the fact that $\int_{s}^{\infty}\exp{(-t^2/2)}dt \leq 1/s\cdot\exp(-s^2/2)$. Now we set $u = 4V_{\max}\sqrt{n\cdot\log{N_{\delta}}}$ in \cref{C.54} and plug in the definition of $\Upsilon$ in \cref{C.50} to obtain 
\begin{equation}
\mathbb{E} \leq 4V_{\max} \log{n\cdot N_{\delta}} + 8V_{\max}\sqrt{n/\log{N_{\delta}}} + 6V_{\max}\sqrt{n/\log{N_{\delta}}} \leq 8V_{\max} \sqrt{n\cdot \log{N_{\delta}}},
    \label{C.55}
\end{equation}
where the last inequality holds when $\log{N_{\delta}}\geq4$. Moreover, the definition of $\Upsilon$ in \cref{C.50} implies that $\Upsilon \leq \max[2V_{\max}\sqrt{\log{N_{\delta}/n}}, \lVert \hat{O} - TQ\rVert^2_\sigma + \delta]$. In the following, we only need to consider the case where $\Upsilon \leq \lVert \hat{O} - TQ \rVert_{\sigma}+ \delta$, since we already have \cref{app conv-result} if $\lVert \hat{O} - TQ\rVert + \delta \leq 2V_{\max}\sqrt{\log{N_{\delta}/n}}$, which concludes the proof.

Then, when $\Upsilon \leq \lvert \hat{O} - TQ\rVert_\sigma + \delta$ holds, combining \cref{C.52} and \cref{C.55} we have,
\begin{equation}
    \begin{split}
      \lVert \hat{O} - TQ\rVert_{\delta}^2  &\leq \mathbb{E}[\lVert \hat{O} - TQ \rVert_n^2] + 8V_{\max}\sqrt{\log{N_{\delta}}/n}\cdot \lVert \hat{O} - TQ \rVert_{\delta} + 8V_{\max}\sqrt{\log{N_{\delta}/n}}\cdot \delta + 8V_{\max}\cdot \delta \\
      &\leq \mathbb{E}[\lVert \hat{O} - TQ \rVert_n^2] + 8 V_{\max}\sqrt{\log{N_{\delta}/n}}\cdot \lVert \hat{O} - TQ \rVert_{\sigma} + 16V_{\max}\cdot \delta.
    \end{split}
    \label{C.56}
\end{equation}
We apply the inequality in \cref{C.45} to \cref{C.56} with $a = \lVert \hat{O} - TQ\rVert_\sigma$, $b = 8V_{\max}\sqrt{\log{N_{\delta}}/n}$, and $c = \mathbb{E}[\lVert \hat{O} - TQ \rVert_n^2] + 16V_{\max}\cdot \delta$ we have. Hence we found
\begin{equation}
\begin{split}
    \lVert \hat{O} - TQ \rVert_\sigma^2 \leq (1 + \epsilon)\cdot \mathbb{E}[\lVert \hat{O} - TQ \rVert_n^2] + (1 + \epsilon)^2\cdot 64V_{\max}\cdot \log{N_{\delta}}/(n \cdot \epsilon) + (1 + \epsilon) \cdot 18 V_{\max} \cdot \delta,
\end{split}
    \label{C.57}
\end{equation}
which concludes the second step of the proof.

Finally, combining steps \textbf{(i)} and together, i.e., \cref{C.46} and \cref{C.57}, we conclude that
\begin{equation}
    \lVert \hat{O} - TQ \rVert_\sigma^2 \leq (1 + \epsilon)^2 \cdot \inf_{f \in \mathcal{F}} \mathbb{E}[\lVert f - TQ\rVert^2_n] + C_1\cdot V_{\max}^2\cdot \log{N_{\delta}/(n\cdot \epsilon)} + C_2\cdot V_{\max}\cdot \delta + 2\beta G^2,
\end{equation}
where $C_1$ and $C_2$ are two absolute constants. Moreover, since $Q\in \mathcal{F}$
\begin{equation}
    \inf_{f \in \mathcal{F}}\mathbb{E}[\lVert f - TQ \rVert_n^2] \leq \sup_{Q\in \mathcal{F}}\left\{\inf_{f \in \mathcal{F}}\mathbb{E}[\lVert f - TQ \rVert_n^2]\right\},
\end{equation}

which concludes the proof of \cref{thm: convergence}.
\end{proof}
 \clearpage
 \clearpage
\section{Appendix: Experimental Settings}\label{appdenix: experimental settings}
In this section, we provide the experimental settings in detail. 

\subsection{Code}

Our project is available at \href{https://sites.google.com/view/peer-cvpr2023/}{https://sites.google.com/view/peer-cvpr2023/}.

\subsection{Experimental Details}\label{sec: implement details}
Our implementation of PEER coupled with CURL/DrQ is based on the CURL/DrQ codebase.

\textbf{Computational resources.} All experiments are conducted on two GPU servers. The first one has 3 Titan XP GPUs and Intel(R) Xeon(R) CPU E5-2640 v4 @ 2.40GHz. The second one has 4 Titan RTX GPUs and an Intel(R) Xeon(R) Gold 6137 CPU @ 3.90GHz. Each run for DMControl takes fifty hours to finish. For PyBullet, MuJoCo, and Atari tasks, it takes 5 hours to finish a run. For PyBullet and MuJoCo suites, we simultaneously launch 70 seeds. For the DMControl and Atari suites, we simultaneously run 15 random seeds.

\textbf{How to plot \cref{fig: alpha visuallization}.} Before computing the distinguishable representation discrepancy (DRD), the representation of $Q$-network is normalized as shown in \cref{app theorem: representation gap}. Then we compute DRD in mini-batch samples. We compute the average DRD of mini-batch samples, and plot \cref{fig: alpha visuallization} over five random seeds.

\textbf{The source of data in \cref{table: exp dm control}.} The evaluation results of State SAC, PlaNet, Dreamer, SAC+AE, and CURL in \cref{table: exp dm control} are taken from the original CURL paper \cite{laskin2020curl}. And the results of DrQ are taken from the original DrQ paper \cite{drq}. As for the data of DrQ-v2, we took from the authors' data (link: https://github.com/facebookresearch/drqv2) and presented the statistics in the same way as the rest of \cref{table: exp dm control}. Note that the author only provides DrQ-v2 results over nine seeds.

\textbf{The source of data in \cref{table: atari exp results}.} The evaluation results of Human, Random, OTRainbow, Eff.Rainbow, and CURL in \cref{table: atari exp results} are taken from the original CURL paper \cite{laskin2020curl}. And the results of Eff. DQN and DrQ are taken from the original DrQ paper \cite{drq}.

\textbf{Data in \cref{fig: normalized score}.} We do not include the DrQ-v2 results in \cref{fig: normalized score} because DrQ is better than DrQ-v2 as shown in \cref{table: exp dm control}.

\textbf{Random seeds.} If not otherwise specified, we evaluate each tested algorithm over 10 random seeds to ensure the reproducibility of our experiments. Also, we set all seeds fixed in our experiments.

\textbf{Grid world.} The grid world is shown in \cref{fig: sub gridworld}. If the agent arrives at $S_T$,  it gets a reward of 10, and other states get a reward of 0. We present the remaining hyper-parameters for the grid world in \cref{app table: parameters for grid world}.

\textbf{PyBullet.} When we train the agent on the Pybullet suite, the agent starts by randomly collecting 25,000 states and actions for better exploration. Then we evaluate the agent for ten episodes every 5,000 timesteps. We take the average return of ten episodes as a key evaluation metric. To ensure a fair evaluation of the algorithms, we do not apply any exploration tricks during the evaluation phase (e.g. injecting noise into actions in TD3), because these exploration tricks may harm the performance of tested algorithms. The complete timesteps are 1 million. The results are reported over ten random seeds. For the hyper-parameter $\beta$ of PEER, we take $5e-4$ for every task.

For all algorithms except METD3, we use the author's implementation \cite{td3} or a commonly used public repository \cite{baselines}. Our implementations of PEER and METD3 are based on TD3 implementation. To fairly evaluate our algorithm, we keep all the original TD3's hyper-parameters without any modification. For the hyper-parameter of METD3, we set the dropout rate equal to $0.1$ as the author \cite{mepg} did. The soft update style is adopted for METD3, PEER with $\eta = 0.005$. We summarize the hyper-parameter settings for the PyBullet suite in \cref{app table: hyper-parameters for bullet}.

\textbf{MuJoCo.} All experiments on MuJoCo are consistent with the PyBullet settings, except for the code of SAC used. We found that the performance of SAC \cite{pytorch_sac} deteriorates on the MuJoCo suite. Therefore, we use the code of Stable-Baselines3 \footnote{Code: https://github.com/DLR-RM/stable-baselines3} \cite{stable-baselines3} for SAC implementation with the same hyper-parameters under PyBullet settings.

\textbf{DMControl.}
We utilize the authors' implementation of CURL and DrQ without any further modification as we discussed. And we do not change the default hyper-parameters for CURL\footnote{Code: https://github.com/MishaLaskin/curl}. For a fair comparison, we keep the hyper-parameters of PEER the same as CURL and DrQ. And the hyper-parameter $\beta=5e-4$ is kept in each environment. 
We summarize the hyper-parameter settings for the DMControl suite in \cref{app table: hyper-parameters dm control} and \cref{app table hyperparameter dmc drq-poer}.

\textbf{Atari.} Our implementation PEER is based on CURL\footnote{Code: https://github.com/aravindsrinivas/curl\_rainbow}. For a fair comparison, we keep the hyper-parameters and settings of CURL the same as CURL. And the hyper-parameter $\beta=5e-4$ is kept in each environment. Check \cref{app table: hyperparameters atari curl-poer} and \cref{app table: atari drqpoer} for more details.

\subsection{Pseudocode for PEER Loss} \label{app sec: Pseudocode}
We provide PyTorch-like pseudocode for the PEER loss as follows.

\begin{lstlisting}[language=Python, caption=Pytorch-like pseudocode for the PEER loss]
def PE_loss_with_PEER(representation, Q, target_representation, target_Q, beta):
    """
    representation: shape = Batch_size * N, representation of critic
    Q: shape = Batch_size * 1, current Q value
    target_representation: shape = Batch_size * N, representation of critic_target
    target_Q: shape = Batch_size * 1, target Q value ( r + \mathcal{E}Q(s',a') )
    beta: a small constant, controlling the regularization effectiveness of PEER
    """
    PEER_loss = torch.einsum('ij,ij->i', [representation, target_representation]).mean()
    PE_loss = torch.nn.functional.mse_loss(Q, target_Q).mean()
    
    loss = PE_loss + beta * PEER_loss
    return loss
\end{lstlisting}

\begin{table*}[htbp]
	\renewcommand\tabcolsep{3.0pt} 
	\centering
	\begin{tabular}{l|c}
		\toprule
		\textbf{Hyper-parameter} & \textbf{Value}  \\
		\midrule
		\textit{Shared hyper-parameters} & \\
        State space & integer: from 0 to 19 \\
        Action space & Discrete(4): up, down, left, right \\
		Discount ($\gamma$) & 0.99 \\
		Replay buffer size & $10^5$ \\
		Optimizer & Adam \cite{adam} \\
		Learning rate for Q-network & $1 \times 10^{-4}$ \\
		Number of hidden layers for all networks & 2 \\
		Number of hidden units per layer & 32 \\
		Activation function & ReLU \\
		Mini-batch size & 64  \\
		Random starting exploration time steps & $10^3$ \\
		Target smoothing coefficient ($\eta$) & 0.005 \\
		Gradient Clipping & False \\
		Exploration Method & Epsilon-Greedy \\
		$\epsilon$ & 0.1 \\
		Evaluation Episode & 10 \\
		Number of Episodes & 2000 \\
		\midrule
		\textit{PEER} & \\
		PEER coefficient ($\beta$) & $5 \times 10^{-4}$ \\
		\bottomrule
	\end{tabular}
 \caption{\label{app table: parameters for grid world}Hyper-parameters settings for Grid World experiments}
	
\end{table*}

\begin{table*}[htbp]
	
	\renewcommand\tabcolsep{3.0pt} 
	\centering
	\begin{tabular}{l|c}
		\toprule
		\textbf{Hyper-parameter} & \textbf{Value}  \\
		\midrule
		\textit{Shared hyper-parameters} & \\
		Discount ($\gamma$) & 0.99 \\
		Replay buffer size & $10^6$ \\
		Optimizer & Adam \cite{adam} \\
		Learning rate for actor & $3 \times 10^{-4}$ \\
		Learning rate for critic & $3 \times 10^{-4}$ \\
		Number of hidden layers for all networks & 2 \\
		Number of hidden units per layer & 256 \\
		Activation function & ReLU \\
		Mini-batch size & 256  \\
		Random starting exploration time steps & $2.5 \times 10^4$ \\
		Target smoothing coefficient ($\eta$) & 0.005 \\
		Gradient Clipping & False \\
		Target update interval ($d$) & 2\\
		\midrule
		\textit{TD3} & \\
		Variance of exploration noise & 0.2 \\
		Variance of target policy smoothing & 0.2 \\
		Noise clip range & $[-0.5, 0.5]$ \\
		Delayed policy update frequency & 2 \\
		\midrule
		\textit{PEER} & \\
		PEER coefficient ($\beta$) & $5 \times 10^{-4}$ \\
		\midrule
		\textit{SAC} & \\
		Target Entropy & - dim of $\mathcal{A}$ \\
		Learning rate for $\alpha$ & $1\times 10^{-4}$ \\
		\bottomrule
	\end{tabular}
	
 \caption{\label{app table: hyper-parameters for bullet}Hyper-parameters settings for PyBullet and MuJoCo experiments}
\end{table*}

\begin{table*}[htbp]	
	\renewcommand\tabcolsep{3.0pt} 
	\centering
	\begin{tabular}{l|c}
		\toprule
		\textbf{Hyper-parameter} & \textbf{Value}  \\
		\midrule
		\text{ PEER} coefficient ($\beta$) & $5 \times 10^{-4}$ \\
		\text { D}iscount $\gamma$ & 0.99 \\
		\text { Replay buffer size } & 100000 \\
		\text { Optimizer } & \text { Adam } \\
		\text { Learning rate } & $1 \times 10^{-4}$ \\
		\text { Learning rate }$\left(f_{\theta}, \pi_{\psi}, Q_{\phi}\right) $& $2 \times 10 ^{-4}$ \text { cheetah, run } \\
& 1 $\times 10^{-3}$ \text { otherwise } \\
		\text { Convolutional layers } & 4 \\
\text { Number of filters } & 32 \\
\text { Activation function } & \text { ReLU } \\
\text { Encoder EMA } $\eta$ & 0.05 \\
\text { Q} function EMA ($\eta$) & 0.01 \\
\text { Mini-batch size } & 512 \\

\text { Target Update interval $(d)$ } & 2 \\
\text { Latent dimension } & 50 \\
\text { Initial temperature } & 0.99 \\
\text { Number of hidden units per layer }(\text { MLP }) & 1024 \\
\text { Evaluation episodes } & 10 \\
\text { Random crop } & \text { True } \\
\text { Observation rendering } & (100,100) \\
\text { Observation downsampling } & (84,84) \\
\text { Initial steps } & 1000 \\
\text { Stacked frames } & 3 \\
\text { Action repeat } & 2 \text { finger, spin; walker, walk } \\
& 8 \text { cartpole, swingup } \\
& 4 \text { otherwise } \\

$\left(\beta_{1}, \beta_{2}\right)  \rightarrow\left(f_{\theta}, \pi_{\psi}, Q_{\phi}\right) $& (.9, .999) \\
$\left(\beta_{1}, \beta_{2}\right) \rightarrow(\alpha) $ & (.9, .999) \\

		\bottomrule
	\end{tabular}
 	\caption{\label{app table: hyper-parameters dm control}Hyper-parameters settings for PEER (coupled with CURL) DMControl experiments.}
	
\end{table*}

\begin{table}[!htbp]

\centering

\begin{tabular}{l|c}
\toprule
\textbf{Hyper-parameter} & \textbf{Value}  \\
\midrule
PEER coefficient ($\beta$) & $5 \times 10^{-4}$ \\
Replay buffer capacity & $100000$ \\
Seed steps & $1000$ \\
Main results minibatch size & $512$ \\
Discount $\gamma$ & $0.99$ \\
Optimizer & Adam \\
Learning rate & $10^{-3}$ \\
Critic target update frequency & $2$ \\
Critic Q-function soft-update rate $\tau$ & $0.01$ \\
Actor update frequency & $2$ \\
Actor log stddev bounds & $[-10, 2]$ \\
Init temperature & $0.1$ \\
\bottomrule
\end{tabular}\\
\caption{\label{app table hyperparameter dmc drq-poer}Hyper-parameters settings for PEER (coupled with DrQ) DMControl experiments.}
\end{table}

\begin{table}[h]
\begin{center}
\begin{small}
\begin{tabular}{l|l}
\toprule
\textbf{Hyper-parameter} & \textbf{Value}  \\
\midrule
PEER coefficient ($\beta$) & $5 \times 10^{-4}$ \\
Random crop    & True  \\ 
Image size    & $(84,84)$  \\ 
Data Augmentation & Random Crop (Train) \\
Replay buffer size    & $100000$ \\ 
Training frames & $400000$ \\ 
Training steps & $100000$ \\
Frame skip & $4$ \\
Stacked frames    & $4$  \\ 
Action repeat    & $4$ \\
Replay period every & $1$ \\
Q network: channels & $32$, $64$ \\
Q network: filter size & $5\times 5, 5\times 5$ \\
Q network: stride & $5$, $5$ \\
Q network: hidden units & $256$ \\
Momentum (EMA for CURL) $\tau$ & $0.001$  \\
Non-linearity & ReLU \\
Reward Clipping   & $[-1, 1]$  \\ 
Multi step return & $20$ \\
Minimum replay size for sampling & $1600$ \\ 
Max frames per episode & $108$K \\
Update & Distributional Double Q \\
Target Network Update Period & every $2000$ updates \\
Support-of-Q-distribution & $51$ bins \\ 
Discount $\gamma$ & $0.99$ \\
Batch Size & $32$  \\
Optimizer    & Adam  \\ 
Optimizer: learning rate & $0.0001$ \\
Optimizer: $\beta1$ & $0.9$ \\
Optimizer: $\beta2$ & $0.999$ \\ 
Optimizer $\epsilon$ & $0.000015$ \\
Max gradient norm & $10$ \\ 
Exploration & Noisy Nets   \\
Noisy nets parameter & $0.1$ \\
Priority exponent & $0.5$ \\
Priority correction & $0.4 \rightarrow 1$ \\
Hardware & GPU \\
\bottomrule
\end{tabular}
\caption{\label{app table: hyperparameters atari curl-poer}Hyper-parameters used for Atari100K PEER (coupled with CURL) experiments.}
\end{small}
\end{center}
\vskip -0.1in
\end{table}

\begin{table}[hb!]
\centering
\begin{tabular}{l|c}
\toprule
\textbf{Hyperparameter} & \textbf{Value}  \\
\midrule
PEER coefficient ($\beta$) & $5 \times 10^{-4}$ \\
Data augmentation & Random shifts and Intensity \\
Grey-scaling & True \\
Observation down-sampling & $84 \times 84$ \\
Frames stacked & $4$ \\
Action repetitions & $4$ \\
Reward clipping & $[-1, 1]$ \\
Terminal on loss of life  & True \\
Max frames per episode & $108$k \\
Update  &  Double Q \\
Dueling & True \\
Target network: update period & $1$ \\
Discount factor &  $0.99$ \\
Minibatch size  & $32$ \\
Optimizer  & Adam \\
Optimizer: learning rate & $0.0001$\\
Optimizer: $\beta_1$  & $0.9$ \\
Optimizer: $\beta_2$ &  $0.999$ \\
Optimizer: $\epsilon$ &  $0.00015$ \\
Max gradient norm & $10$ \\
Training steps &  $100$k \\
Evaluation steps & $125$k \\
Min replay size for sampling& $1600$\\
Memory size & Unbounded \\
Replay period every & $1$ step \\
Multi-step return length & $10$ \\
Q network: channels  & $32, 64, 64$ \\ 
Q network: filter size  & $8 \times 8$, $4 \times 4$, $3 \times 3$ \\
Q network: stride  & $4, 2, 1$ \\
Q network: hidden units  & $512$ \\
Non-linearity & \texttt{ReLU}\\
Exploration & $\epsilon$-greedy \\
$\epsilon$-decay & $5000$ \\
\bottomrule
\end{tabular}\\
\caption{\label{app table: atari drqpoer} Hyper-parameters used for Atari100K PEER (coupled with DrQ algorithm) experiments.}
\end{table}
 \clearpage
 \section{Appendix: Experimental Suites}\label{app sec experimentals suites}
The experimental suites we use are Bullet \cite{bullet3}, MuJoCo\cite{mujoco}, DMcontrol\cite{dm-control}, and Atari\cite{atari}.
We show the environments of bullet, MuJoCo, DMControl, and Atari in \cref{app fig: bullet suites}, \cref{app fig: mujoco suites}, \cref{app fig: dmcontrol suite}, \cref{app fig: atari images}, and \cref{app fig: atari images continuation}, respectively. 

Besides, We list the state and action information for the four suites in \cref{app table: bullet info}, \cref{app table: mujoco info}, \cref{app table: dmc info}, and \cref{app table: atari info}. respectively.

\begin{figure*}[!htbp]

	\begin{center}
	\includegraphics[width=\linewidth]{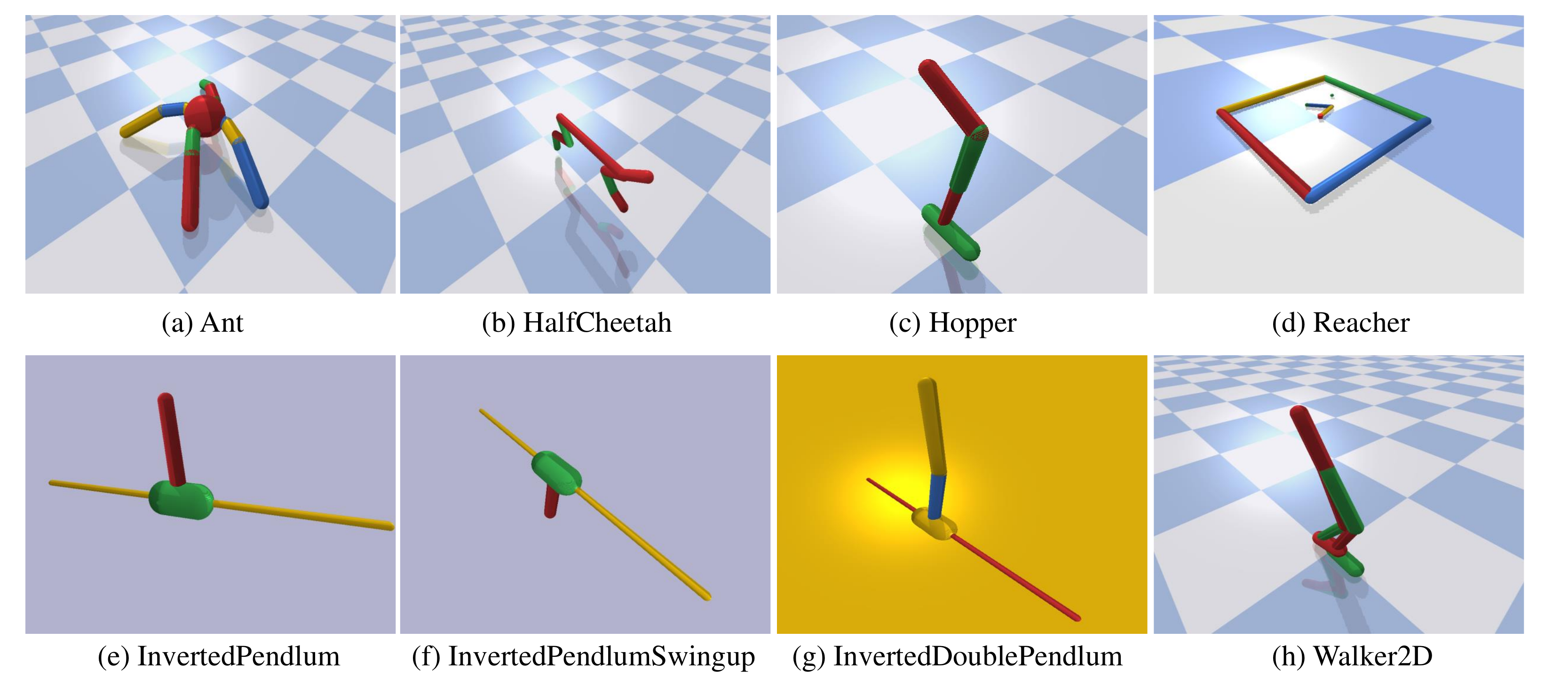}
	\end{center}
	\caption{\label{app fig: bullet suites}Images for PyBullet suite used in our experiments. The states for this suite are vectors.}
\end{figure*}

\begin{figure*}[!htbp]

	\begin{center}
	\includegraphics[width=\linewidth]{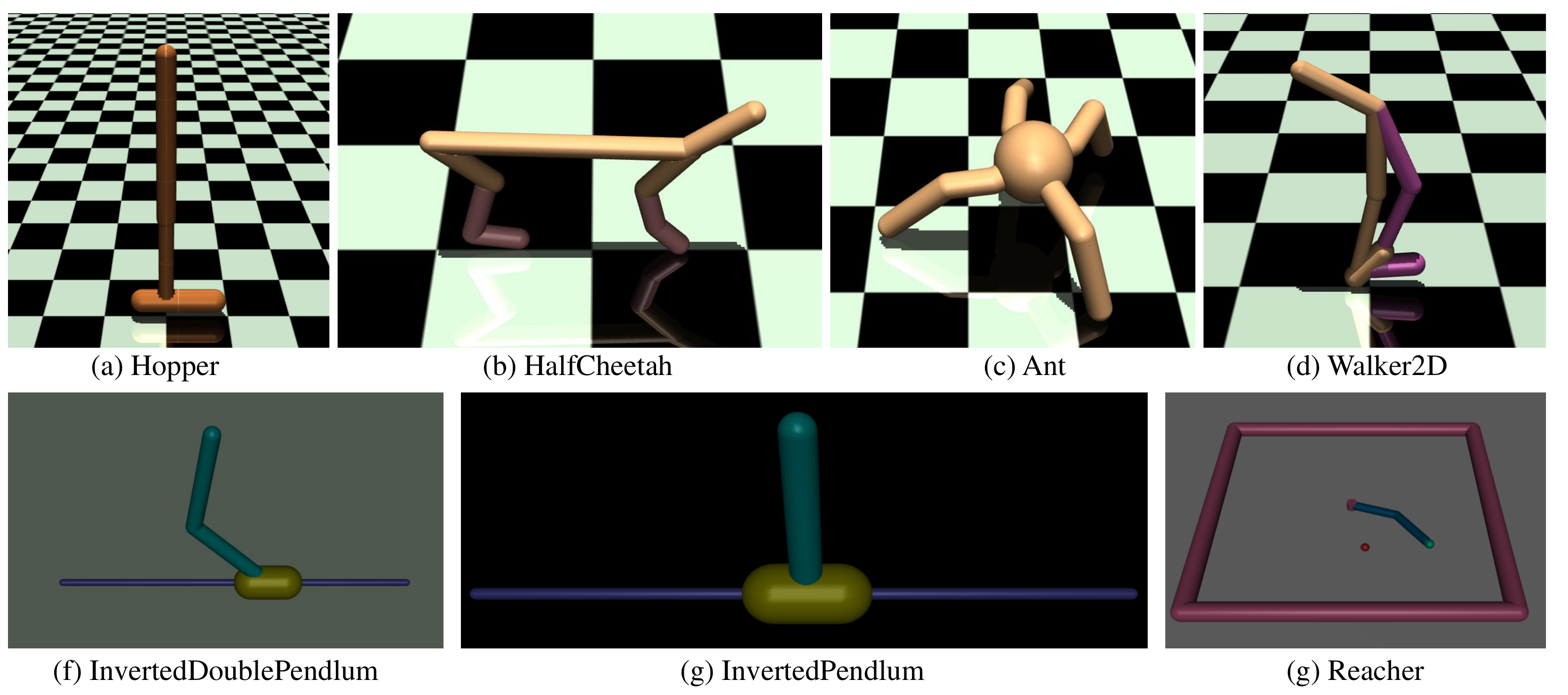}
	\end{center}
	\caption{\label{app fig: mujoco suites}Images for MuJoCo suite used in our experiments. The states for this suite are vectors.}
\end{figure*}

\begin{figure*}
\centering
\begin{subfigure}{0.33\textwidth}
    \includegraphics[width=1\textwidth]{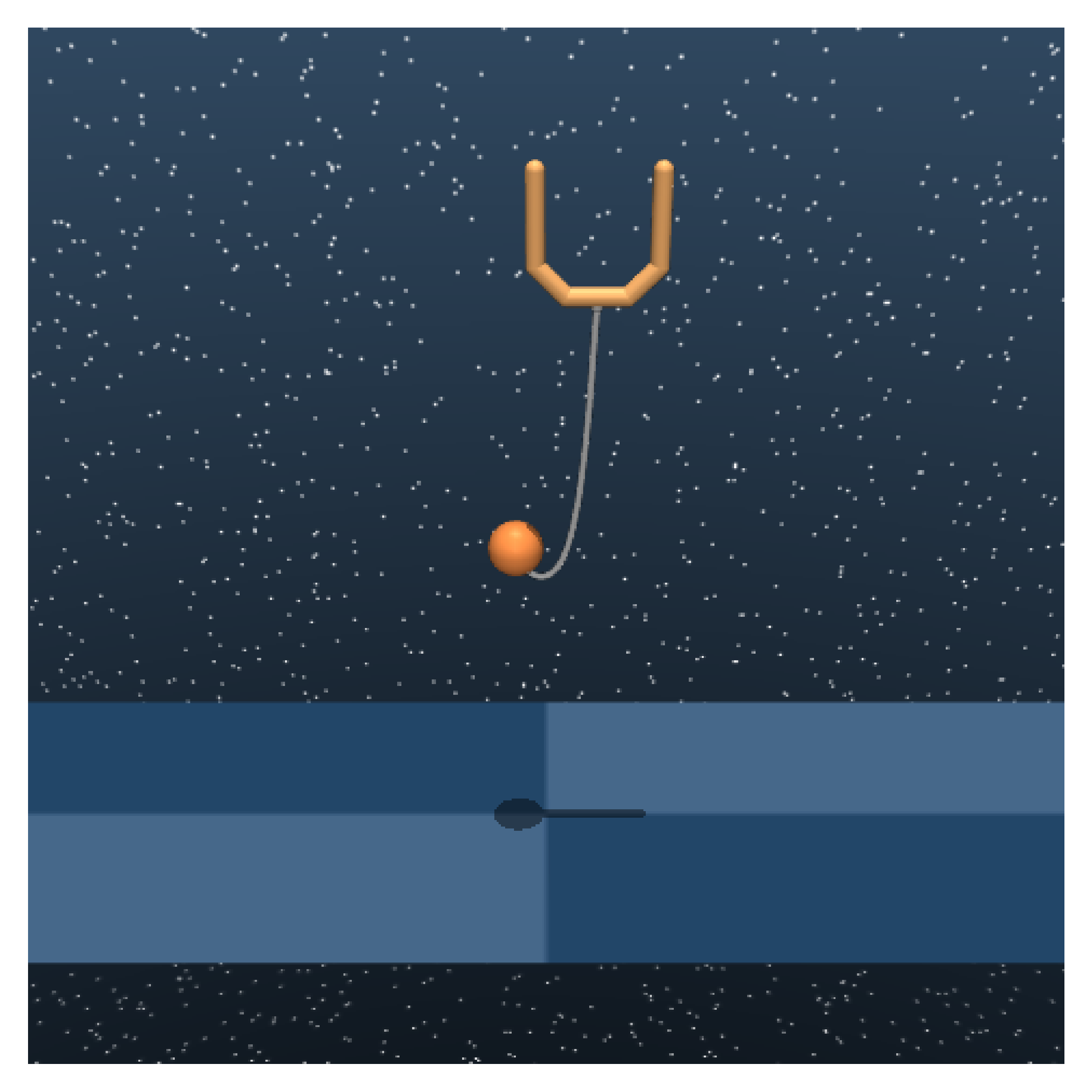}
    \caption{ball\_in\_cup\_catch}
\end{subfigure}
\hspace{-0.1in}
\begin{subfigure}{0.33\textwidth}
    \includegraphics[width=1\textwidth]{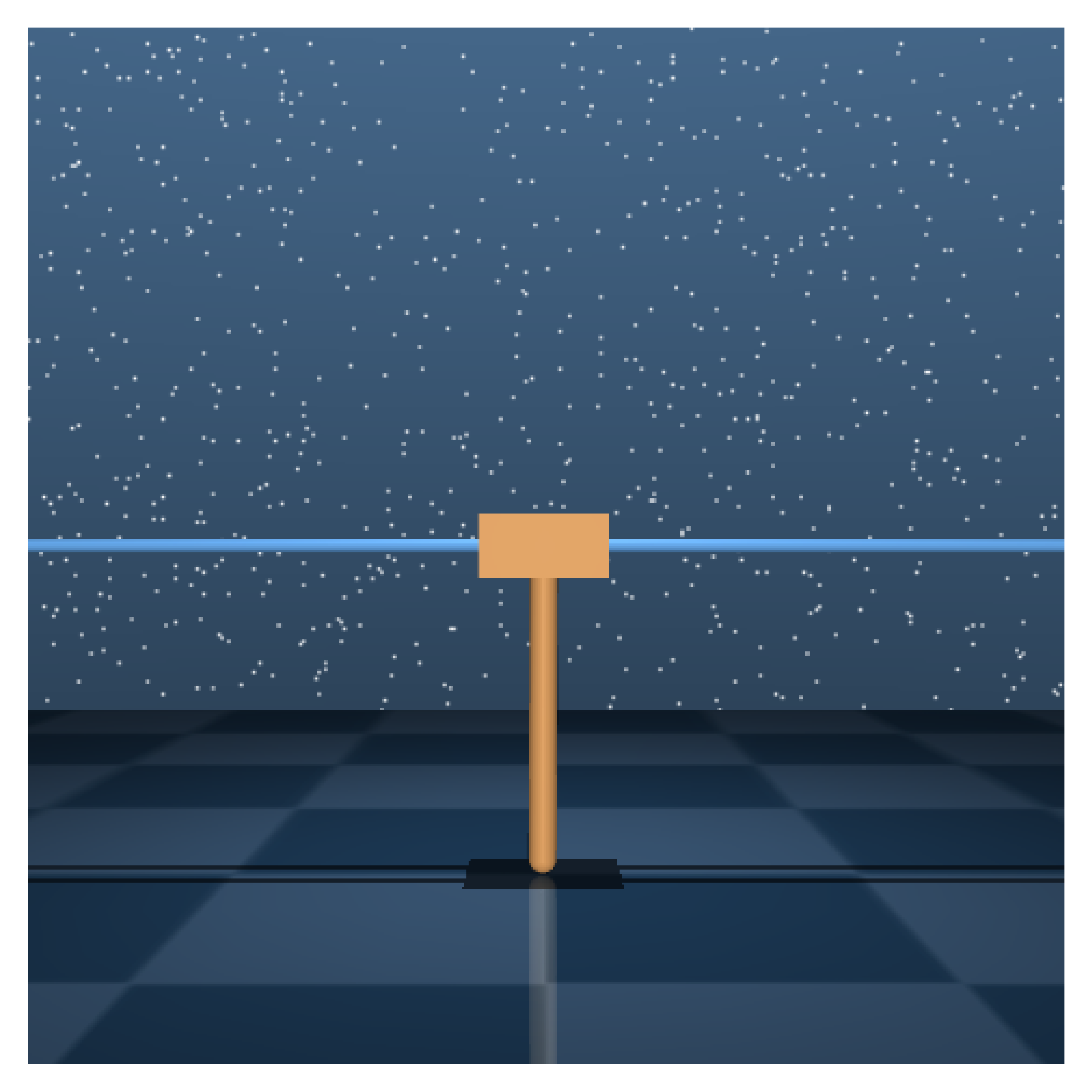}
    \caption{cartpole\_swingup}
\end{subfigure}
\hspace{-0.1in}
\begin{subfigure}{0.33\textwidth}
    \includegraphics[width=1\textwidth]{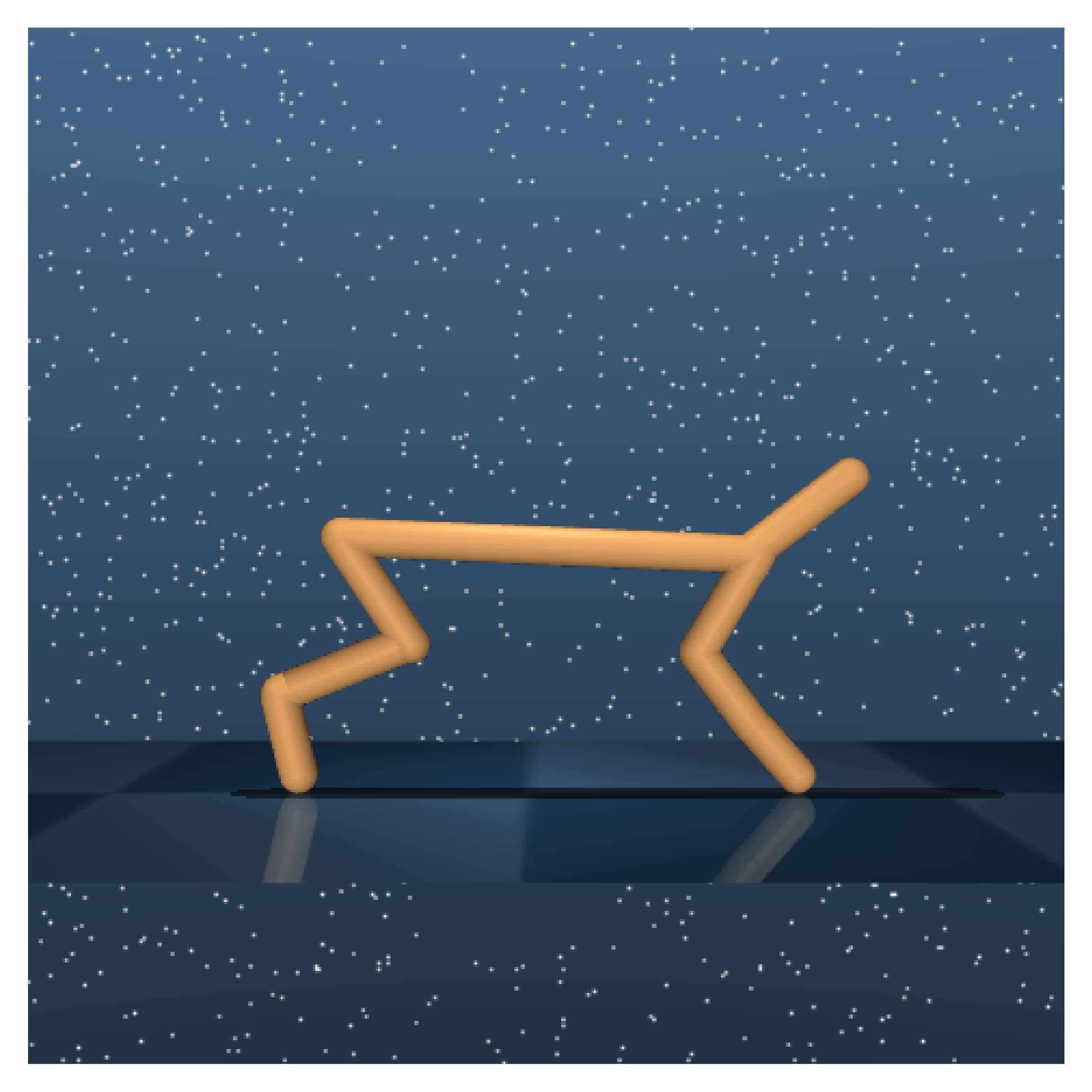}
    \caption{cheetah\_run}
\end{subfigure}
\vspace{-0.1in}
\begin{subfigure}{0.33\textwidth}
    \includegraphics[width=1\textwidth]{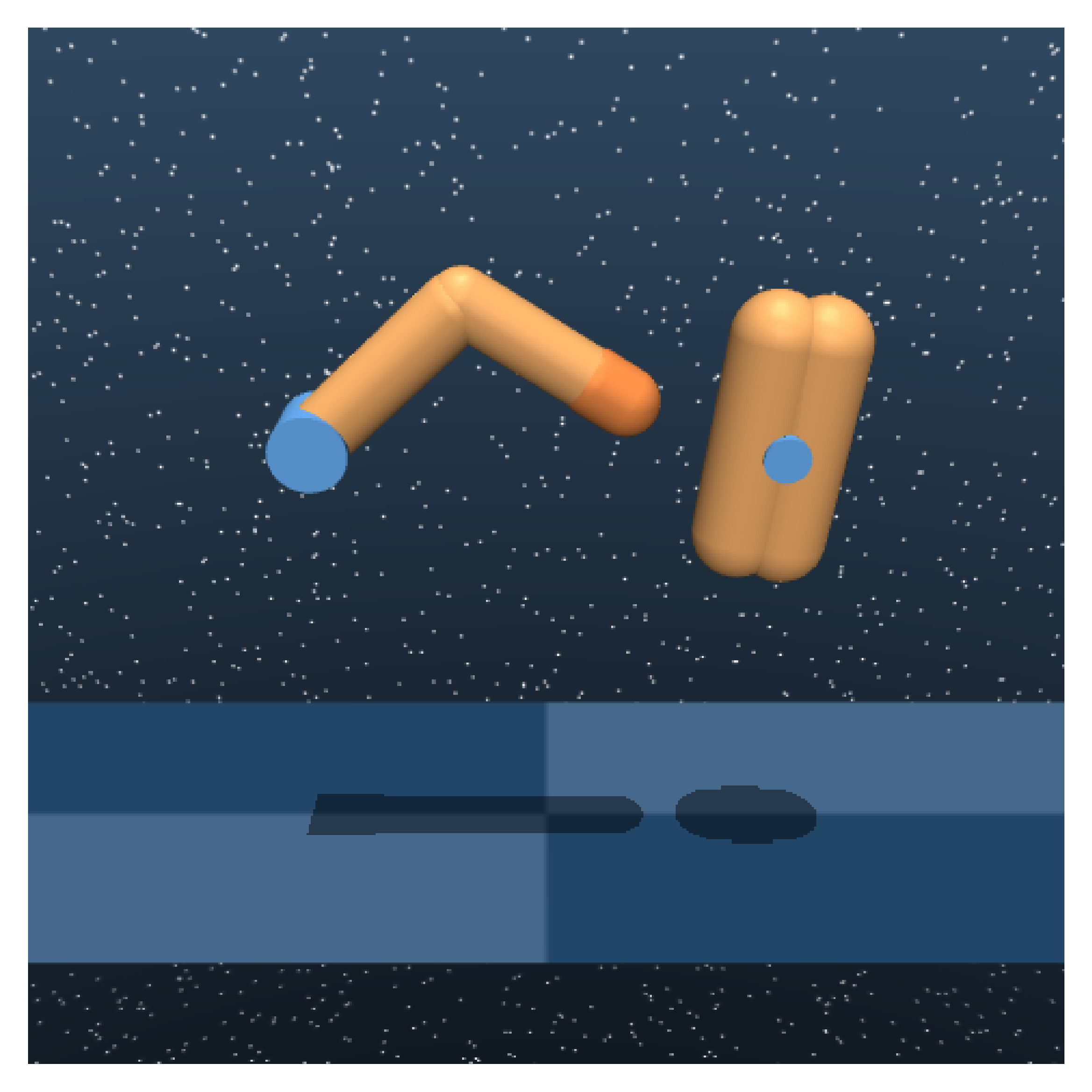}
    \caption{finger\_spin}
\end{subfigure}
\hspace{-0.1in}
\begin{subfigure}{0.33\textwidth}
    \includegraphics[width=1\textwidth]{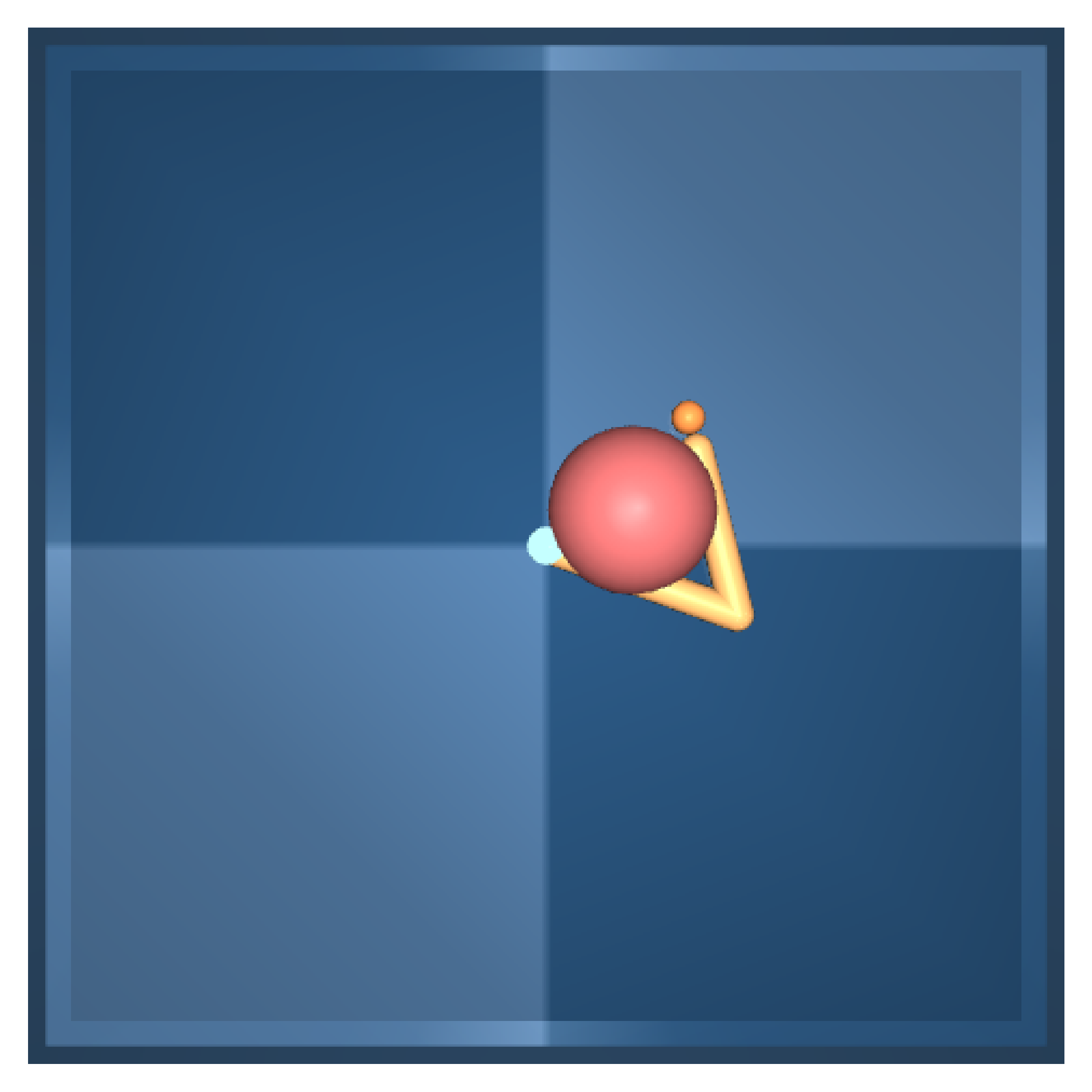}
    \caption{reacher\_easy}
\end{subfigure}
\hspace{-0.1in}
\begin{subfigure}{0.33\textwidth}
    \includegraphics[width=1\textwidth]{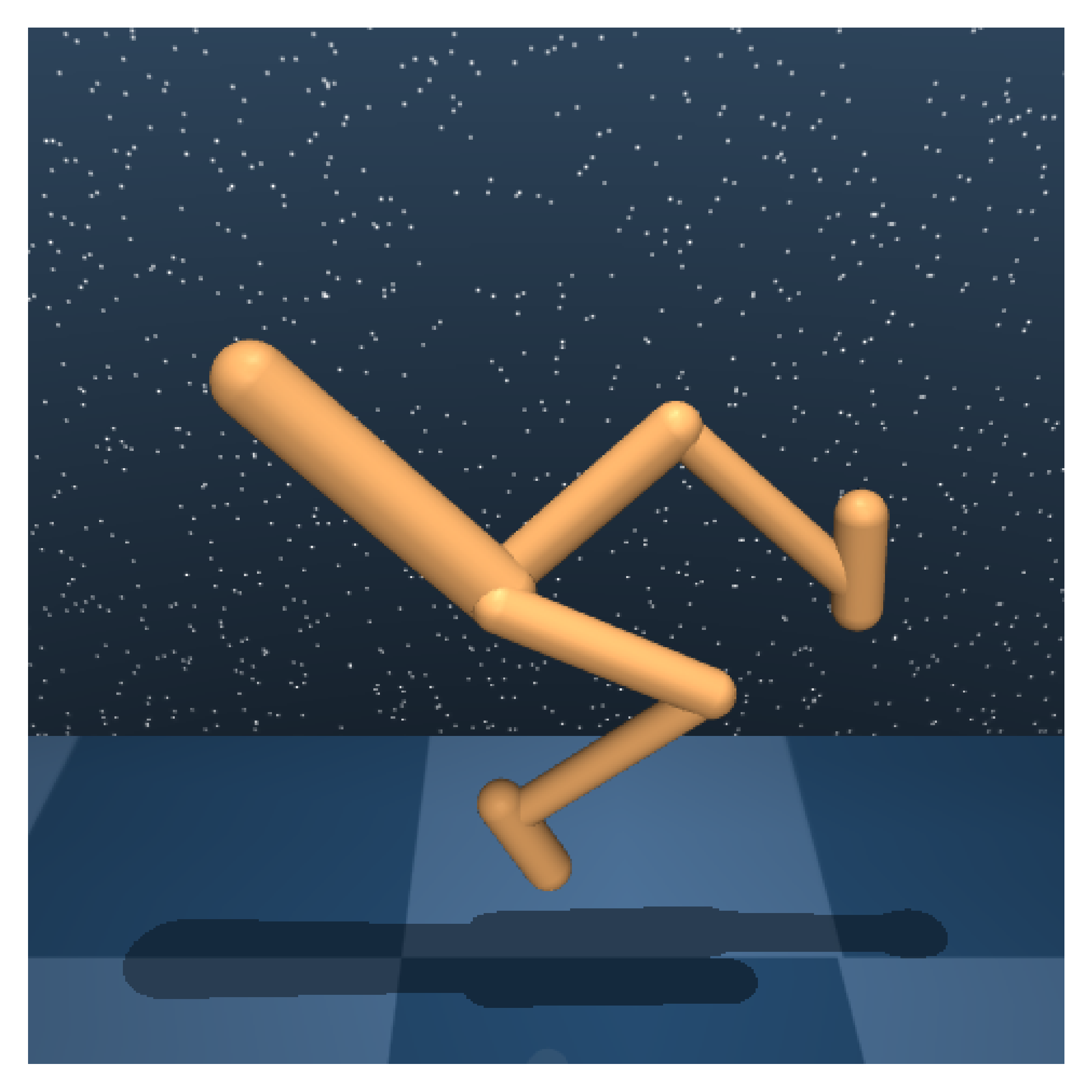}
    \caption{walker\_walk}
\end{subfigure}

\caption{\label{app fig: dmcontrol suite}Images for DMControl suites used in our experiments. Each image is a frame of a specific DMControl suite.}
\end{figure*}

\begin{figure*}[!htbp]
\centering
\begin{subfigure}{1.7in}
    \includegraphics[width=1\textwidth, height=2.0in]{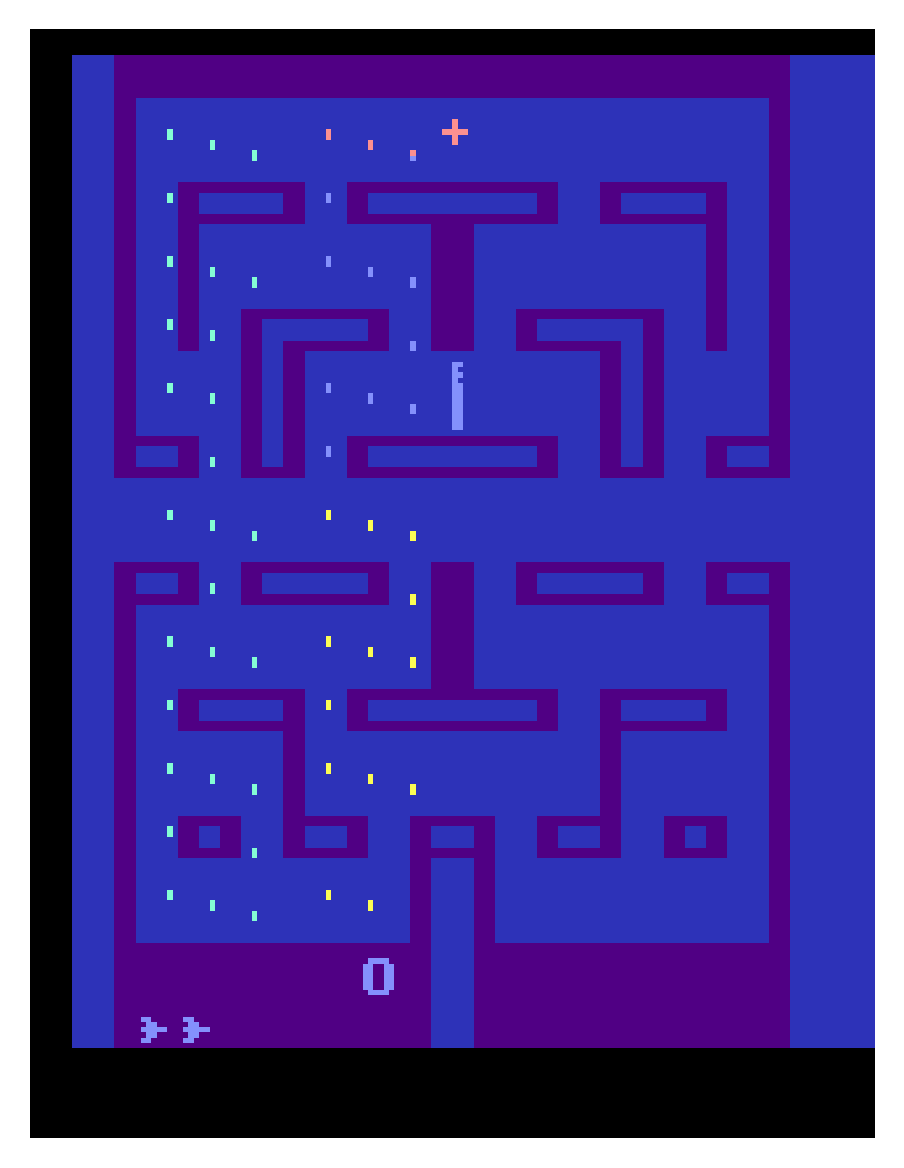}
    \caption{Alien}
\end{subfigure}
\hspace{-0.1in}
\begin{subfigure}{1.7in}
    \includegraphics[width=1\textwidth, height=2.0in]{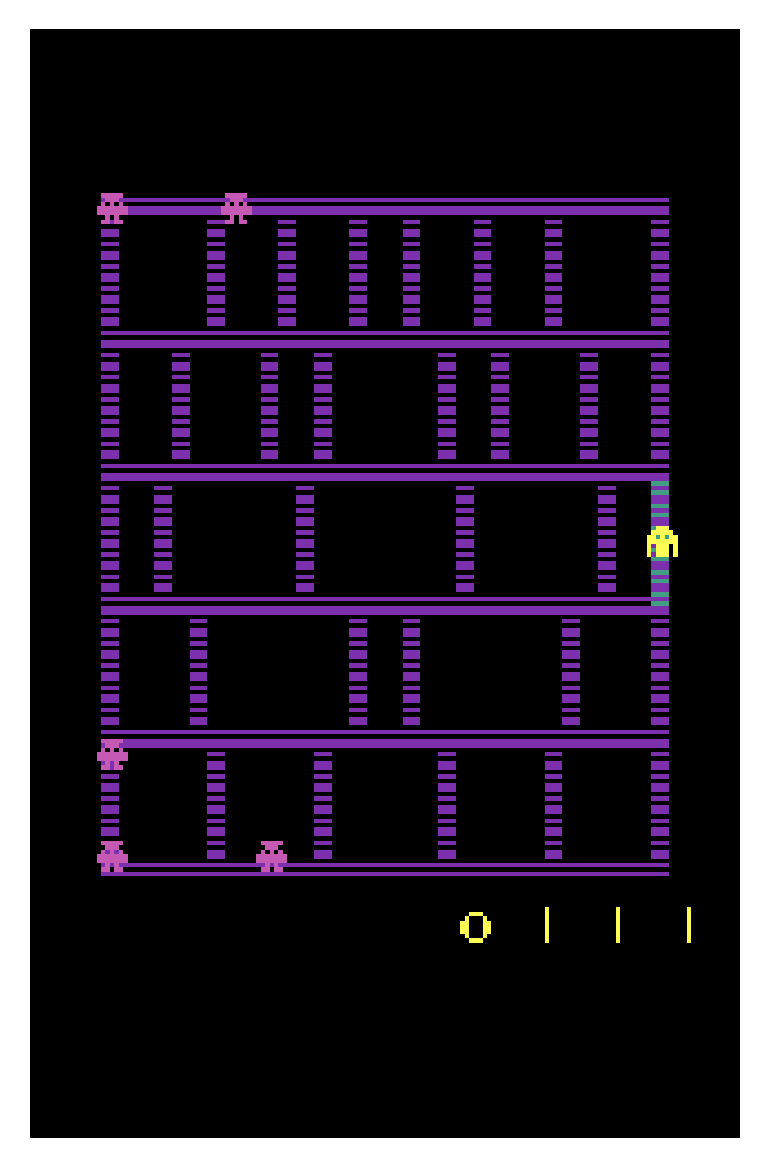}
    \caption{Amidar}
\end{subfigure}
\hspace{-0.1in}
\begin{subfigure}{1.7in}
    \includegraphics[width=1\textwidth, height=2.0in]{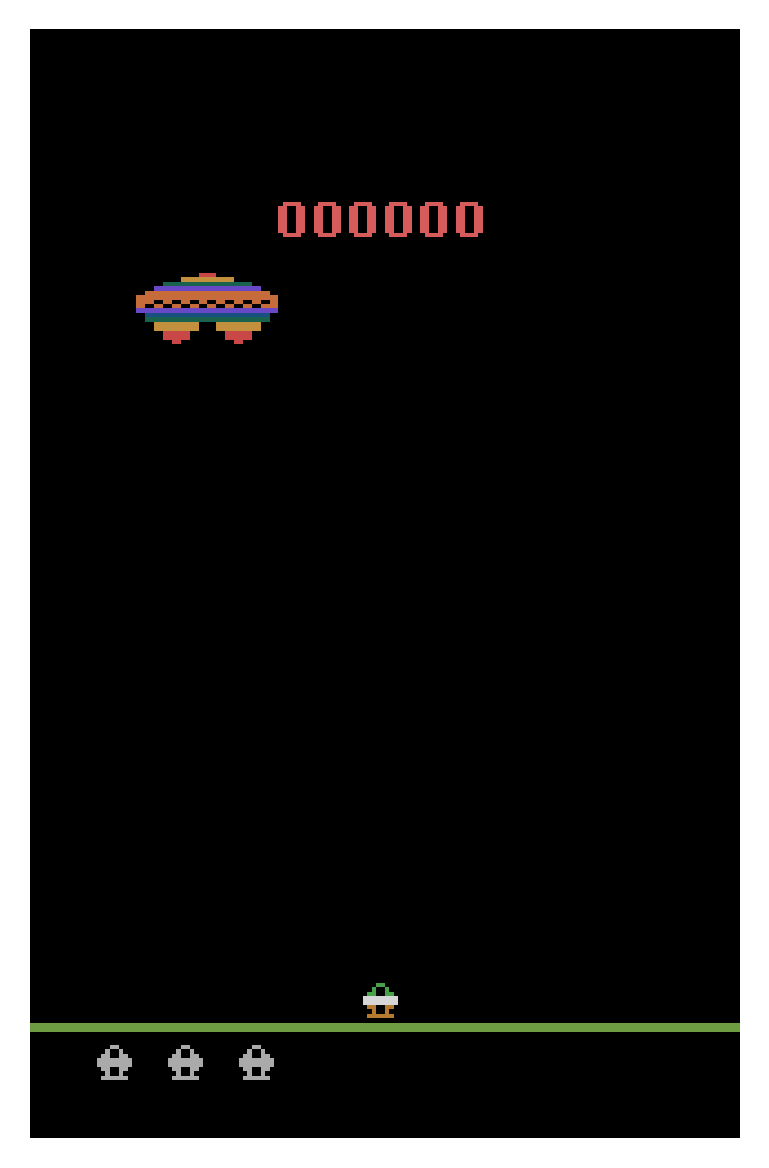}
    \caption{Assault}
\end{subfigure}
\hspace{-0.1in}
\begin{subfigure}{1.7in}
    \includegraphics[width=1\textwidth, height=2.0in]{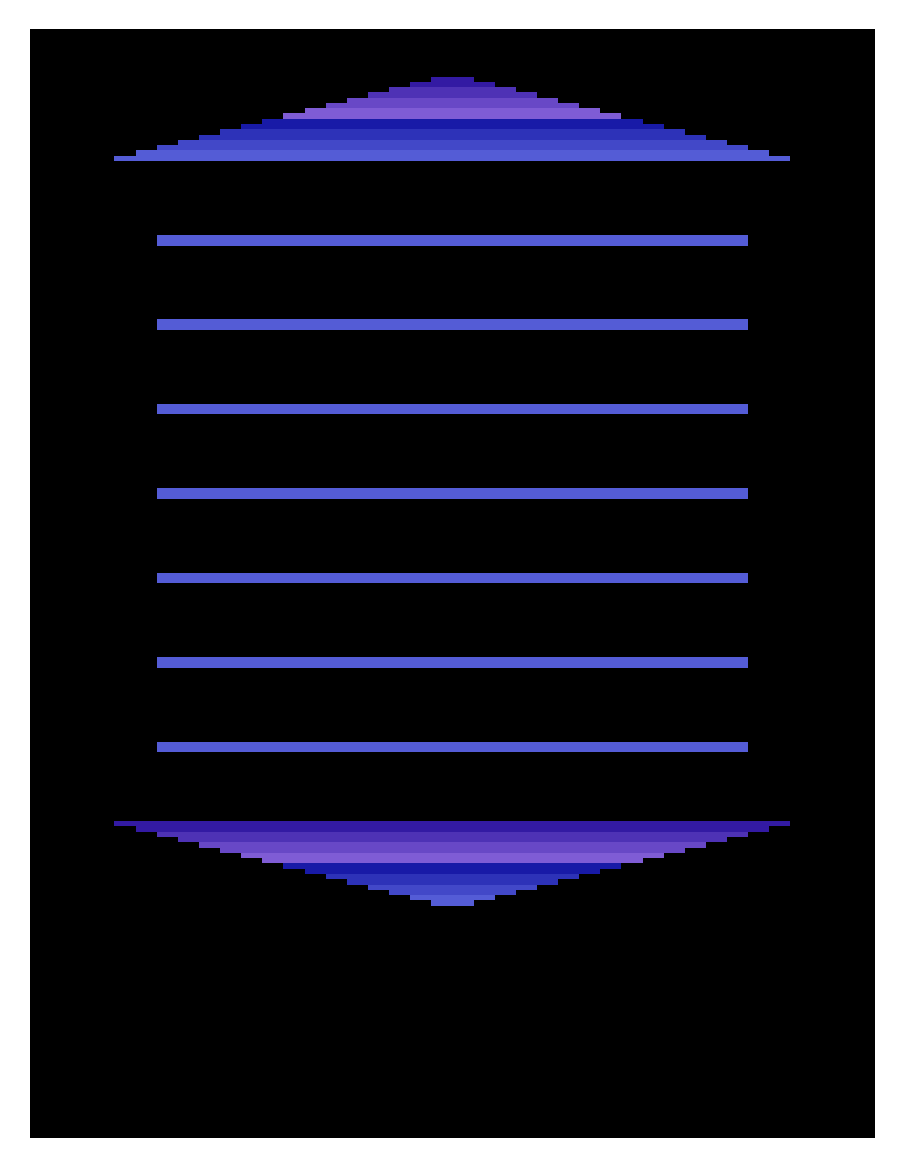}
    \caption{Asterix}
\end{subfigure}
\vspace{-0.0in}
\begin{subfigure}{1.7in}
    \includegraphics[width=1\textwidth, height=2.0in]{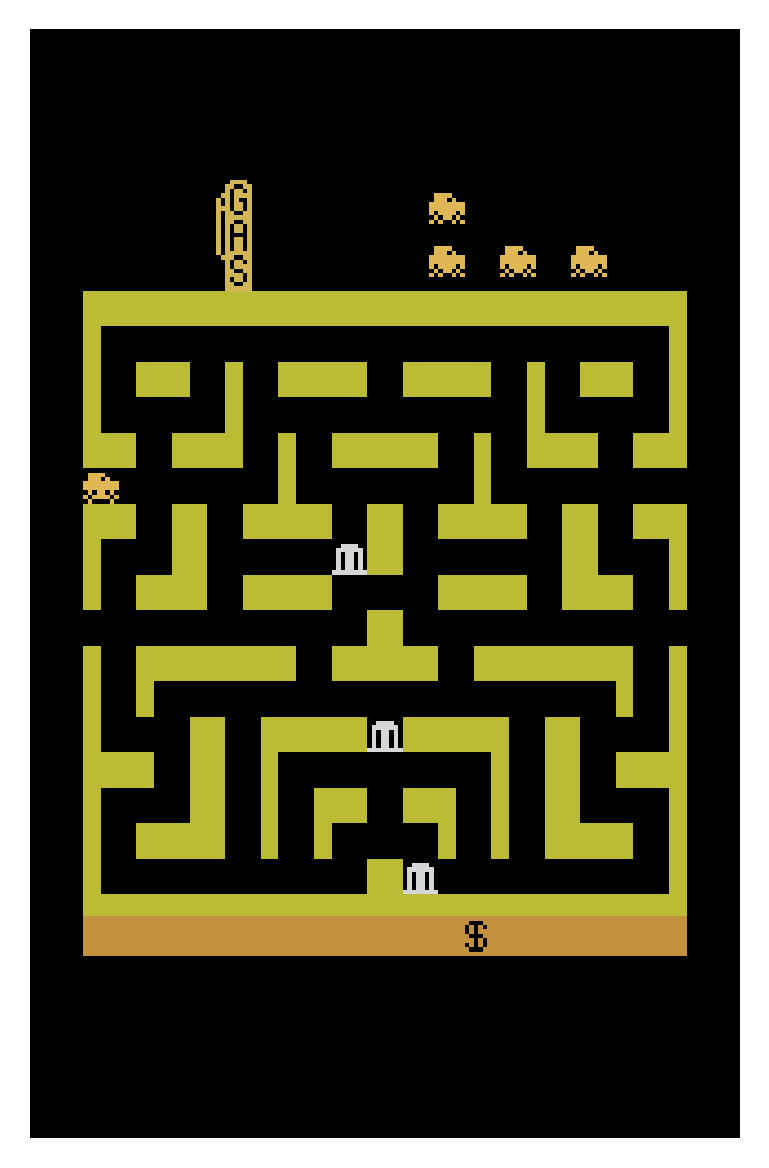}
    \caption{BankHeist}
\end{subfigure}
\hspace{-0.1in}
\begin{subfigure}{1.7in}
    \includegraphics[width=1\textwidth, height=2.0in]{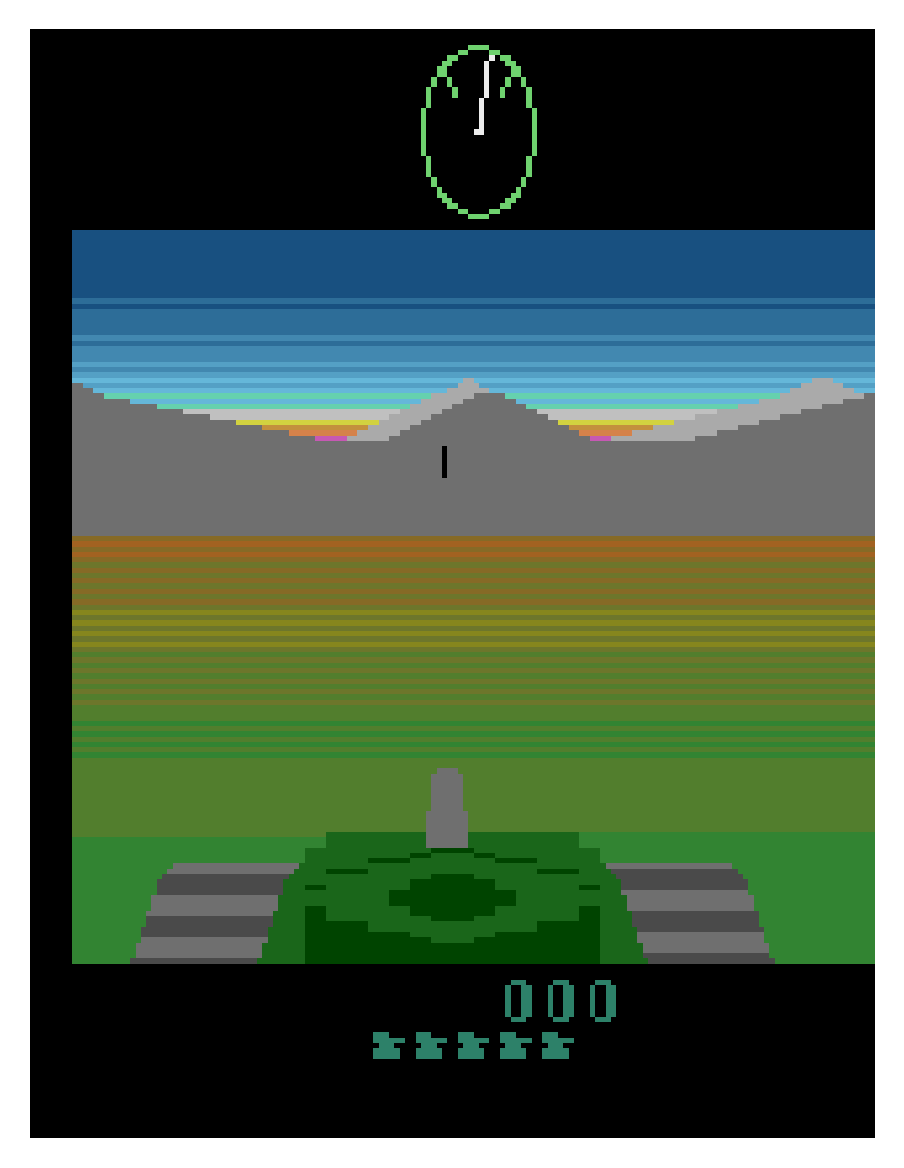}
    \caption{BattleZone}
\end{subfigure}
\hspace{-0.1in}
\begin{subfigure}{1.7in}
    \includegraphics[width=1\textwidth, height=2.0in]{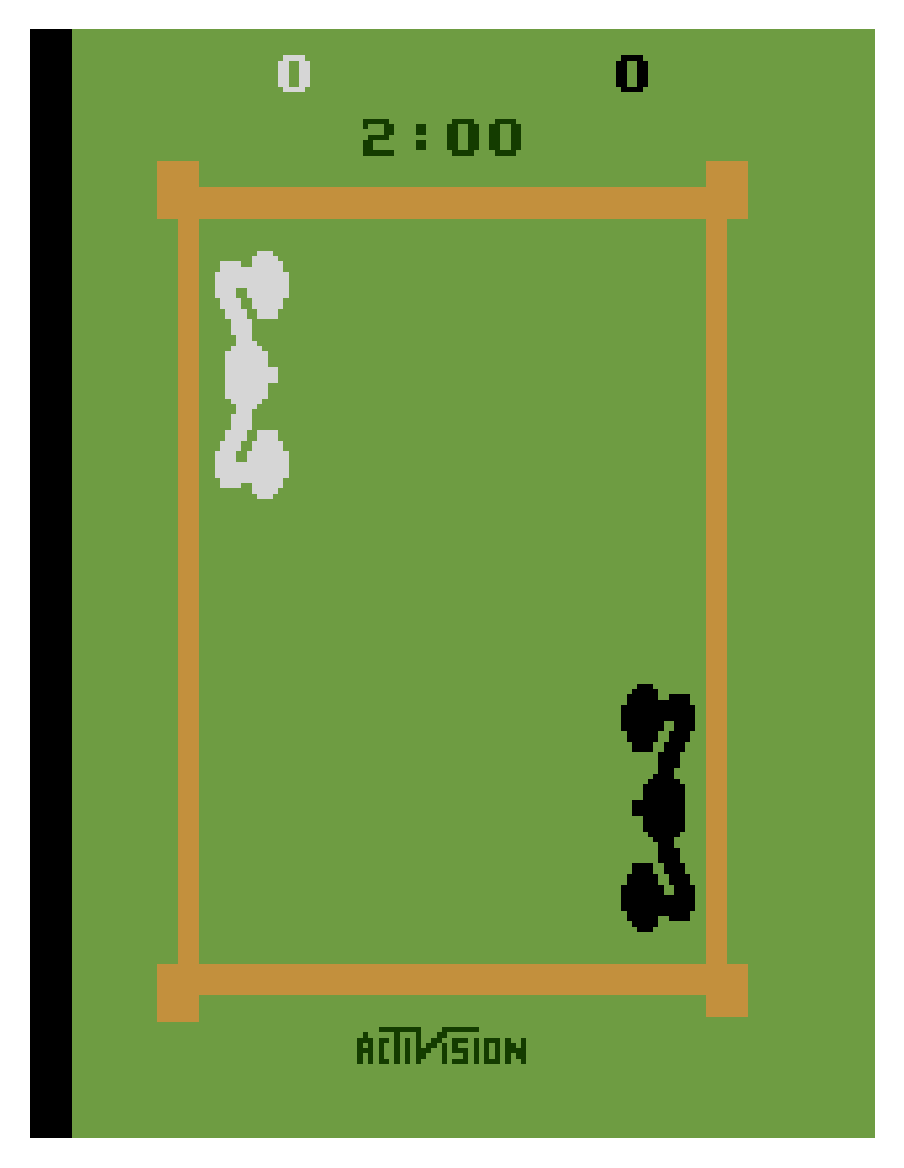}
    \caption{Boxing}
\end{subfigure}
\hspace{-0.1in}
\begin{subfigure}{1.7in}
    \includegraphics[width=1\textwidth, height=2.0in]{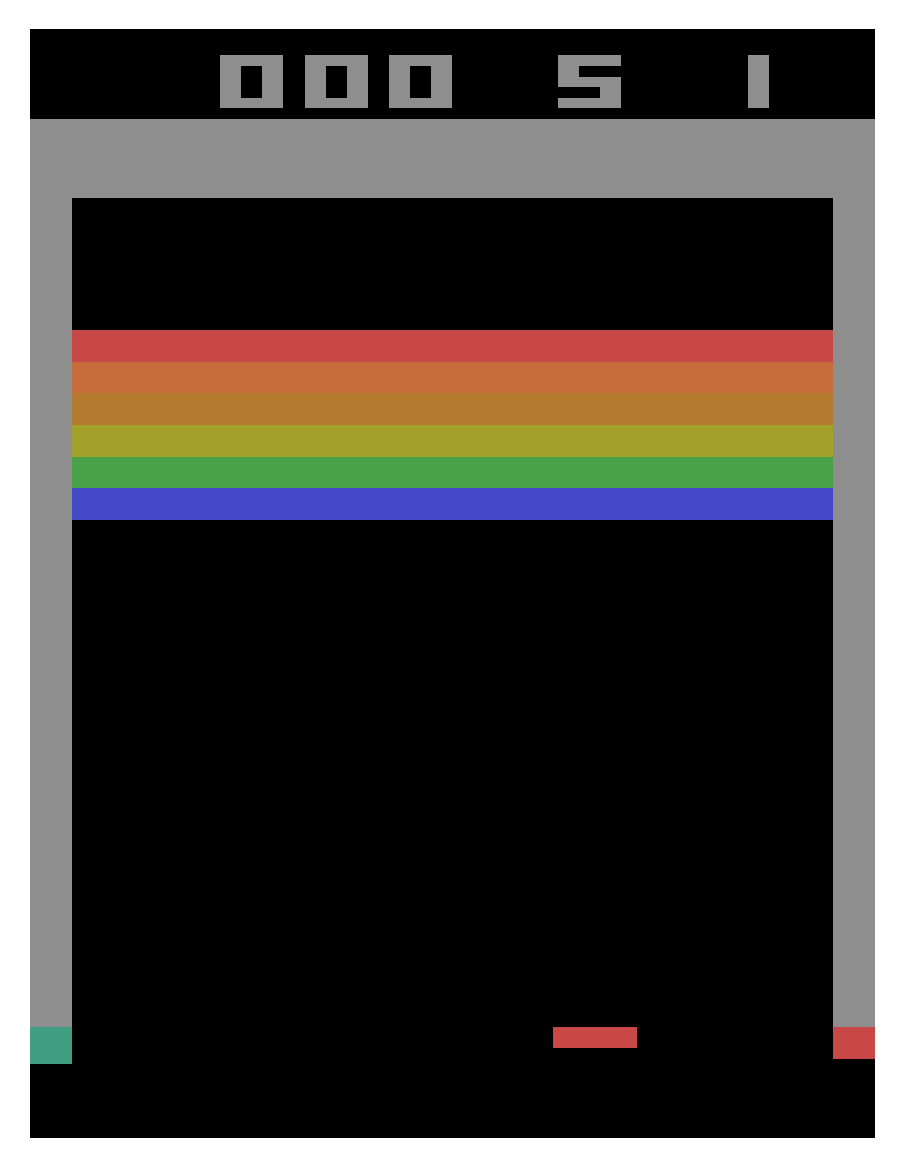}
    \caption{Breakout}
\end{subfigure}
\vspace{-0.0in}
\begin{subfigure}{1.7in}
    \includegraphics[width=1\textwidth, height=2.0in]{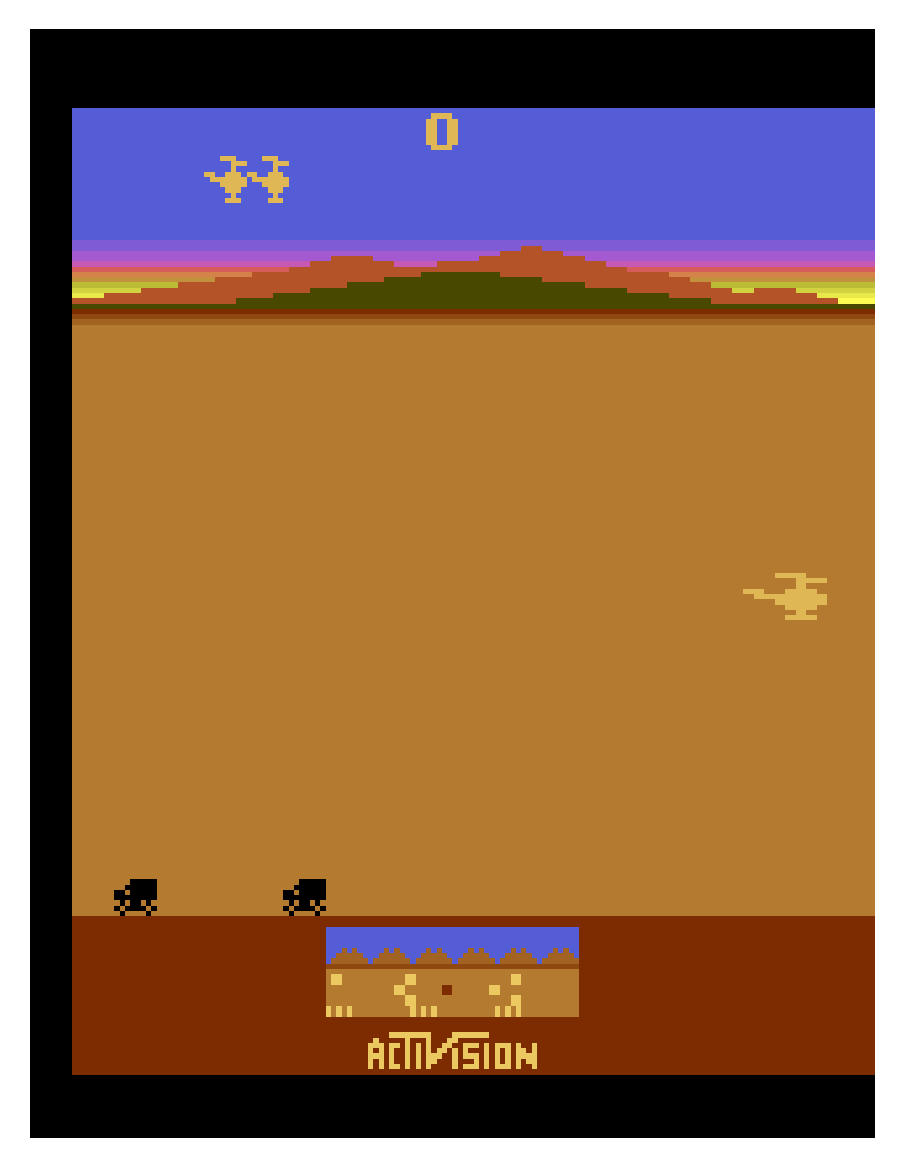}
    \caption{ChopperCommand}
\end{subfigure}
\hspace{-0.1in}
\begin{subfigure}{1.7in}
    \includegraphics[width=1\textwidth, height=2.0in]{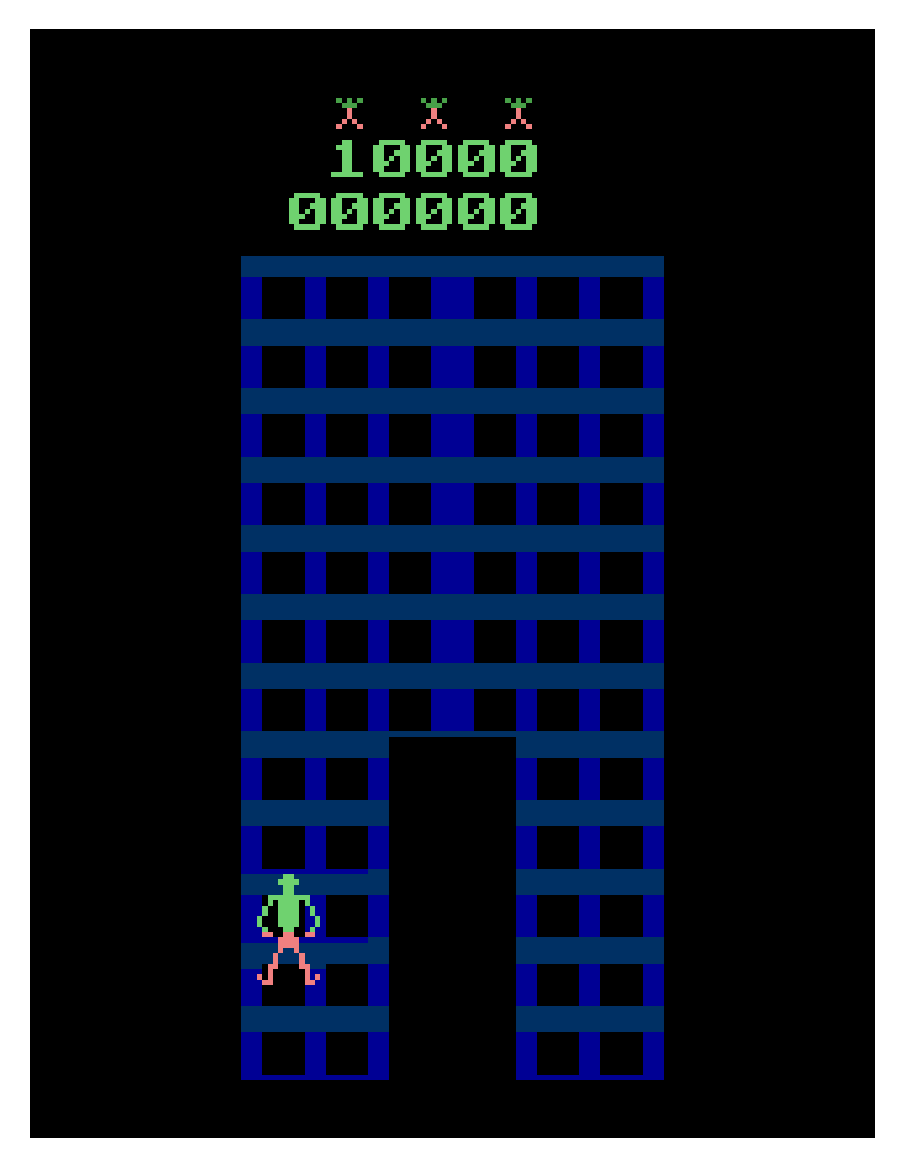}
    \caption{CrazyClimber}
\end{subfigure}
\hspace{-0.1in}
\begin{subfigure}{1.7in}
    \includegraphics[width=1\textwidth, height=2.0in]{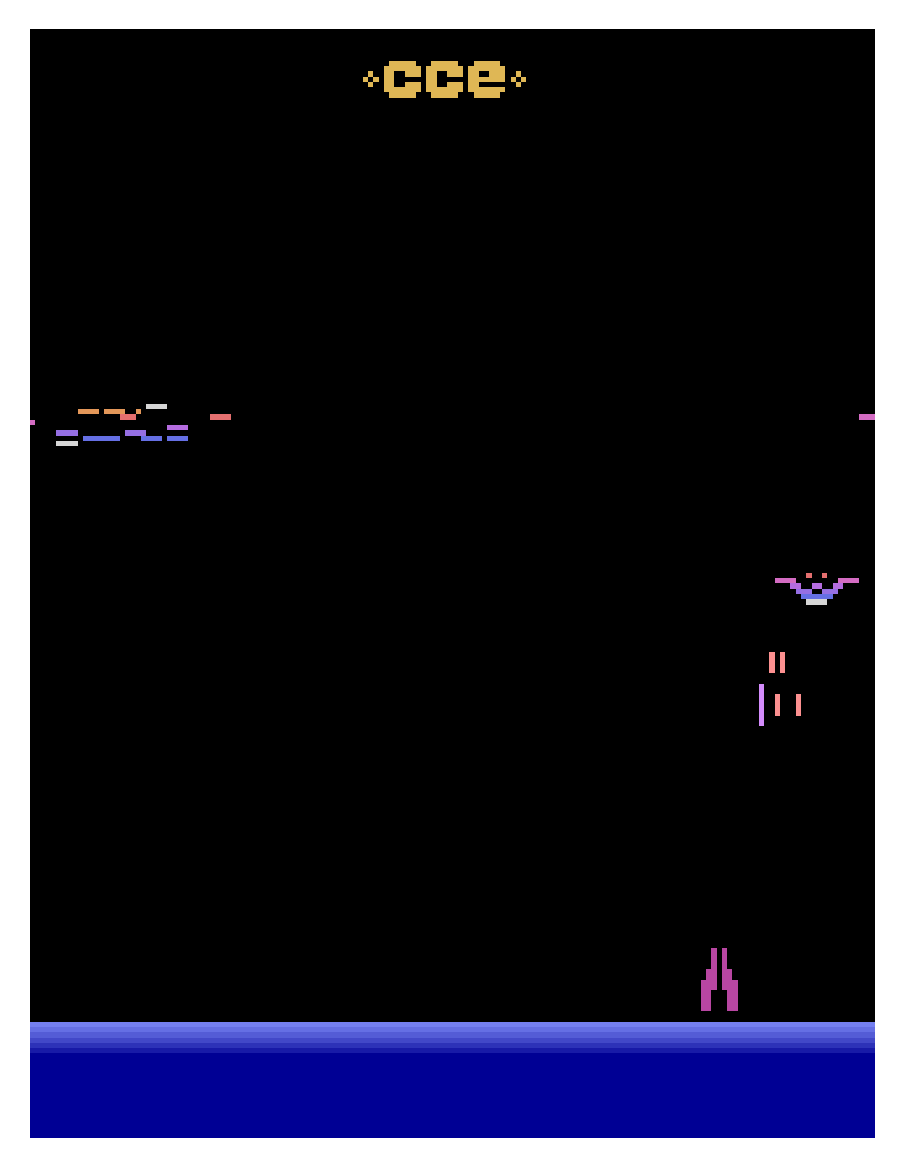}
    \caption{DemonAttack}
\end{subfigure}
\hspace{-0.1in}
\begin{subfigure}{1.7in}
    \includegraphics[width=1\textwidth, height=2.0in]{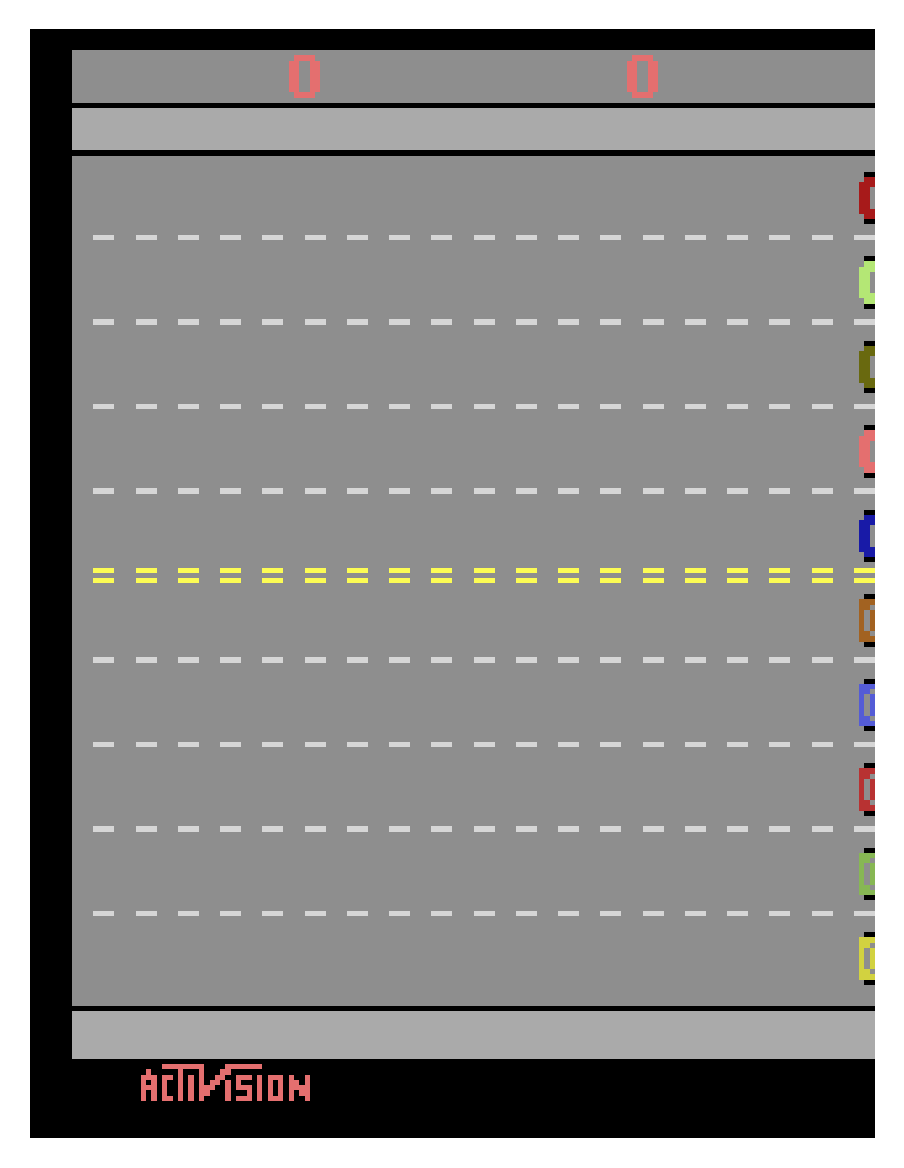}
    \caption{Freeway}
\end{subfigure}
\vspace{-0.0in}
\begin{subfigure}{1.7in}
    \includegraphics[width=1\textwidth, height=2.0in]{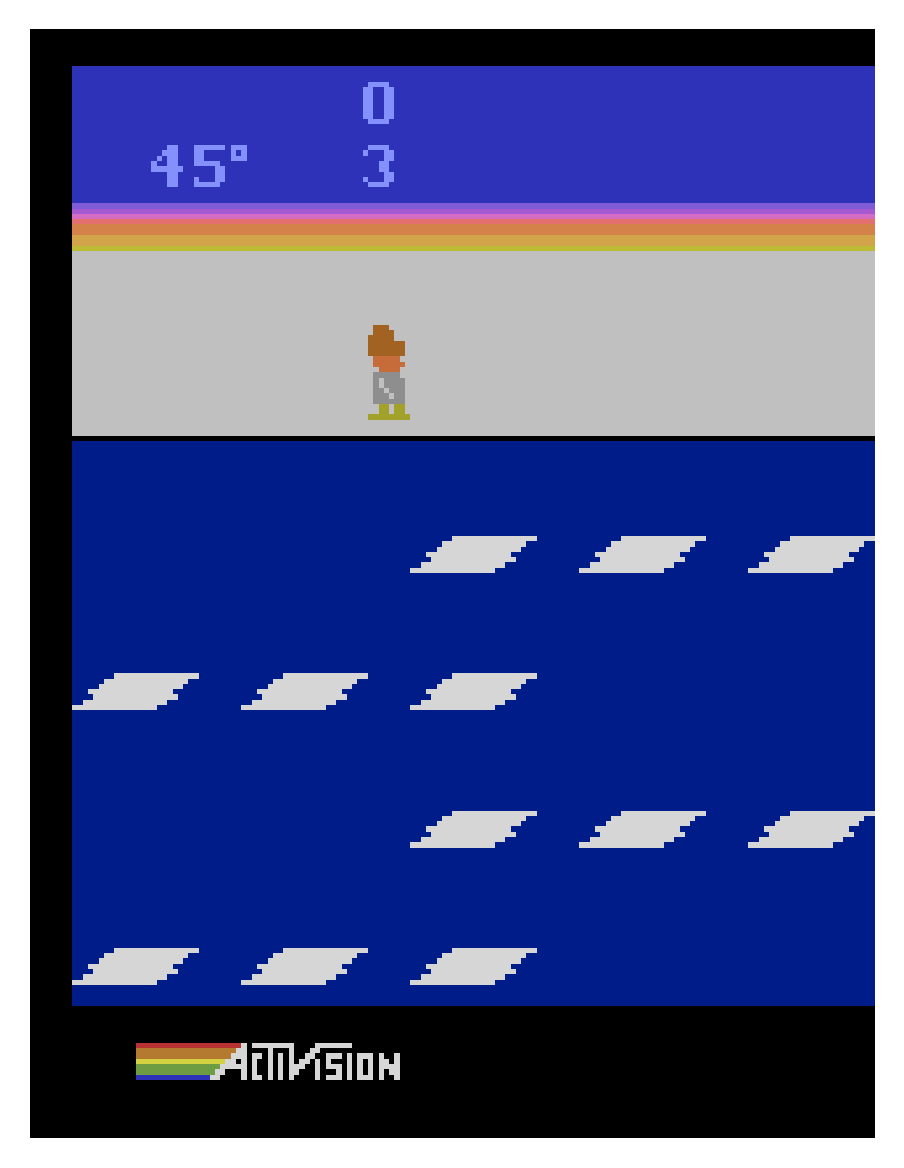}
    \caption{Frostbite}
\end{subfigure}
\hspace{-0.1in}
\begin{subfigure}{1.7in}
    \includegraphics[width=1\textwidth, height=2.0in]{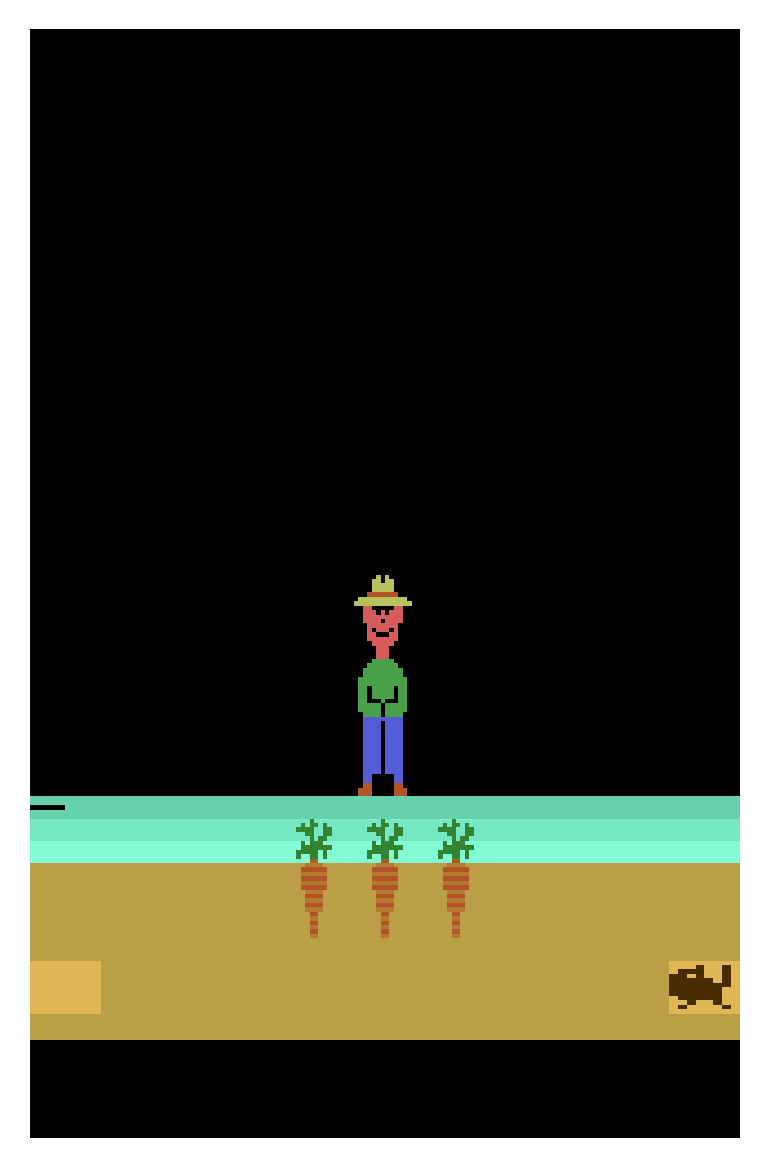}
    \caption{Gopher}
\end{subfigure}
\hspace{-0.1in}
\begin{subfigure}{1.7in}
    \includegraphics[width=1\textwidth, height=2.0in]{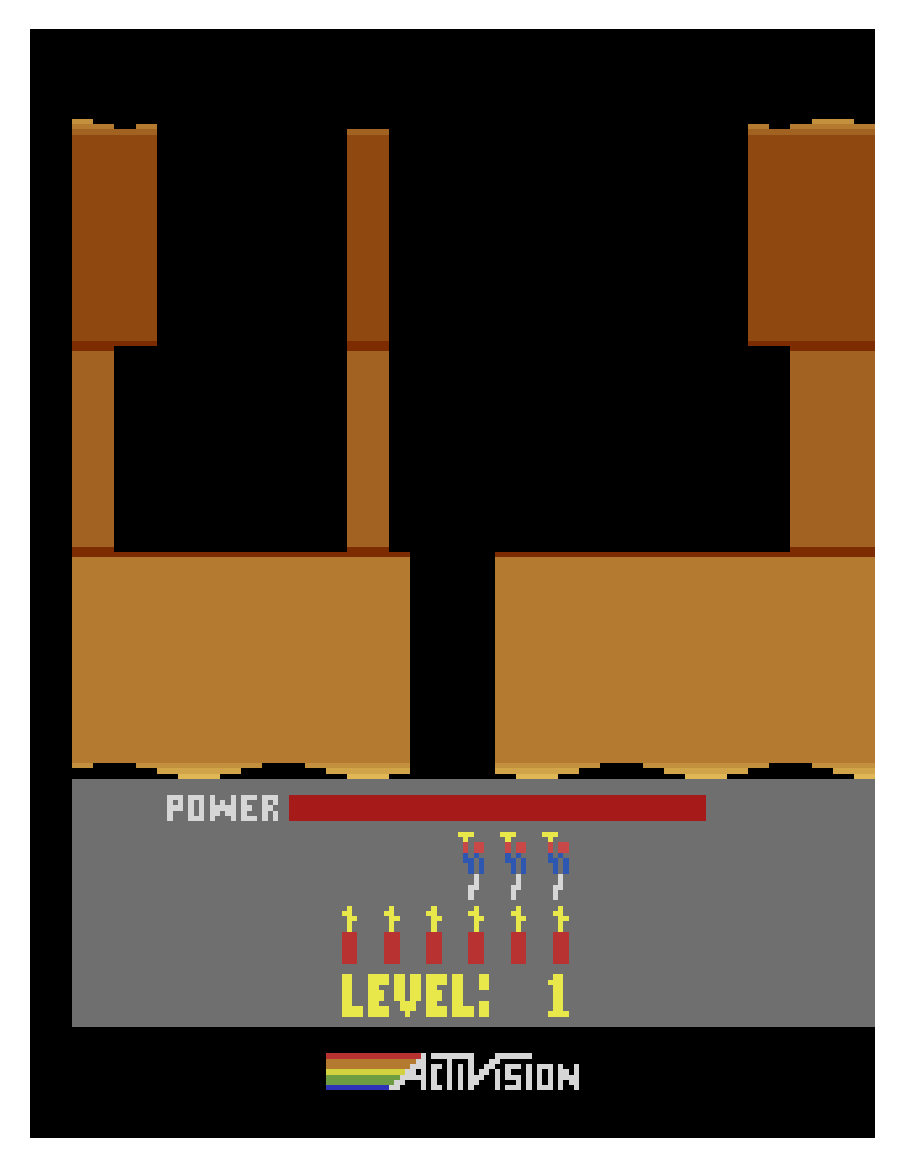}
    \caption{Hero}
\end{subfigure}
\hspace{-0.1in}
\begin{subfigure}{1.7in}
    \includegraphics[width=1\textwidth, height=2.0in]{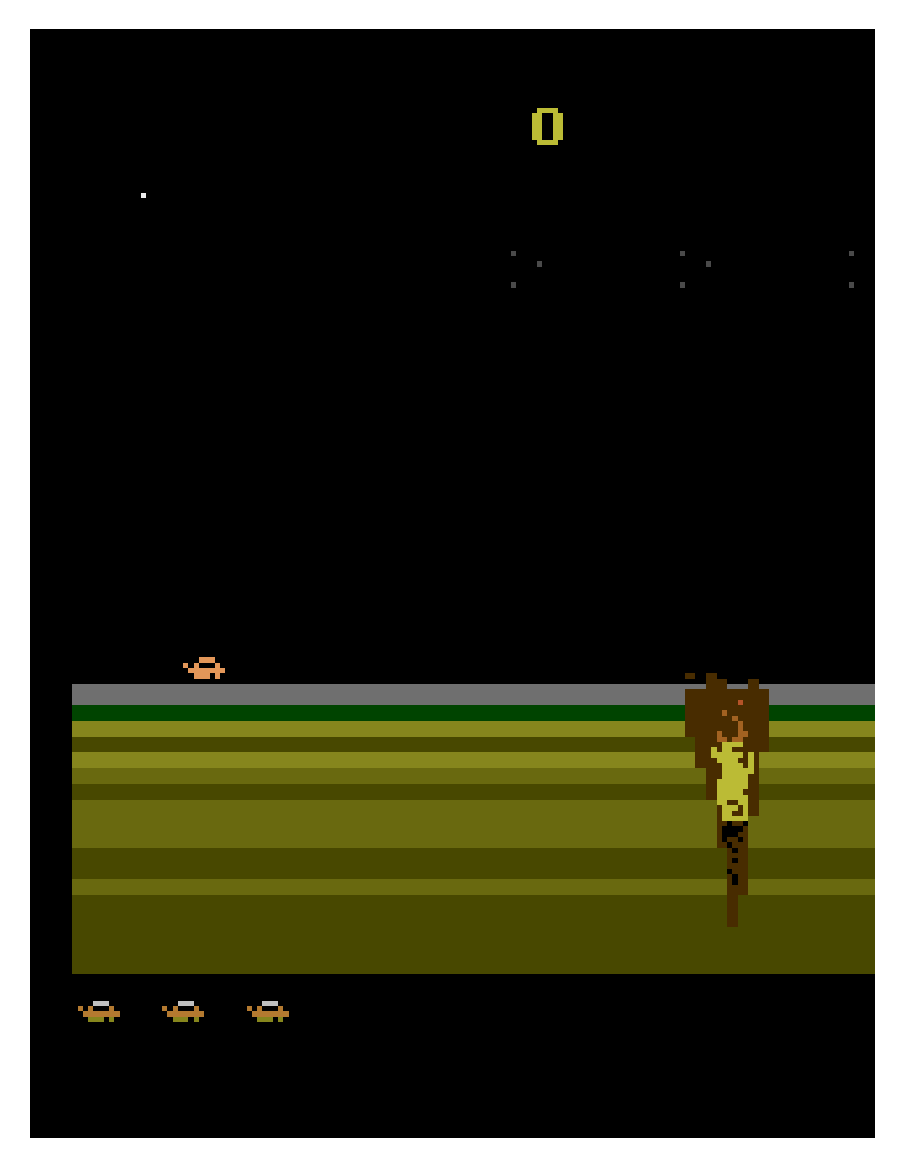}
    \caption{Jamesbond}
\end{subfigure}
\vspace{-0.0in}
\caption{\label{app fig: atari images}Images for Atari100k suites used in our experiments. Each image is a frame of a specific Atari game.}
\end{figure*}
\begin{figure*}[!htbp]
\centering
\begin{subfigure}{1.7in}
    \includegraphics[width=1\textwidth, height=2.0in]{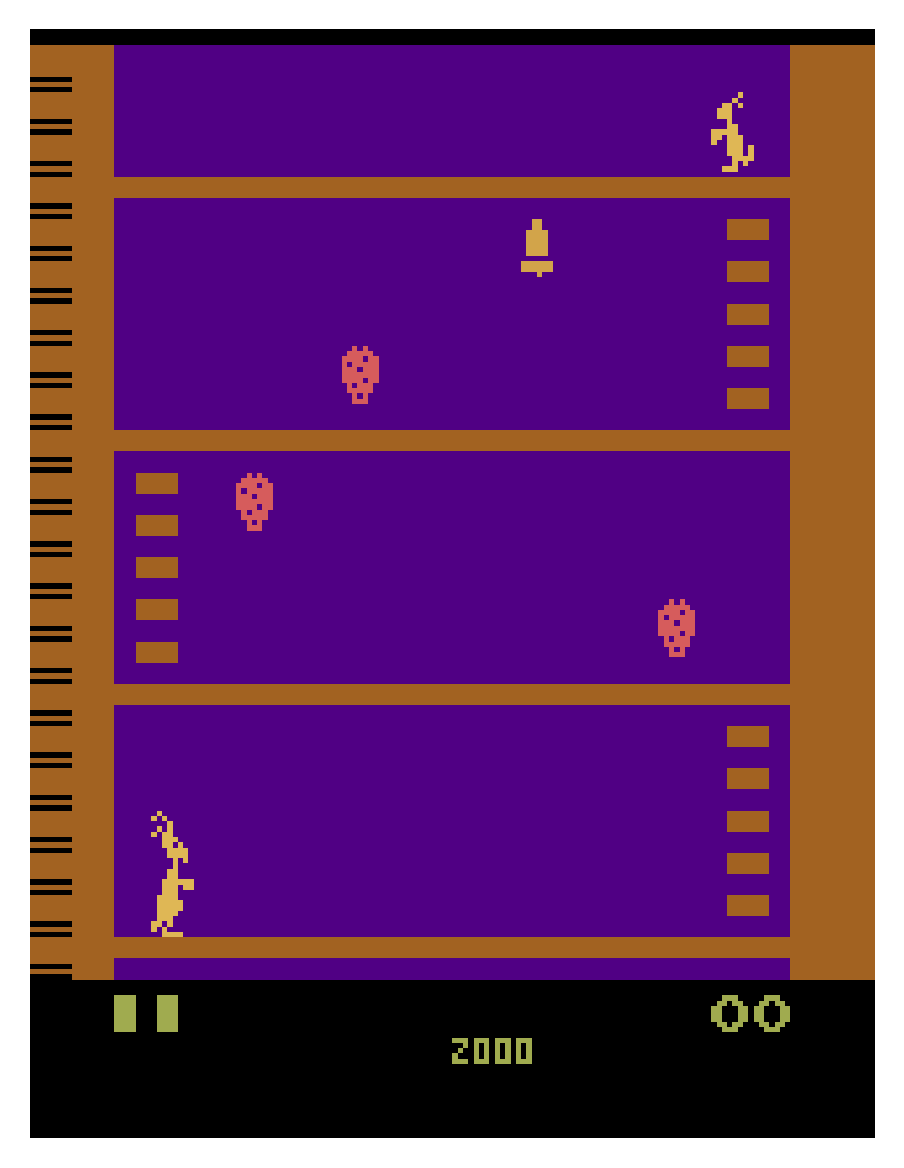}
    \caption{Kangaroo}
\end{subfigure}
\hspace{-0.1in}
\begin{subfigure}{1.7in}
    \includegraphics[width=1\textwidth, height=2.0in]{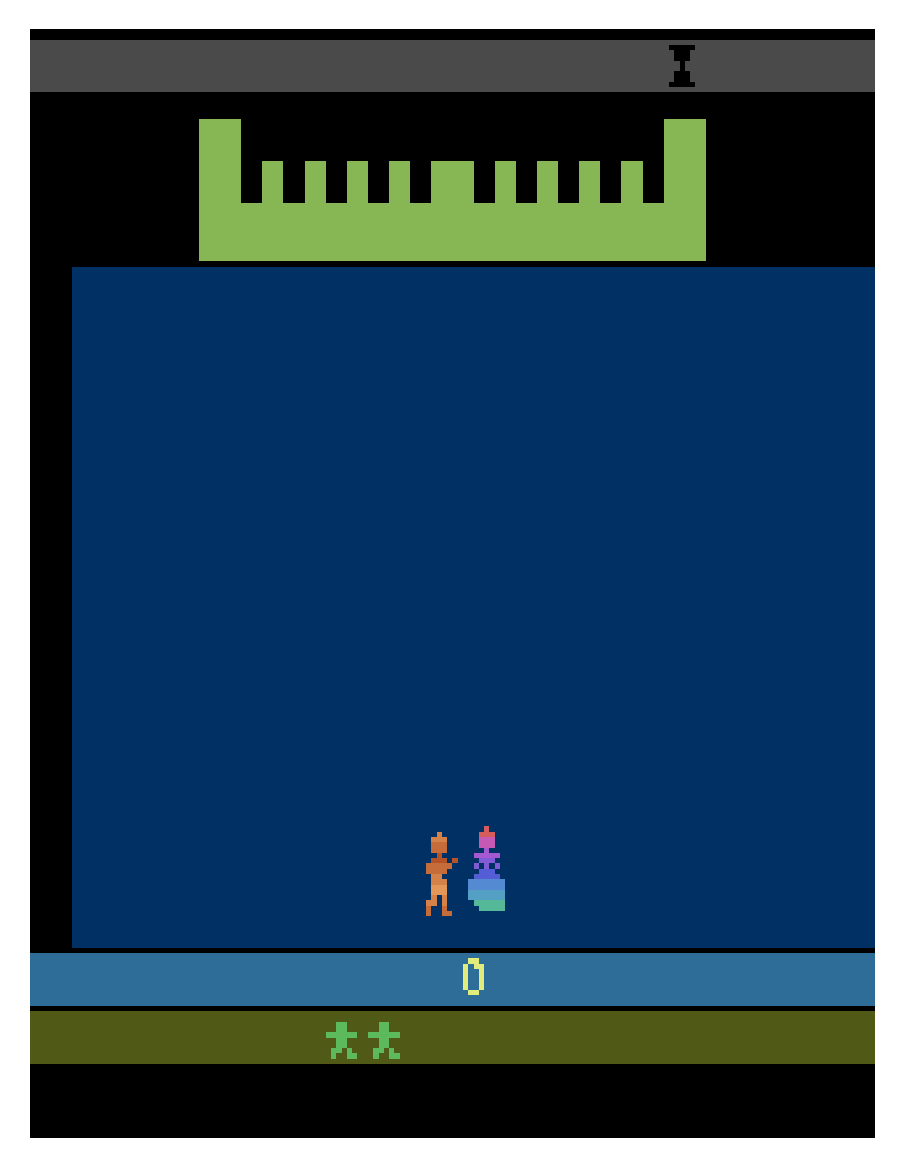}
    \caption{Krull}
\end{subfigure}
\hspace{-0.1in}
\begin{subfigure}{1.7in}
    \includegraphics[width=1\textwidth, height=2.0in]{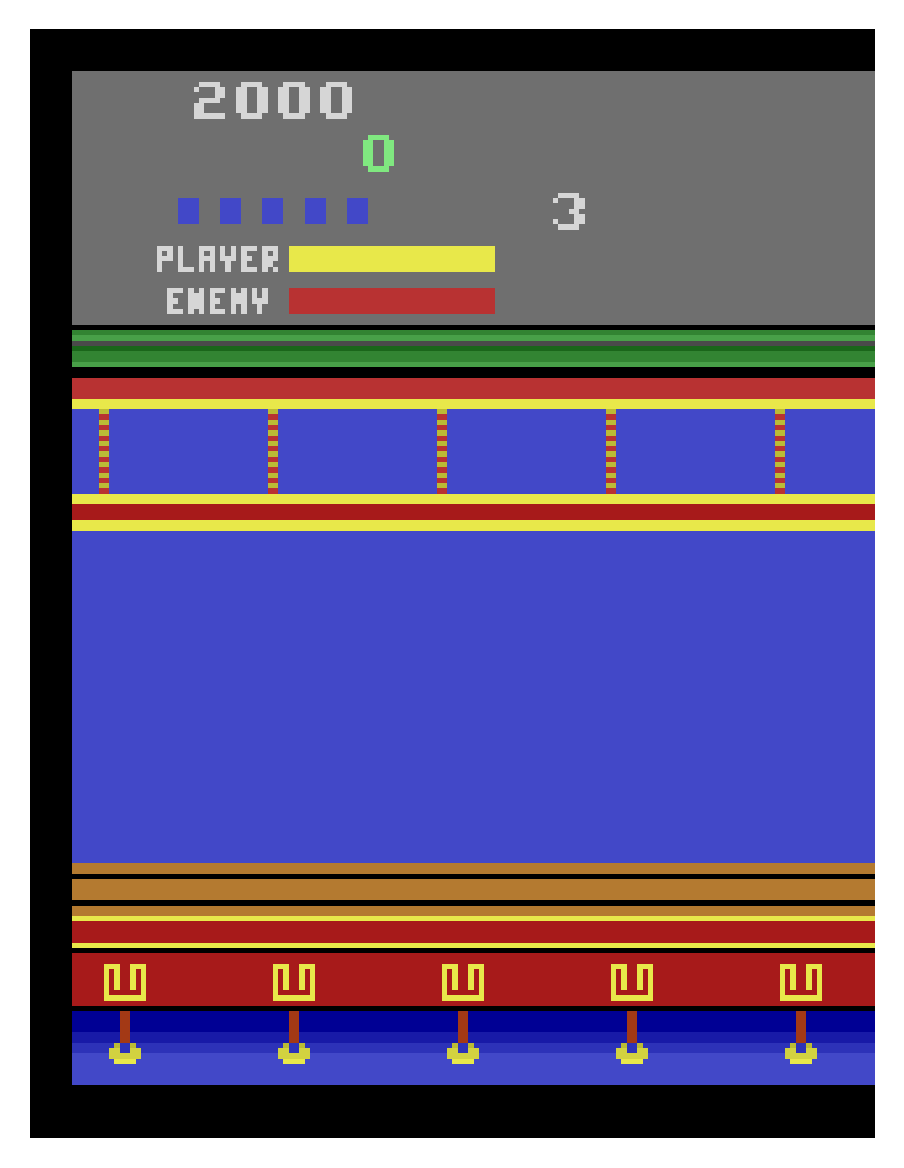}
    \caption{KungFuMaster}
\end{subfigure}
\hspace{-0.1in}
\begin{subfigure}{1.7in}
    \includegraphics[width=1\textwidth, height=2.0in]{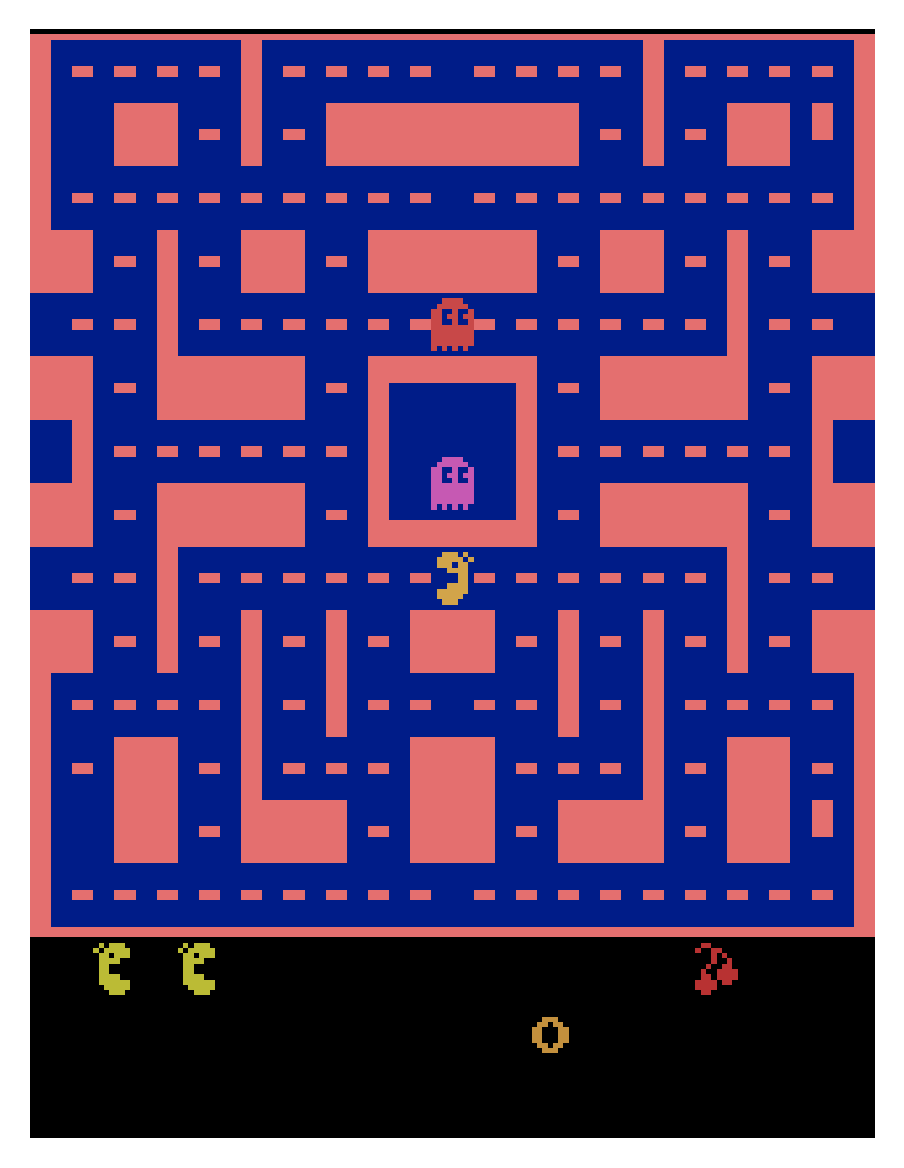}
    \caption{MsPacman}
\end{subfigure}
\begin{subfigure}{1.7in}
    \includegraphics[width=1\textwidth, height=2.0in]{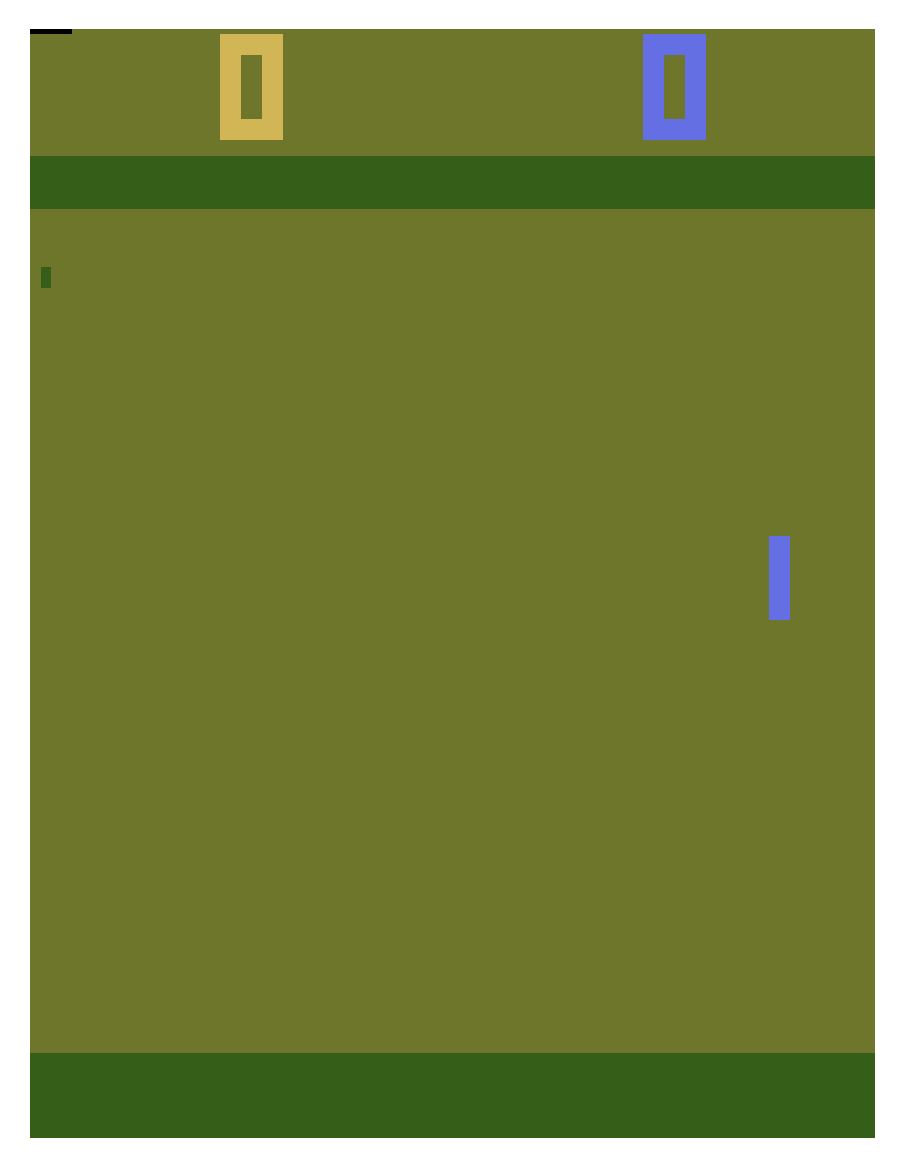}
    \caption{Pong}
\end{subfigure}
\hspace{-0.1in}
\begin{subfigure}{1.7in}
    \includegraphics[width=1\textwidth, height=2.0in]{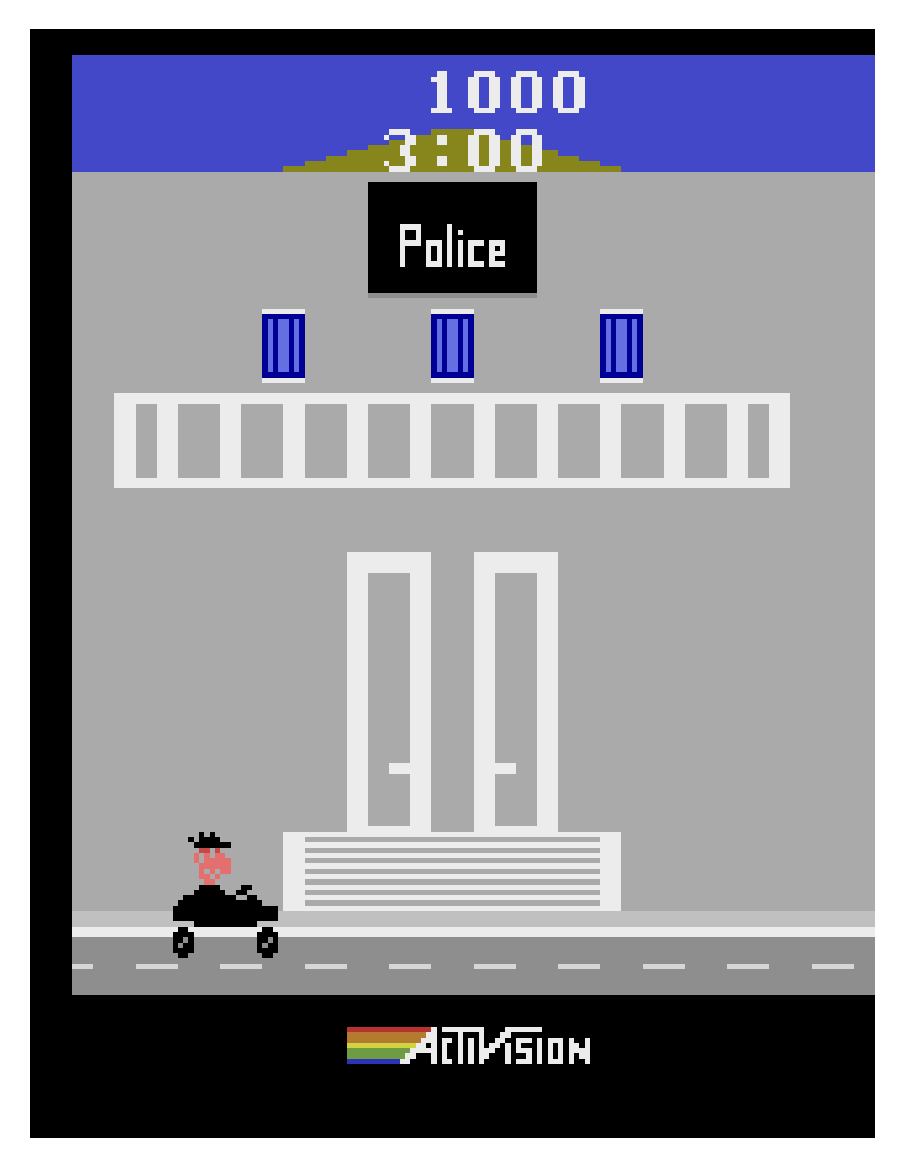}
    \caption{PrivateEye}
\end{subfigure}
\hspace{-0.1in}
\begin{subfigure}{1.7in}
    \includegraphics[width=1\textwidth, height=2.0in]{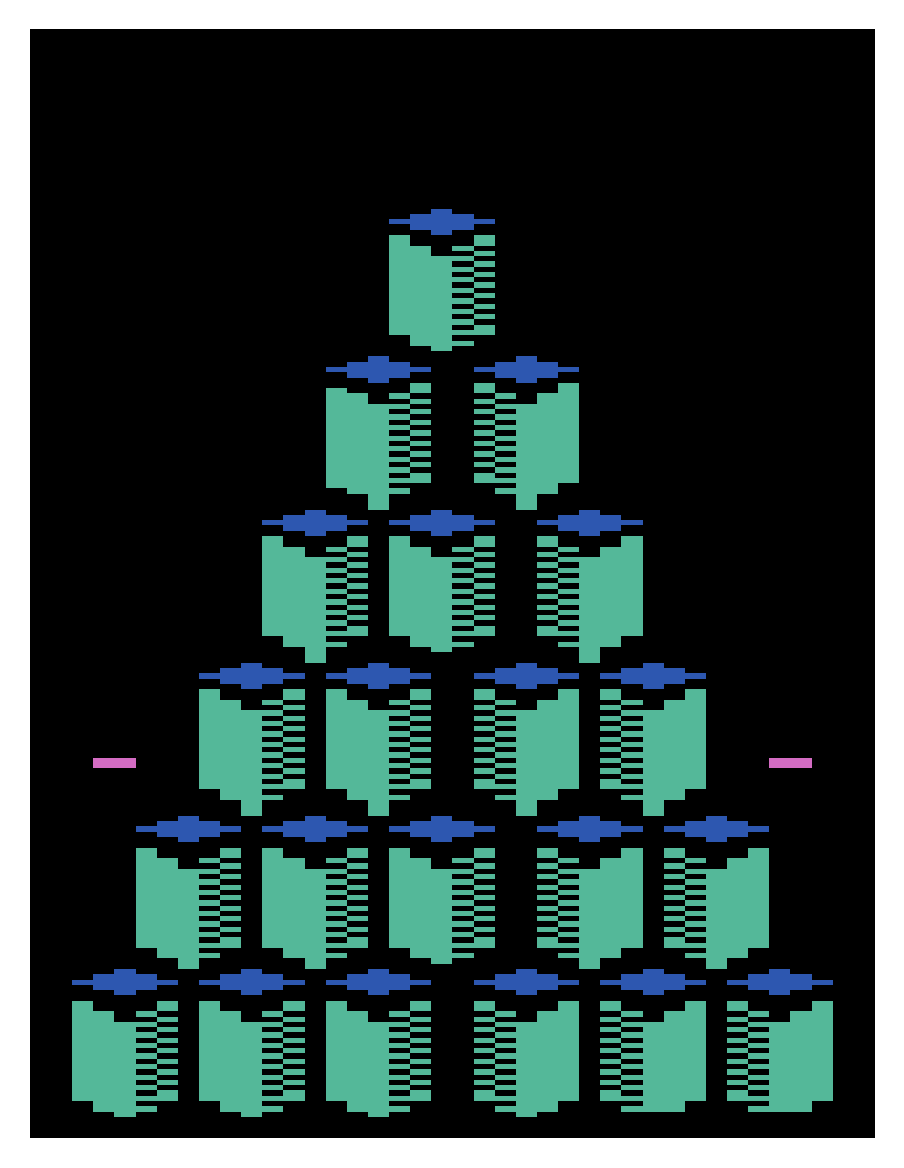}
    \caption{Qbert}
\end{subfigure}
\hspace{-0.1in}
\begin{subfigure}{1.7in}
    \includegraphics[width=1\textwidth, height=2.0in]{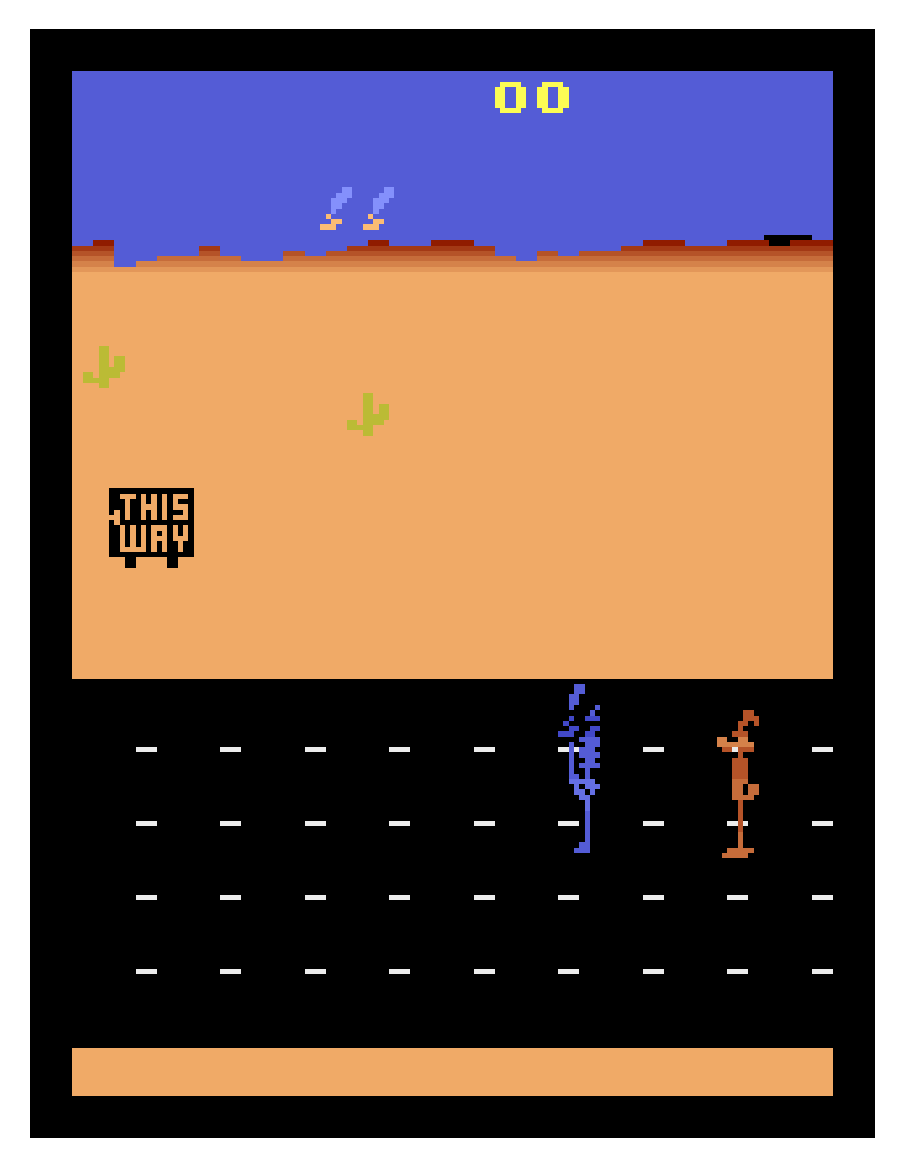}
    \caption{RoadRunner}
\end{subfigure}
\begin{subfigure}{1.7in}
    \includegraphics[width=1\textwidth, height=2.0in]{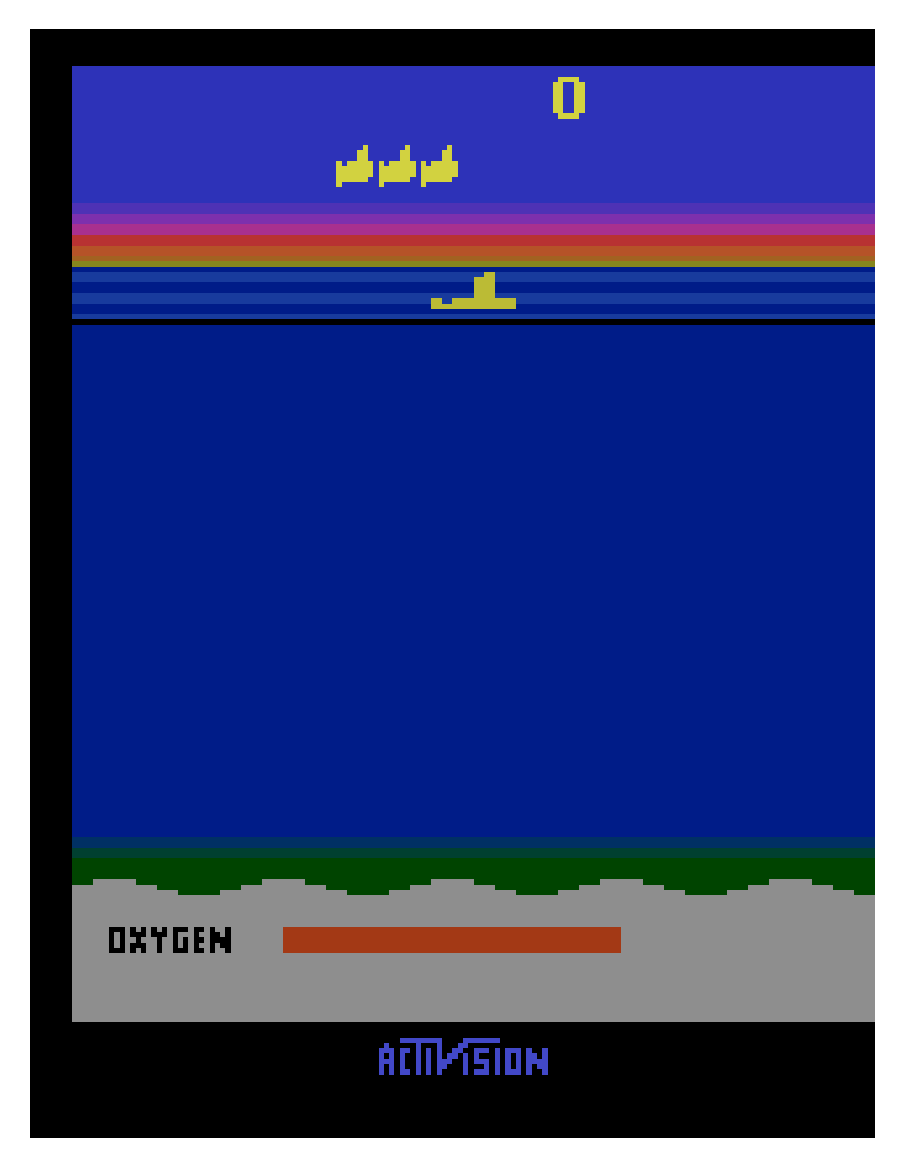}
    \caption{Seaquest}
\end{subfigure}
\hspace{-0.1in}
\begin{subfigure}{1.7in}
    \includegraphics[width=1\textwidth, height=2.0in]{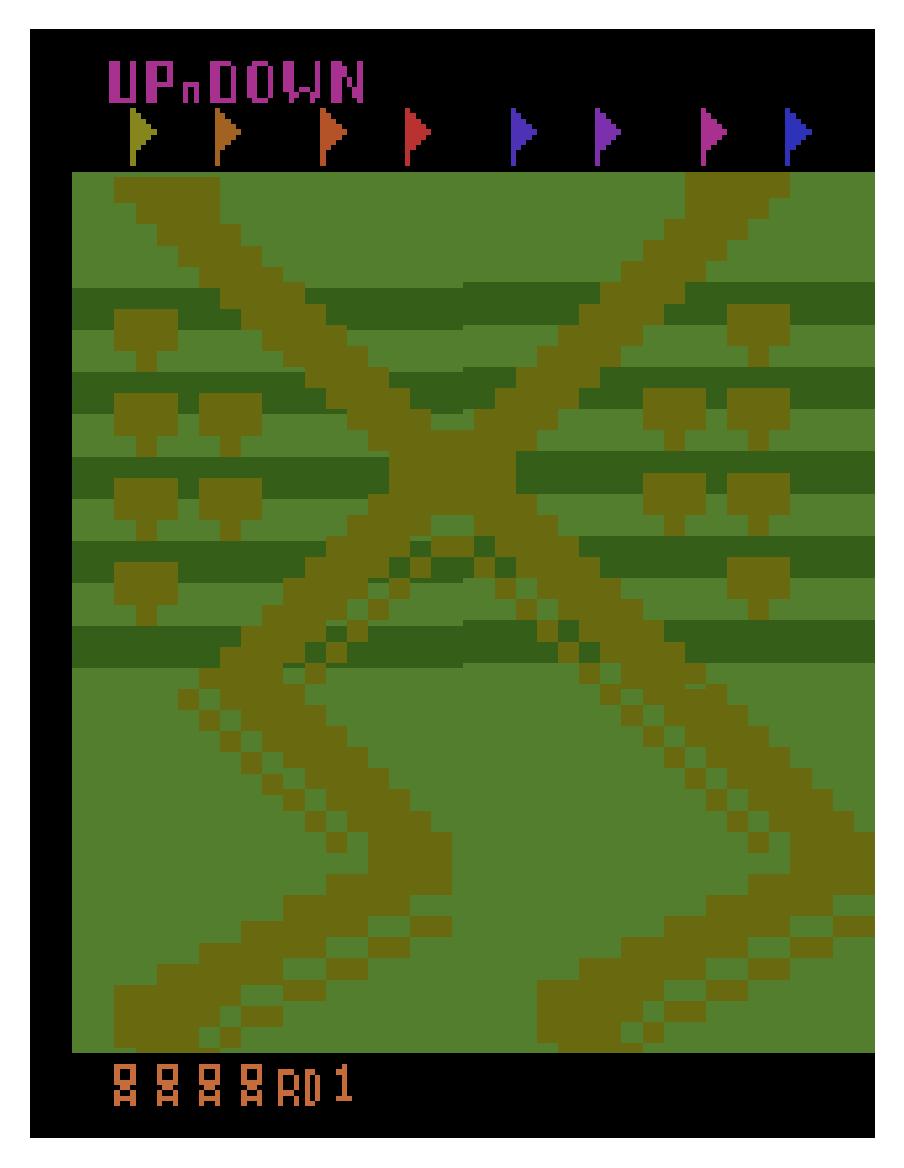}
    \caption{UpNDown}
\end{subfigure}
\caption{\label{app fig: atari images continuation}Images for Atari100k suites used in our experiments (continuation of \cref{app fig: atari images}). Each image is a frame of a specific Atari game.}
\end{figure*}

\begin{table*}[htbp]
\centering
\begin{tabular}{l|cl}
\toprule
Env&State Dimension & Action Dimension\\
\midrule
 InvertedPendulum & 5 & Continuous(1) \\
 InvertedDoublePendulum & 9 & Continuous(1) \\
 InvertedPendulumSwingup & 5 & Continuous(1) \\
 Reacher & 9 & Continuous(2) \\
 Walker2D & 22 & Continuous(6) \\
 HalfCheetah & 26 & Continuous(6) \\
 Ant & 28 & Continuous(8) \\
 Hopper & 15 & Continuous(3) \\
\bottomrule
        \end{tabular}
        \caption{\label{app table: bullet info}State dimension and action space for Bullet suite. Continuous($x$) means the action space is continuous with dimension $x$. }
\end{table*}

\begin{table*}[htbp]
\centering
\begin{tabular}{l|ll}
\toprule
Env&State Dimension & Action Dimension\\
\midrule
 Reacher & 11 & Continuous(2) \\
 Walker2d & 17 & Continuous(6) \\
 HalfCheetah & 17 & Continuous(6) \\
 Swimmer & 8 & Continuous(2) \\
 Ant & 111 & Continuous(8) \\
 Hopper & 11 & Continuous(3) \\
 InvertedPendulum & 4 & Continuous(1) \\
 InvertedDoublePendulum & 11 & Continuous(1) \\
\bottomrule
        \end{tabular}
        \caption{\label{app table: mujoco info}State space and action space for MuJoCo suite. Continuous($x$) means the action space is continuous with dimension $x$.}
\end{table*}

\begin{table*}[htbp]
\centering
\begin{tabular}{ll|ll}
\toprule
Domain&Tasks&State Space & Action Space\\
\midrule

ball\_in\_cup & catch & (3, 100, 100) & Continuous(2) \\
cartpole & balance & (3, 100, 100) & Continuous(1) \\
cartpole & balance\_sparse & (3, 100, 100) & Continuous(1) \\
cartpole & swingup & (3, 100, 100) & Continuous(1) \\
cartpole & swingup\_sparse & (3, 100, 100) & Continuous(1) \\
cheetah & run & (3, 100, 100) & Continuous(6) \\
finger & spin & (3, 100, 100) & Continuous(2) \\
finger & turn\_easy & (3, 100, 100) & Continuous(2) \\
finger & turn\_hard & (3, 100, 100) & Continuous(2) \\
hopper & hop & (3, 100, 100) & Continuous(4) \\
hopper & stand & (3, 100, 100) & Continuous(4) \\
pendulum & swingup & (3, 100, 100) & Continuous(1) \\
reacher & easy & (3, 100, 100) & Continuous(2) \\
reacher & hard & (3, 100, 100) & Continuous(2) \\
walker & stand & (3, 100, 100) & Continuous(6) \\
walker & walk & (3, 100, 100) & Continuous(6) \\

\bottomrule
        \end{tabular}
        \caption{\label{app table: dmc info}State space and action space for DMControl suite. Continuous($x$) means the action space is continuous with dimension $x$.}
\end{table*}

\begin{table*}[htbp]
\centering
\begin{tabular}{l|ll}
\toprule
\textbf{Game}&State Space & Action Space\\
\midrule

Alien & (210, 160, 3) & Discrete(18) \\
Amidar & (210, 160, 3) & Discrete(10) \\
Assault & (210, 160, 3) & Discrete(7) \\
Asterix & (210, 160, 3) & Discrete(9) \\
BankHeist & (210, 160, 3) & Discrete(18) \\
BattleZone & (210, 160, 3) & Discrete(18) \\
Boxing & (210, 160, 3) & Discrete(18) \\
Breakout & (210, 160, 3) & Discrete(4) \\
ChopperCommand & (210, 160, 3) & Discrete(18) \\
CrazyClimber & (210, 160, 3) & Discrete(9) \\
DemonAttack & (210, 160, 3) & Discrete(6) \\
Freeway & (210, 160, 3) & Discrete(3) \\
Frostbite & (210, 160, 3) & Discrete(18) \\
Gopher & (210, 160, 3) & Discrete(8) \\
Hero & (210, 160, 3) & Discrete(18) \\
Jamesbond & (210, 160, 3) & Discrete(18) \\
Kangaroo & (210, 160, 3) & Discrete(18) \\
Krull & (210, 160, 3) & Discrete(18) \\
KungFuMaster & (210, 160, 3) & Discrete(14) \\
MsPacman & (210, 160, 3) & Discrete(9) \\
Pong & (210, 160, 3) & Discrete(6) \\
PrivateEye & (210, 160, 3) & Discrete(18) \\
Qbert & (210, 160, 3) & Discrete(6) \\
RoadRunner & (210, 160, 3) & Discrete(18) \\
Seaquest & (210, 160, 3) & Discrete(18) \\
UpNDown & (210, 160, 3) & Discrete(6) \\
\bottomrule
        \end{tabular}
        \caption{\label{app table: atari info}State space and action space for Atari suite. Discrete($x$) means the action space is discrete with $x$ actions.}
\end{table*}
 \clearpage
 \clearpage
\section{Appendix: Additional Experimental Results}

In this section, we provide additional experimental results. PEER works by adding a regularization term to backbone DRL algorithms. Thus, the comparison with the backbone algorithm of PEER naturally becomes an ablation experiment. We provide more experiments to demonstrate the effectiveness of PEER.

\subsection{Experiments on MuJoCo Suite}

We present the performance of PEER on the MuJoCo suite in \cref{app table: exp mujoco}. The results show that our proposed PEER outperforms or matches the compared algorithms in \textbf{5} out \textbf{7} MuJoCo environments. Compared with its backbone algorithm TD3, PEER surpasses it in \textbf{6} out of \textbf{7} environments.

\begin{table}[!htbp]
\setlength\tabcolsep{3pt}
    \centering

\begin{tabular}{llllllll}
\toprule
\textbf{Algorithm}&\textbf{Ant}&\textbf{HalfCheetah}&\textbf{Hopper}&\textbf{InvDouPen}&\textbf{InvPen}&\textbf{Reacher}&\textbf{Walker}\\
\midrule

PEER &\colorbox{mine}{\!\!5386} $\pm${\footnotesize493} &10832 $\pm$ {\footnotesize501} &\colorbox{mine}{\!\!3424} $\pm${\footnotesize180} &7470 $\pm$ {\footnotesize3721} &\colorbox{mine}{\!\!1000} $\pm${\footnotesize0}
 &\colorbox{mine}{\!\!-4} $\pm${\footnotesize1} &\colorbox{mine}{\!\!4223} $\pm${\footnotesize655}\\

TD3 &5102 $\pm$ {\footnotesize787} &\colorbox{mine}{\!\!10858} $\pm${\footnotesize637} &3163 $\pm$ {\footnotesize367} &7312 $\pm$ {\footnotesize3653} &1000 $\pm$ {\footnotesize0} &-4 $\pm$ {\footnotesize1} &3762 $\pm$ {\footnotesize956}\\
METD3 &2256 $\pm$ {\footnotesize431} &5696 $\pm$ {\footnotesize1740} &804 $\pm$ {\footnotesize71} &7815 $\pm$ {\footnotesize0} &912 $\pm$ {\footnotesize71} &-8 $\pm$ {\footnotesize3} &2079 $\pm$ {\footnotesize1096}\\

SAC &4233 $\pm$ {\footnotesize806} &10482 $\pm$ {\footnotesize959} &2666 $\pm$ {\footnotesize320} &\colorbox{mine}{\!\!9358} $\pm${\footnotesize0} &1000 $\pm$ {\footnotesize0} &-4 $\pm$ {\footnotesize0} &4187 $\pm$ {\footnotesize304}\\

\bottomrule
\end{tabular}
   \caption{\label{app table: exp mujoco}The average return of the last ten evaluations over ten random seeds. The maximum average returns are bolded. PEER outperforms or matches the other tested algorithms in 5 out of 7 environments.}
\end{table}

In \cref{app table: redq mujoco}, we show comparisons with model-free algorithm REDQ~\cite{redq} on pybullet suite.

\begin{table}[!h]
    \centering
\begin{tabular}{llll}
\toprule
Algo&Ant&Hopper&Walker \\
\midrule
PEER  & \colorbox{mine}{\!\!5386} $\pm$ {\footnotesize 493}  & \colorbox{mine}{\!\!3424} $\pm$ {\footnotesize 180} & 
 \colorbox{mine}{\!\!4223} $\pm$ {\footnotesize 655} \\
REDQ & 3900 $\pm$ {\footnotesize 890} & 2656 $\pm$ {\footnotesize 759} & 4211 $\pm$ {\footnotesize 524} \\
\bottomrule
\end{tabular}
\caption{\label{app table: redq mujoco}Average return for PEER and REDQ. PEER surpasses REDQ on all tested tasks. The REDQ results are obtained using the authors' implementation and are reported over 20 trials.}
\end{table}

\subsection{Combination with Model-based Algorithm}

In \cref{app table: dmc TDMPC}, we show comparisons with model-based methods algorithms TDMPC and Dreamer-v2 on DMControl suites. In \cref{app table atari26 performance}, we show comparisons with model-based algorithms Dreamer-v2. Note that the data we take directly from the authors' dreamer-v2 codebase (\href{https://github.com/danijar/dreamerv2/tree/main/scores}{https://github.com/danijar/dreamerv2/tree/main/scores}), the amount of data they use is 1000k, which is \textbf{10} times more than our PEER. The PEER scores in \cref{app table atari26 performance} are taken as the largest of PEER+DrQ and PEER+CURL.

\begin{table}[h]
    \centering
    \scalebox{1}{
    \begin{tabular}{l|llllll}
    \toprule 
        500K Step Scores & Finger, Spin
 & Cartpole, Swingup
 & Reacher, Easy
 & Cheetah, run
 & Walker, Walk
 & Ball\_in\_cup, Catch
 \\ \midrule
        PEER + CURL &\colorbox{mine}{\!\!864} $\pm$ 160 &\colorbox{mine}{\!\!866} $\pm$ 17&\colorbox{mine}{\!\!980} $\pm$ 3 &\colorbox{mine}{\!\!732} $\pm$ 41 &\colorbox{mine}{\!\!946} $\pm$ 17 &\colorbox{mine}{\!\!971} $\pm$ 5 \\
Dreamer-V2 &386 $\pm$ 83 &{853} $\pm$ 15 &876 $\pm$ 60 &610 $\pm$ 117 &934 $\pm$ 16 &792 $\pm$300  \\ 
\midrule 
100K Step Scores \\
\midrule 
PRER {+TDMPC} &772 $\pm$ 107 &\colorbox{mine}{\!\!848}  $\pm$ 25&\colorbox{mine}{\!\!841} $\pm$ 115 &\colorbox{mine}{\!\!636}  $\pm$ 35&\colorbox{mine}{\!\!876}  $\pm$ 41 &\colorbox{mine}{\!\!937}  $\pm$ 96\\
TDMPC &\colorbox{mine}{\!\!943}  $\pm$ 59&770 $\pm$ 70 &628 $\pm$ 105 &222 $\pm$ 88 &577 $\pm$ 208 &933  $\pm$ 24  \\
Dreamer-V2 &414 $\pm$ 93 &697 $\pm$ 176 &633 $\pm$ 248  &501 $\pm$ 146 &705 $\pm$ 232 & 693 $\pm$335  \\
\bottomrule
\end{tabular}}
 \caption{\label{app table: dmc TDMPC} Comparison with model-based methods. PEER outperforms Dreamer-v2 on 12 out of 12 tasks. PEER (combined with TDMPC~\cite{tdmpc}) outperforms TDMPC by on 5 out of 6 tasks.}
\end{table}

\begin{table*}
    \centering
    \scalebox{1.0}{
    \begin{tabular}{llll}
    \toprule
Game & PEER & Dreamer-V2 & MuZero \\
\midrule
Alien & \colorbox{mine}{\!\!1218.9} & 384.1 & 530.0  \\
Amidar & \colorbox{mine}{\!\!185.2} & 29.8 & 38.8  \\
Assault & \colorbox{mine}{\!\!721.0} & 433.4 & 500.1  \\
Asterix & 918.2 & 330.6 & \colorbox{mine}{\!\!1734.0} \\
BHeist & 78.6 & 127.1 & \colorbox{mine}{\!\!192.5} \\
BZone & \colorbox{mine}{\!\!15727.3} & 4200.0 & 7687.5  \\
Boxing & 14.5 & \colorbox{mine}{\!\!37.7} & 15.1  \\
Breakout & 8.5 & 1.5 & \colorbox{mine}{\!\!48.0} \\
ChpCmd & \colorbox{mine}{\!\!1451.8} & 687.5 & 1350.0  \\
CzClmr & 18922.7 & 25232.5 & \colorbox{mine}{\!\!56937.0} \\
DmAttack & 1236.7 & 182.9 & \colorbox{mine}{\!\!3527.0} \\
Freeway & \colorbox{mine}{\!\!30.4} & 11.6 & 21.8  \\
Frostbite & \colorbox{mine}{\!\!2151.0} & 302.5 & 255.0  \\
Gopher & 681.8 & 820.2 & \colorbox{mine}{\!\!1256.0} \\
Hero & \colorbox{mine}{\!\!7499.9} & 2185.0 & 3095.0  \\
Jbond & \colorbox{mine}{\!\!414.1} & 81.2 & 87.5  \\
Kangaroo & \colorbox{mine}{\!\!1148.2} & 150.0 & 62.5  \\
Krull & \colorbox{mine}{\!\!5444.7} & 3853.8 & 4890.8  \\
KFMaster & 15439.1 & 12420.3 & \colorbox{mine}{\!\!18813.0} \\
MsPacman & \colorbox{mine}{\!\!1768.4} & 647.9 & 1265.6  \\
Pong & -9.5 & -18.3 & \colorbox{mine}{\!\!-6.7} \\
PriEye & \colorbox{mine}{\!\!3207.7} & 188.8 & 56.3  \\
Qbert & 2197.7 & 318.6 & \colorbox{mine}{\!\!3952.0} \\
RdRunner & \colorbox{mine}{\!\!10697.3} & 3622.5 & 2500.0  \\
Squest & \colorbox{mine}{\!\!538.5} & 356.0 & 208.0  \\
UpNDown & 7680.9 & \colorbox{mine}{\!\!8025.1} & 2896.9  \\
\bottomrule
 \end{tabular}
 }
\caption{\label{app table atari26 performance}PEER outperforms Dreamer-v2 and Muzero on 21 and 16 games of Atari26 where Dreamer-v2 even uses \textit{10} times the data of PEER. Note that the data we take directly from the authors' dreamer-v2 codebase, the amount of data for Dreamer-V2 they use is 1000k, which is \textbf{10} times more than our PEER.}
\end{table*} 

\newpage
\subsection{Various $\beta$ for Performance Improvement}
Fine-tuning for hyper-parameters probably improves the performance of PEER. To see this, we select 7 Atari environments to investigate the effect of fine-tuning $\beta$, where PEER (coupled with CURL) achieves SOTA performance. We present the results in \cref{fig: various beta}. There is no one value taken that is significantly better than the other. We see that large $\beta$ (=1e-2) may result in the failure of learning (on Freeway game) but may also bring the best performance improvements (on Kangaroo game). Overall, fine-tuning the hyper-parameter $\beta$ may improve the empirical performance by a large margin.

\begin{figure*}[!htbp]
	\begin{center}
		\includegraphics[width=\linewidth]{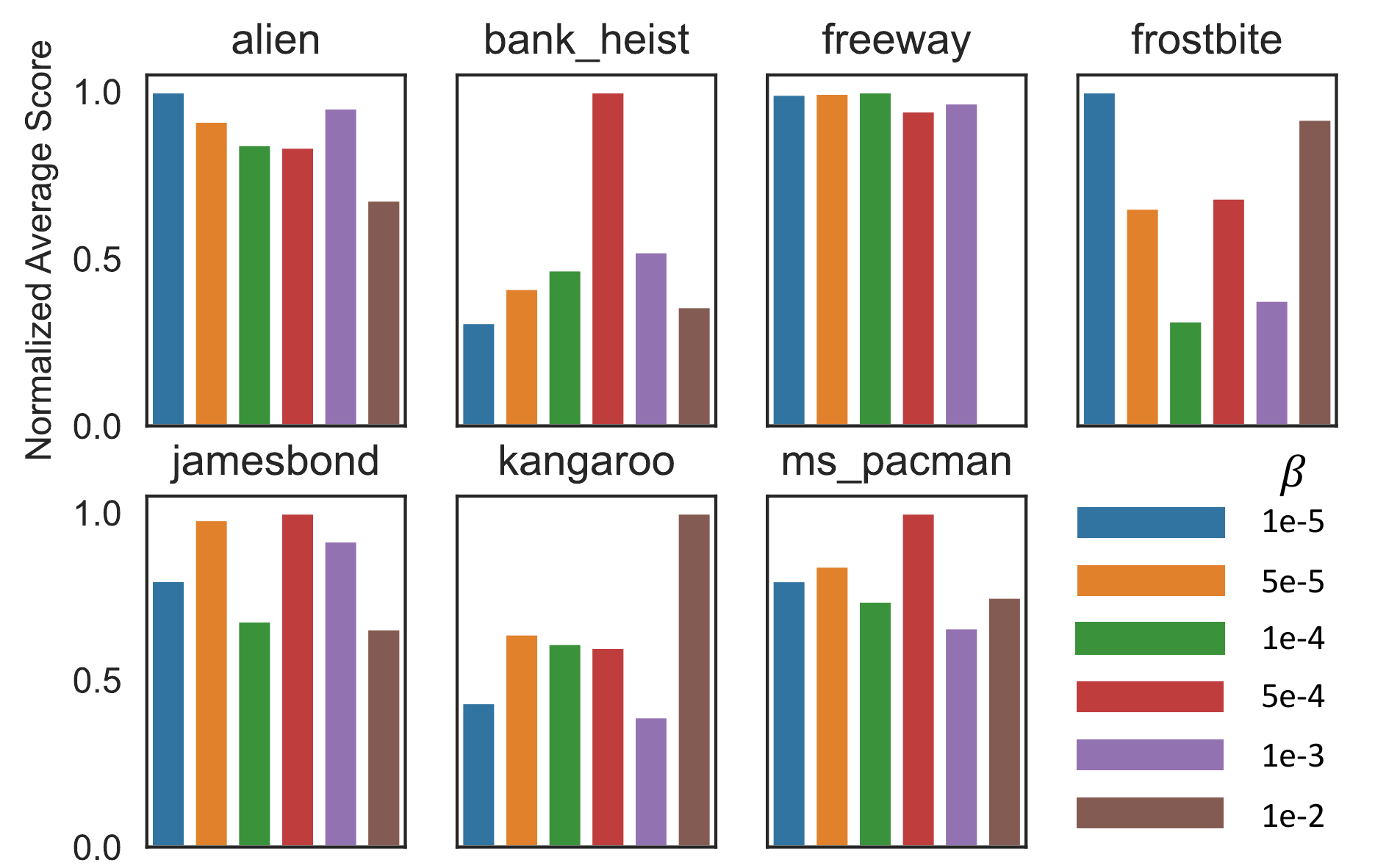} 
	\end{center}
	\caption{\label{fig: various beta}The average scores normalized by the max average score on the 7 Atari games for selected 6 hyper-parameter $\beta$. From the experiments, we can see that fine-tuning the $\beta$ may result in performance improvements.} 
\end{figure*}

\subsection{Performance curves on DMControl Tasks}

We present the performance curves of PEER on a total of 16 DMControl environments in \cref{app fig: additional results} and \cref{app fig: additional results drq}. We run 10 seeds in each environment.

\begin{figure*}[!htbp]
\centering
\includegraphics[width=\linewidth]{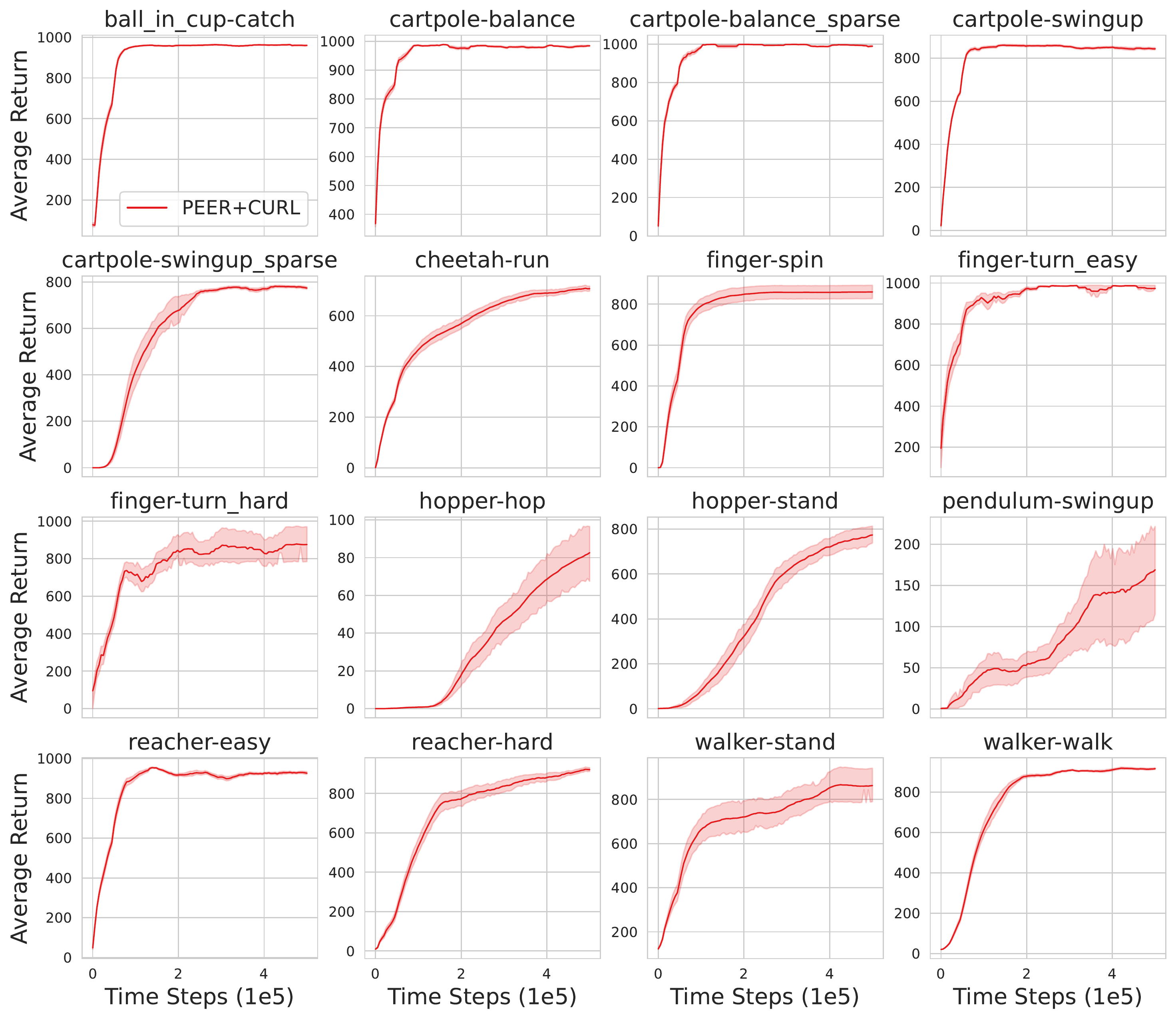}
	\caption{\label{app fig: additional results}Performance curves for PEER (coupled with CURL) on DMControl suite. The shaded region represents half the standard deviation of the average evaluation over 10 seeds. The curves are smoothed by moving average.} 
\end{figure*}

\begin{figure*}[!htbp]
\centering
\includegraphics[width=\linewidth]{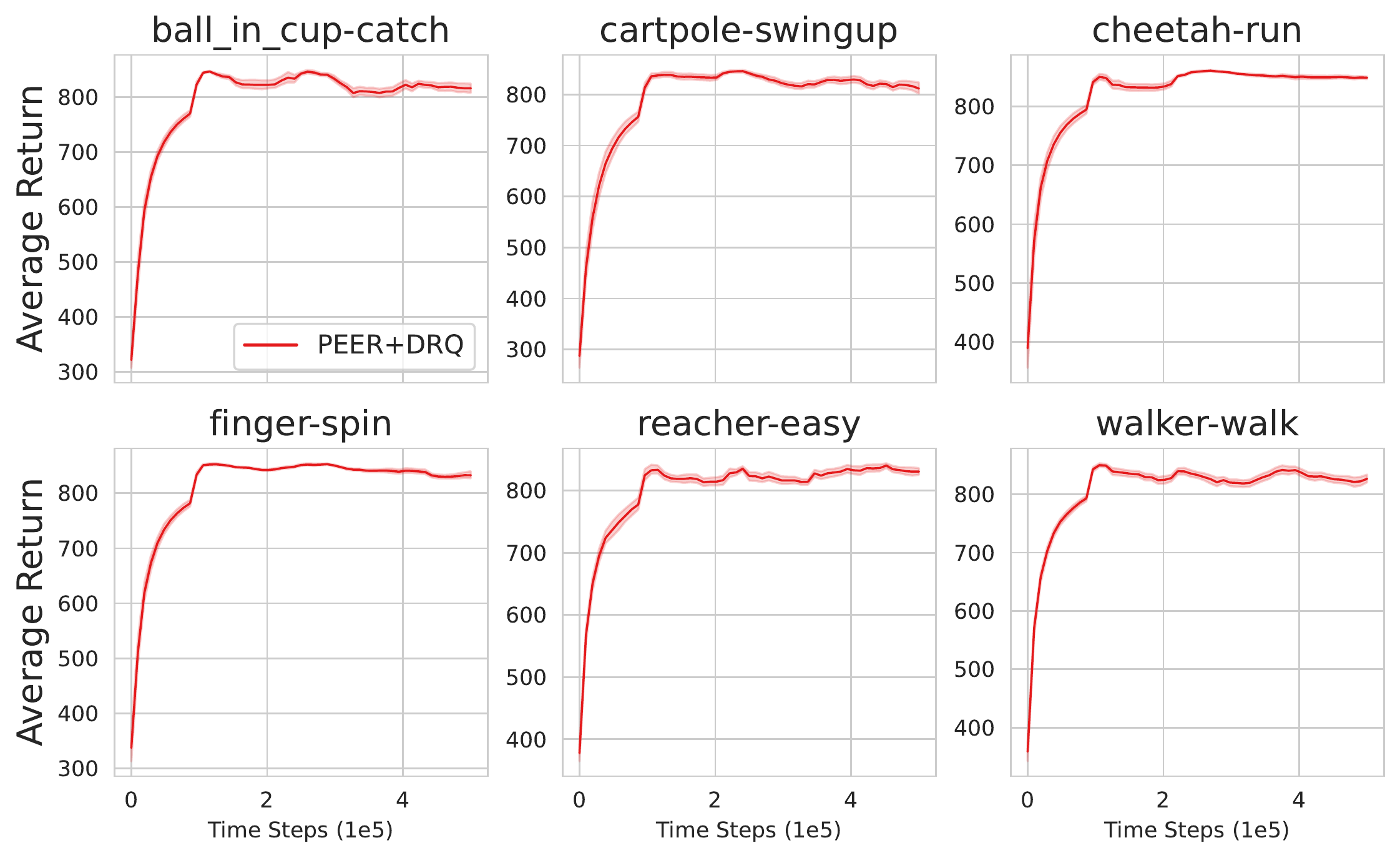}
	\caption{\label{app fig: additional results drq}Performance curves for PEER (coupled with DrQ) on DMControl suite. The shaded region represents half the standard deviation of the average evaluation over 10 seeds. The curves are smoothed by moving average.} 
\end{figure*}
\end{document}